\documentclass{article}

\PassOptionsToPackage{numbers,compress}{natbib}

 \usepackage[preprint]{neurips_2026}

\usepackage[utf8]{inputenc} 
\usepackage[T1]{fontenc}    
\usepackage{hyperref}       
\usepackage{graphicx}       
\usepackage{url}            
\usepackage{booktabs}       
\usepackage{amsmath}        
\usepackage{amsfonts}       
\usepackage{amsthm}         
\usepackage{nicefrac}       
\usepackage{microtype}      
\usepackage{xcolor}         
\usepackage{subcaption} 
\usepackage{caption} 
\usepackage{colortbl} 
\usepackage{multirow} 
\usepackage{bbding}
\usepackage{bm}
\usepackage{microtype}      
\usepackage[normalem]{ulem} 

\newtheorem{theorem}{Theorem}[section]
\newtheorem{lemma}[theorem]{Lemma}
\newtheorem{assumption}[theorem]{Assumption}

\theoremstyle{definition}
\newtheorem{definition}[theorem]{Definition}
\theoremstyle{remark}

\title{Beyond Rigid Alignment: Graph Federated Learning via Dual Manifold Calibration}

\author{%
  Wentao Yu\\
  School of Computer Science and Engineering\\
  Nanjing University of Science and Technology\\
  Nanjing, China \\
  \texttt{wentao.yu@njust.edu.cn} \\
  \And
  Bo Han\\
  Department of Computer Science\\
  Hong Kong Baptist University\\
  Hong Kong, China\\
  \texttt{bhanml@comp.hkbu.edu.hk}\\
  \AND
  Jie Yang \\
  School of Automation and Intelligent Sensing\\
  Shanghai Jiao Tong University\\
  Shanghai, China\\
  \texttt{jieyang@sjtu.edu.cn}\\
  \And
  Chen Gong\\
  School of Automation and Intelligent Sensing\\
  Shanghai Jiao Tong University\\
  Shanghai, China\\
  \texttt{chen.gong@sjtu.edu.cn}\\
}

\begin{document}

\maketitle

\begin{abstract}
Graph Federated Learning (GFL) enables collaborative representation learning across distributed subgraphs while preserving privacy. However, heterogeneity remains a critical challenge, as subgraphs across clients typically differ significantly in both semantics and structures. Existing methods address heterogeneity by enforcing the rigid alignment of model parameters or prototypes between clients and the server. However, these alignments implicitly rely on a restrictive global linearity assumption that summarizes local data distributions using a single and globally consistent representation space. This severely compresses the personalized representation space of clients and fails to preserve diverse local graph distributions. To overcome these limitations, we propose \underline{\textbf{Fed}}erated \underline{\textbf{G}}raph \underline{\textbf{M}}anifold \underline{\textbf{C}}alibration (FedGMC), a novel paradigm that tackles semantic heterogeneity and structural heterogeneity from a unified manifold perspective. Instead of enforcing rigid alignment, FedGMC introduces a dual manifold calibration mechanism that preserves global commonalities while maximizing the personalized representation space of local clients. Specifically, for semantic heterogeneity, the server constructs a geometrically optimal semantic manifold via equidistant semantic anchors, so as to guide the calibration of local semantic manifolds. For structural heterogeneity, the server constructs a global structural manifold by building global structural templates, so as to guide the calibration of local structural manifolds. Finally, the server dynamically refines both global semantic manifolds and structural manifolds by aggregating local manifolds. Extensive experiments on eleven homophilic and heterophilic graphs demonstrate that FedGMC effectively balances global commonality and local personalization, thereby significantly outperforming state-of-the-art baseline methods.
\end{abstract}

\section{Introduction}
Graphs are ubiquitous data structures that model complex relational data across diverse domains, which include social networks, traffic systems, and molecular chemistry~\cite{bai2022two, 10032180, zhou2024traffic, zhou2025fedtps, zhou2026fedhint, yu2026atom}. In many real-world scenarios, large-scale graphs are typically partitioned across multiple isolated organizations or devices, which leads to locally accessible subgraphs due to stringent privacy regulations. Graph Federated Learning (GFL)~\cite{baek2023personalized, wentao2025fediih, yu2025homophily, yuintegrating, yu2026heterogeneity} has emerged as a promising paradigm that enables multiple clients to collaboratively train Graph Neural Networks (GNNs) without exposing raw graph data. However, the non-IID (non-independent and identically distributed) nature of subgraphs induces severe heterogeneity in GFL. In contrast to Euclidean data such as images, heterogeneity in GFL is substantially more complex and challenging, as it arises simultaneously from diverse node features (\textit{i.e.}, semantic heterogeneity) and highly varied topological structures (\textit{i.e.}, structural heterogeneity) across clients~\cite{NEURIPS2021_9c6947bd}.

To address heterogeneity, most existing methods focus on enforcing rigid alignment of model parameters or prototypes between clients and the server. For example, FedProx~\cite{MLSYS2020_1f5fe839} introduces a proximal term to constrain local updates toward the global model. Meanwhile, FedProto~\cite{tan2022fedproto} aligns local and global prototypes to alleviate heterogeneity. More recently, some approaches~\cite{yu2026heterogeneity, huang2023federated} have attempted to tackle node feature heterogeneity and structural heterogeneity through semantic alignment and structural alignment, respectively. However, these methods share a critical yet largely overlooked limitation: They implicitly rely on a restrictive global linearity assumption~\cite{ma2026fedmc}. This assumption attempts to represent diverse local data distributions within a single and globally consistent representation space. By enforcing such rigid alignment, existing methods force local representations toward a global consensus. This approach fundamentally conflicts with the inherently non-IID nature of local graphs. It severely compresses the personalized representation space of clients, which compels them to sacrifice unique semantic characteristics. Moreover, it is particularly harmful to structural representations, as blindly aligning graph topologies can obliterate distinctive connectivity patterns (\textit{e.g.}, diverse topological structures) that are crucial for downstream tasks~\cite{zhu2020beyond}.

To address these fundamental issues, we argue that the key lies not in forcing clients to align parameters nor prototypes, but rather in calibrating the intrinsic geometric structures of their local manifolds. To this end, we \textit{shift} the paradigm from enforcing rigid alignment to manifold calibration~\cite{ma2026fedmc, wangconsistency}. In contrast to rigid alignment, manifold calibration preserves the intrinsic geometry of each client's local manifold (\textit{i.e.}, the relative distances among representations). To achieve this, each client merely needs to compute a transformation (\textit{e.g.}, rotation) to establish geometric consistency between the local manifold and a shared global reference. In this way, we can establish a global consensus without sacrificing the essential personalized representation space of clients.

Motivated by this insight, we propose a novel paradigm named \underline{\textbf{Fed}}erated \underline{\textbf{G}}raph \underline{\textbf{M}}anifold \underline{\textbf{C}}alibration (FedGMC), which tackles both semantic heterogeneity and structural heterogeneity from a unified manifold perspective. We are the \textit{first} to introduce the manifold calibration paradigm into GFL to address heterogeneity. Unlike existing methods that require rigid alignment, FedGMC preserves global commonalities while maximizing the personalized representation space of local clients.

Specifically, FedGMC introduces a dual manifold calibration mechanism. (i) For semantic heterogeneity, the server constructs a geometrically optimal semantic manifold by using a simplex Equiangular Tight Frame (ETF), where equidistant semantic anchors serve as a shared global reference to guide the consistent calibration of local semantic manifolds. Each client then extracts structure-free ego-embeddings to construct local semantic manifolds, and employs the orthogonal Procrustes transformation to isometrically calibrate them towards the global semantic anchors. This orthogonal transformation effectively preserves the personalized representations of local semantics. (ii) For structural heterogeneity, the server constructs a global structural manifold by building global structural templates, which serve as a shared global reference to guide the consistent calibration of local structural manifolds. Each client captures the topological structures and connectivity characteristics of local graph to construct local structural manifold. It then leverages the optimal transport to match its local structural manifold with global structural templates. This calibration effectively preserves the personalized structural characteristics of local graphs while achieving global consistency. Finally, to resolve the geometric discrepancy between geometrically optimal global manifolds and diverse local manifolds, the server dynamically refines both global semantic and global structural manifolds via difficulty-weighted updates and Gromov-Wasserstein distances, respectively.
\section{Related Work}
\subsection{Graph Federated Learning}
Graph Federated Learning (GFL) has emerged as a promising paradigm for collaborative graph representation learning without sharing raw graph data. However, unlike conventional federated learning on Euclidean data (\textit{e.g.}, images and text), GFL suffers from significantly more intricate heterogeneity due to the simultaneous presence of diverse node features and highly varied topological structures across clients.

To address heterogeneity, existing GFL methods primarily enforce the rigid alignment of model parameters or prototypes between clients and the server. The first category of methods tackles heterogeneity by dynamically adjusting the aggregation weights of model parameters based on client similarities. For example, GCFL~\cite{NEURIPS2021_9c6947bd} and FedGNN~\cite{wu2021fedgnn} cluster clients according to the similarity of local gradients and local graphs, respectively. They then perform federated averaging within each cluster. FED-PUB~\cite{baek2023personalized} estimates pairwise client similarities from the outputs of local models. It then employs these similarities to perform weighted aggregation of model parameters, where higher aggregation weights are assigned to more similar clients. Similarly, FedIIH~\cite{wentao2025fediih} conducts a weighted federation of model parameters based on similarities derived from inferred graph distributions. Consequently, these methods can be viewed as performing similarity-based alignment of model parameters, as they assign higher aggregation weights to clients with higher similarity.

The second category of methods addresses heterogeneity through the alignment of semantics or structures. For instance, FedTAD~\cite{zhu2024fedtad} performs topology-aware knowledge distillation to align the global model with reliable local knowledge during aggregation. FedICI~\cite{yuintegrating} aggregates low-frequency components in a low-rank space to enhance the consistency of common features. More recently, FedSSA~\cite{yu2026heterogeneity} explicitly aligns both semantic knowledge and spectral characteristics. Ultimately, these methods can be regarded as prototype-level alignment, as they force the alignment of representations or knowledge toward a global consensus.

Despite these advances, existing methods share a critical yet largely overlooked limitation: They implicitly rely on a restrictive global linearity assumption~\cite{ma2026fedmc}. This assumption severely compresses the personalized representation space of clients, which compels them to sacrifice unique semantic and structural characteristics. In contrast, our FedGMC addresses heterogeneity from a manifold perspective, which preserves global commonalities while maximizing the personalized representation space of local clients.

\subsection{Manifold Learning}
According to the manifold hypothesis, high-dimensional data typically resides on or near low-dimensional manifolds rather than being uniformly distributed across the ambient Euclidean space~\cite{tenenbaum2000global, lei2020geometric}. Manifold learning aims to uncover and preserve these intrinsic geometric structures, leading to its widespread application across diverse domains~\cite{caterini2021rectangular, kiani2024hardness}. In the context of federated learning, recent studies have begun leveraging manifold learning to mitigate heterogeneity by calibrating the intrinsic geometries of local distributions. For instance, FedMC~\cite{ma2026fedmc} introduces a manifold calibration mechanism that establishes cross-client consistency by orienting local manifolds toward a global reference. However, this method is fundamentally tailored for Euclidean data (\textit{e.g.}, images) and fails to capture the intricate topological structures that are indispensable in GFL. In contrast, our FedGMC is the \textit{first} to extend the manifold calibration paradigm into GFL, effectively addressing both semantic heterogeneity and structural heterogeneity from a unified manifold perspective.

\section{Problem Definition}
In this work, we focus on the node classification task in GFL. Our goal is to collaboratively train GNNs across $M$ clients, where each client $m \in \{1, 2, \cdots, M\}$ holds a private local subgraph $\mathcal{G}_m = (\mathcal{V}_m, \mathcal{E}_m)$. Here, $\mathcal{V}_m$ and $\mathcal{E}_m$ denote the node set and edge set of $\mathcal{G}_m$, respectively. Let $\mathbf{X}_m \in \mathbb{R}^{n_m \times d}$ and $\mathbf{A}_m \in \mathbb{R}^{n_m \times n_m}$ be the node feature matrix and adjacency matrix of $\mathcal{G}_m$, where $n_m = |\mathcal{V}_m|$ is the number of nodes and $d$ is the feature dimension. Due to privacy constraints, each client's local subgraph $\mathcal{G}_m$ is inaccessible to all other clients.

\section{Our Proposed Method}
In this section, we present the details of our proposed \underline{\textbf{Fed}}erated \underline{\textbf{G}}raph \underline{\textbf{M}}anifold \underline{\textbf{C}}alibration (FedGMC).

\subsection{Framework Overview}
As shown in Fig.~\ref{fig1}, FedGMC introduces a dual manifold calibration mechanism. (i) \textbf{Semantic Manifold Calibration:} To address semantic heterogeneity, the server first constructs a geometrically optimal global semantic manifold. Each client then extracts structure-free ego-embeddings to form its local semantic manifold and employs an orthogonal Procrustes transformation to calibrate it toward the global semantic anchors. (ii) \textbf{Structural Manifold Calibration:} To address structural heterogeneity, the server constructs global structural templates that serve as a global structural manifold. Each client captures the topological structures and connectivity characteristics of its local graph to construct its local structural manifold, which is then matched against the global structural templates. (iii) \textbf{Global Manifold Refinement:} Finally, to resolve the geometric discrepancy between the optimal global manifolds and the diverse local manifolds, the server dynamically refines both the global semantic and structural manifolds.
\begin{figure}[t]
	\centering
	\includegraphics[width=13.5cm]{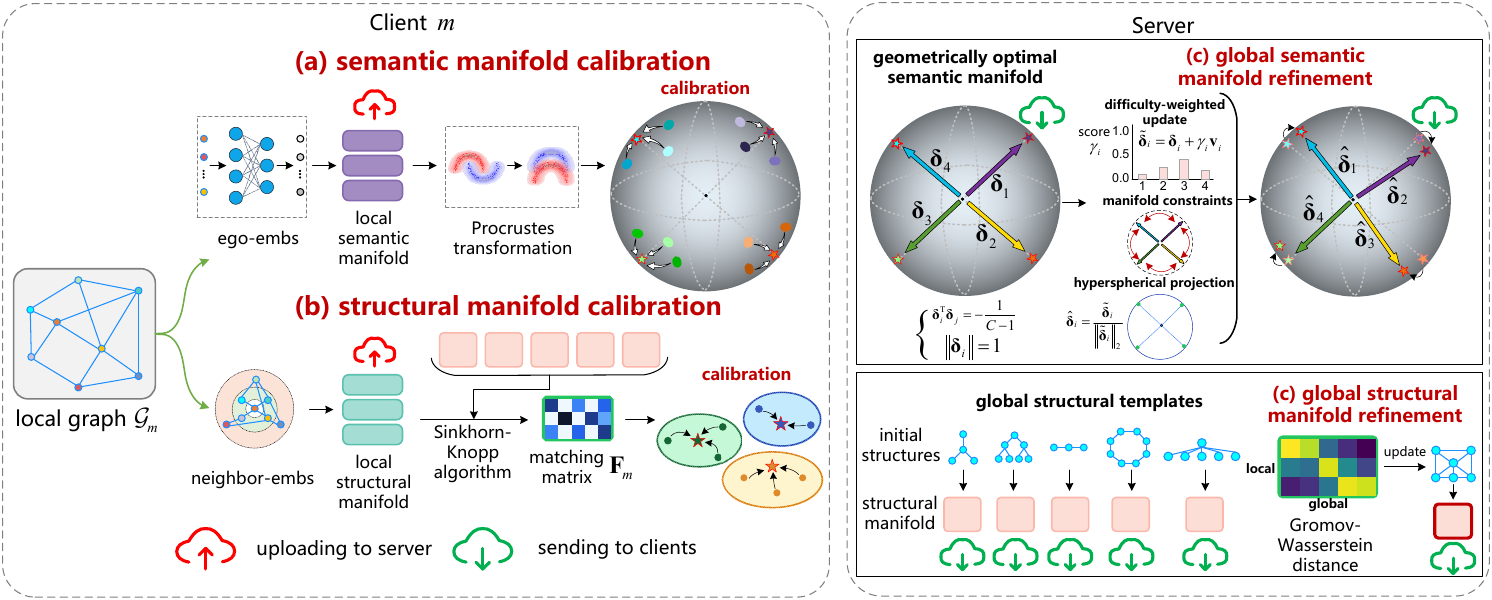}
	\caption{The overview of our proposed FedGMC. }
	\label{fig1}
\end{figure}
\subsection{Semantic Manifold Calibration}
\subsubsection{Constructing Global Semantic Manifold}
To establish a geometrically optimal global reference for semantic calibration, the server constructs a global semantic manifold based on the simplex Equiangular Tight Frame (ETF)~\cite{fickus2016equiangular}. Given $C$ classes, we define a set of semantic anchors $\Delta = \{\bm{\delta}_1, \bm{\delta}_2, \cdots, \bm{\delta}_C\}$ that satisfy the following geometric optimality constraint:
\begin{equation}
\bm{\delta}_i^{\top} \bm{\delta}_j = \frac{C}{C-1} \lambda_{i,j} - \frac{1}{C-1}, \quad \forall i, j \in \{1, \cdots, C\},
\label{eq:etf_constraint}
\end{equation}
where $\bm{\delta}_i \in \mathbb{R}^{d}$, $d$ denotes the embedding dimension, $\lambda_{i,j} = 1$ if $i = j$, and $\lambda_{i,j} = 0$ otherwise. To initialize the global semantic manifold $\Delta \in \mathbb{R}^{d \times C}$, we follow the standard simplex ETF construction~\cite{markou2024guiding}:
\begin{equation}
\Delta = \sqrt{\frac{C}{C-1}} \Phi \left( \mathbf{I}_C - \frac{1}{C} \mathbf{1}_C \mathbf{1}_C^{\top} \right),
\label{eq:etf_init}
\end{equation}
where $\Phi \in \mathbb{R}^{d \times C}$ represents a fundamental isometric transformation that embeds the $C$-dimensional simplex into the semantic space, $\mathbf{I}_C$ is the $C \times C$ identity matrix, and $\mathbf{1}_C$ is a $C$-dimensional all-ones column vector. This simplex ETF construction ensures that all semantic anchors possess unit norm and satisfy the maximal equidistance property, thereby providing a geometrically optimal global reference for semantic calibration.
\subsubsection{Local Semantic Manifold Calibration}
\paragraph{Extracting ego-embeddings} To construct the local semantic manifold, we first extract structure-free ego-embeddings for each node in the local graph $\mathcal{G}_m$. Unlike conventional node embeddings that aggregate neighborhood information, ego-embeddings capture only the intrinsic semantic features of individual nodes without being influenced by topological structures. Formally, for client $m$, the ego-embeddings are computed as $\mathbf{H}_m^{\text{ego}} = \sigma(\mathbf{X}_m \mathbf{W}_m^{\text{ego}})$, where $\sigma(\cdot)$ is a non-linear activation function, and $\mathbf{W}_m^{\text{ego}} \in \mathbb{R}^{d \times d}$ is a learnable parameter matrix. By relying solely on node features, ego-embeddings provide a pure representation of semantic characteristics, which is essential for semantic manifold calibration.
\paragraph{Constructing local semantic manifolds} Given ego-embeddings $\mathbf{H}_m^{\text{ego}}$, we construct the local semantic manifold $\mathbf{P}_m \in \mathbb{R}^{d \times C}$ by computing the class-wise mean of ego-embeddings, \textit{i.e.}, $\mathbf{P}_m = \left[ \bar{\mathbf{h}}_{m,1}, \bar{\mathbf{h}}_{m,2}, \cdots, \bar{\mathbf{h}}_{m,C} \right]$, where $\bar{\mathbf{h}}_{m,c} = \frac{1}{|\mathcal{V}_{m,c}|} \sum_{v \in \mathcal{V}_{m,c}} \mathbf{h}_{m,v}^{\text{ego}}$ represents the mean ego-embedding of all nodes belonging to class $c$ in client $m$, and $\mathcal{V}_{m,c}$ denotes the set of nodes with label $c$ in $\mathcal{G}_m$.
\paragraph{Orthogonal Procrustes transformation} To calibrate the local semantic manifold $\mathbf{P}_m$ toward the global semantic manifold $\Delta$, we seek an optimal orthogonal transformation matrix $\mathbf{R}_m \in \mathbb{R}^{d \times d}$ that minimizes the Frobenius norm of the alignment error by solving $\mathbf{R}_m = \arg\min_{\mathbf{R}_m^{\top} \mathbf{R}_m = \mathbf{I}} \left\| \mathbf{R}_m \mathbf{P}_m  - \Delta \right\|_\mathrm{F}^2$. This optimization problem admits a closed-form solution via Singular Value Decomposition (SVD). Specifically, we first compute the SVD of $\Delta \mathbf{P}_m^{\top}$, $\textit{i.e.}$, $\Delta \mathbf{P}_m^{\top} = \mathbf{U}_m \mathbf{\Sigma}_m \mathbf{V}_m^{\top}$, and then obtain the optimal orthogonal transformation as $\mathbf{R}_m = \mathbf{U}_m \mathbf{V}_m^{\top}$. The orthogonality constraint $\mathbf{R}_m^{\top} \mathbf{R}_m = \mathbf{I}$ ensures that the transformation is an isometric mapping, involving only rotations and reflections. This property is crucial, as it maximally preserves the intrinsic geometry and personalized characteristics of the local semantic manifold while establishing consistency with the global reference.
\paragraph{Semantic calibration loss} To enforce the calibration of local semantic manifolds toward the global semantic anchors, we propose the following semantic calibration loss:
\begin{equation}
\mathcal{L}_{\text{se\_calibration}} = \sum_{m=1}^{M} \sum_{v \in \mathcal{V}_m} \left\| \mathbf{h}_{m,v}^{\text{ego}} \mathbf{R}_m - \bm{\delta}_{y_v} \right\|_2^2,
\label{eq:semantic_calibration_loss}
\end{equation}
where $\mathbf{h}_{m,v}^{\text{ego}}$ denotes the ego-embedding of node $v$ in client $m$, and $\bm{\delta}_{y_v}$ represents the global semantic anchor corresponding to the label $y_v$ of node $v$. This semantic calibration loss encourages the calibrated local ego-embeddings to align with their corresponding global semantic anchors, thereby achieving semantic consistency across clients.

\subsection{Structural Manifold Calibration}
\subsubsection{Constructing Global Structural Manifold}
To establish a global reference for structural calibration, the server constructs a global structural manifold $\mathcal{T} = \{\mathbf{T}_1, \mathbf{T}_2, \cdots, \mathbf{T}_Q\}$, where each global structural template $\mathbf{T}_q \in \mathbb{R}^{2d}$ encodes a prototypical topological pattern. These templates are initialized randomly and will be dynamically refined through aggregation of local structural manifolds.
\subsubsection{Local Structural Manifold Calibration}
\paragraph{Local structural embeddings} To capture the structural information of the local graph $\mathcal{G}_m$, we sample $B$ nodes and compute their multi-hop neighbor embeddings. For each sampled node $b$, we aggregate the ego-embeddings of its 1-hop and 2-hop neighbors:
\begin{equation}
\mathbf{h}_{m,b}^{(1)} = \text{COMBINE}\left( \text{AGGR}\left\{ \mathbf{h}_{m,u}^{\text{ego}} : u \in \mathcal{N}_1(b) \right\} \right),
\label{eq:1hop_neighbor}
\end{equation}
\begin{equation}
\mathbf{h}_{m,b}^{(2)} = \text{COMBINE}\left( \text{AGGR}\left\{ \mathbf{h}_{m,u}^{\text{ego}} : u \in \mathcal{N}_1(b) \right\}, \text{AGGR}\left\{ \mathbf{h}_{m,u}^{\text{ego}} : u \in \mathcal{N}_2(b) \right\} \right),
\label{eq:2hop_neighbor}
\end{equation}
where $\mathcal{N}_k(b)$ denotes the $k$-hop neighborhood of node $b$, and AGGR and COMBINE are aggregation and combination functions, respectively.
\paragraph{Radial sequence matrix} To characterize how structural information propagates across multiple hops, we explicitly compute the 1-hop and 2-hop neighbor embeddings and amplify the structural signals via $\ell_2$ normalization:
\begin{equation}
\mathbf{r}_{m,b}^{(1)} = \frac{\mathbf{h}_{m,b}^{(1)}}{\|\mathbf{h}_{m,b}^{(1)}\|_2}, \quad
\mathbf{r}_{m,b}^{(2)} = \frac{\mathbf{h}_{m,b}^{(2)}}{\|\mathbf{h}_{m,b}^{(2)}\|_2}, \quad
\label{eq:residual}
\end{equation}
Afterwards, we introduce the radial sequence matrix, which characterizes the structural manifold at each topological ring as graph signals propagate outward from node $b$, \textit{i.e.}, $\mathbf{R}_{m,b} = \left[ \mathbf{r}_{m,b}^{(1)} \mid \mathbf{r}_{m,b}^{(2)} \right]^{\top} \in \mathbb{R}^{2d}$. This radial sequence matrix effectively encodes the structural characteristics of the local graph at different topological distances from node $b$, and serves as the local structural manifold for subsequent calibration.

\paragraph{Optimal transport matching} To match the $B$ local structural manifolds of client $m$ with $Q$ global structural templates, each client independently computes a client-specific matching matrix $\mathbf{F}_m \in \mathbb{R}^{B \times Q}$ by using the Sinkhorn-Knopp algorithm~\cite{cuturi2013sinkhorn, NEURIPS2022_bda6843d, chang2023csot}. Specifically, $(\mathbf{F}_m)_{b,q}$ represents the probability that the $b$-th local structural manifold is assigned to the $q$-th global template.

\paragraph{Structural calibration loss} To enforce the calibration of local structural manifolds toward global structural manifolds, we formulate the calibration loss as a soft assignment-weighted optimal transport discrepancy. Let $\mu_{m,b}$ denote the empirical measure induced by the $b$-th local structural manifold on the $m$-th client, and let $\nu_q$ denote the empirical measure associated with the $q$-th global structural manifold. Given the matching probabilities $\mathbf (\mathbf{F}_m)_{b,q}$ over $Q$ manifolds, with $\sum_{q=1}^{Q}\mathbf (\mathbf{F}_m)_{b,q}=1$, we define
\begin{equation}
\mathcal{L}_{\mathrm{str\_calibration}}
=
\frac{1}{MB}
\sum_{m=1}^{M}
\sum_{b=1}^{B}
\sum_{q=1}^{Q}
(\mathbf{F}_m)_{b,q}\,
\mathcal{W}
\left(
\mu_{m,b}, \nu_q
\right),
\label{eq:structural_calibration_loss}
\end{equation}
where $\mathcal W(\cdot,\cdot)$ denotes the optimal transport distance. This structural calibration loss encourages the local structural manifolds to align with the global structural templates while preserving their personalized topological characteristics.

\subsection{Global Manifold Refinement}
While the global semantic and structural manifolds provides the theoretically sound references, they may not perfectly align with the diverse local manifolds due to the non-IID nature of local graphs. To resolve this geometric discrepancy, the server dynamically refines the global semantic and structural manifolds by aggregating local semantic manifolds while preserving the maximal equidistance property.
\subsubsection{Dynamic Refinement of Global Semantic Manifold}
\paragraph{Computing candidate update vectors} After aggregating local semantic manifolds from all clients, the server computes a candidate update vector for each semantic anchor. For class $i$, we first define the global semantic deviation update vector as the mean deviation between local semantic manifolds and the global semantic anchor, \textit{i.e.}, $\mathbf{v}_i = \frac{1}{M} \sum_{m=1}^{M} \left( \mathbf{K}_{m,i} - \bm{\delta}_i \right)$, where $\mathbf{K}_{m,i}$ represents the mean semantic embedding of class $i$ in client $m$ after calibration. To provide the adjustment flexibility for difficult classes, we introduce a difficulty weight, \textit{i.e.}, $\gamma_i = \text{Softmax}\left( \frac{\bar{\mathcal{L}}_i}{\tau} \right) = \frac{\exp\left( \bar{\mathcal{L}}_i / \tau \right)}{\sum_{j=1}^{C} \exp\left( \bar{\mathcal{L}}_j / \tau \right)}$, where $\bar{\mathcal{L}}_i$ is the average semantic alignment loss for class $i$ across all clients, and $\tau$ is a temperature hyperparameter that controls the sharpness of the weight distribution. However, relying solely on data-driven updates would disrupt the strict simplex structure. To maintain the maximal equidistance property, we introduce a manifold constraint vector, namely $\mathbf{s}_i = -\sum_{j \neq i} \frac{\bm{\delta}_j}{\left\| \bm{\delta}_i - \bm{\delta}_j \right\|_2^2}$. This constraint vector ensures that even when adjusting semantic anchors, they continue to mutually repel each other to approximate the theoretical maximal equidistance. By combining the deviation update vector, difficulty weight, and manifold constraint vector, we obtain the candidate update vector for semantic anchor $i$, \textit{i.e.}, $\tilde{\delta}_i = \delta_i + \gamma_i \mathbf{v}_i + \mathbf{s}_i$.
\paragraph{Spherical projection} To prevent oscillations in the updates of semantic anchors, we constrain the update magnitude within a feasible region by introducing an adaptive scaling factor $t_i = \min\left(1, \frac{\eta}{\left\| \tilde{\delta}_i - \delta_i \right\|_2 + \epsilon}\right)$, where $\eta$ is a hyperparameter representing the maximum update step size, and $\epsilon$ is a small constant to prevent division by zero. Finally, we project the candidate update vector back onto the unit hypersphere to complete the dynamic refinement of the semantic anchor:
\begin{equation}
\hat{\bm{\delta}}_i = \frac{\bm{\delta}_i + t_i \left( \tilde{\bm{\delta}}_i - \bm{\delta}_i \right)}{\left\| \bm{\delta}_i + t_i \left( \tilde{\bm{\delta}}_i - \bm{\delta}_i \right) \right\|_2}.
\label{eq:spherical_projection}
\end{equation}
Through this mechanism, the global semantic manifold aggregates local semantic distributions while effectively maintaining the maximal equidistance property endowed by neural collapse theory~\cite{markou2024guiding}.
\subsubsection{Dynamic Refinement of Global Structural Manifold}
After aggregating the radial sequence matrices uploaded by clients, the server dynamically refines the $Q$ global structural templates. For the $q$-th template structural in $\mathcal{T}$, we update it by solving a Gromov-Wasserstein barycenter problem:
\begin{equation}
\mathbf{T}_q^{*}
=
\arg\min_{\mathbf{T}_q}
\sum_{m=1}^{M}
\sum_{b=1}^{B}
(\mathbf{F}_m)_{b,q} \,
\mathrm{GW}
\left(
\mu_{m,b}, \nu_q
\right),
\label{eq:structural_template_update}
\end{equation}
where $\text{GW}(\cdot, \cdot)$ denotes the Gromov-Wasserstein distance~\cite{beier2022linear}, which measures the structural dissimilarity between two metric spaces while being invariant to isometric transformations.

\section{Experiments}
To rigorously validate the effectiveness of our proposed FedGMC, we conduct extensive experiments on eleven widely adopted benchmark datasets that encompass both homophilic and heterophilic graphs. We compare FedGMC against thirteen competitive baseline methods. Following prior work~\cite{wentao2025fediih, yu2026heterogeneity}, we evaluate a total of 66 scenarios by systematically varying the number of clients under both non-overlapping and overlapping partitioning schemes. To ensure fair and robust comparisons, we report the mean classification accuracy together with the standard deviation over ten independent runs. Comprehensive implementation details are provided in Appendix~\ref{implementation_details}.

\subsection{Main Results}
\paragraph{Homophilic datasets} Tab.~\ref{table1} presents the results on homophilic datasets under the non-overlapping partitioning setting. FedGMC achieves the best average performance across all baseline methods and settings. Moreover, FedGMC is also maintaining relatively small standard deviations, which demonstrates both superior effectiveness and greater stability compared with the baselines. Experimental results under the overlapping partitioning setting are reported in Appendix~\ref{additional_tables}.

\paragraph{Heterophilic datasets} Tab.~\ref{table3} summarizes the results on heterophilic datasets under the non-overlapping partitioning setting. FedGMC attains the best average performance across all baselines and outperforms the second-best method (\textit{i.e.}, FedICI) by 1.07\% in classification accuracy. This gain arises because FedGMC explicitly addresses both node feature heterogeneity and structural heterogeneity through semantic and structural manifold calibration. Experimental results under the overlapping partitioning setting are provided in Appendix~\ref{additional_tables}.

\begin{table*}[t]
  \centering
  \scriptsize
  \caption{Accuracy (\%) of methods on six \textbf{homophilic} graph datasets under \textbf{non-overlapping} subgraph partitioning setting.}
    \label{table1}
    \renewcommand{\arraystretch}{0.9} 
       \scalebox{0.75}{
  \begin{tabular}{lcccccccccc}
  \hline
  \rowcolor{gray!50}
  \textbf{}     & \multicolumn{3}{c}{Cora}                                                    & \multicolumn{3}{c}{CiteSeer}                                                & \multicolumn{3}{c}{PubMed}                                                  & -              \\ \cline{2-11}
  Methods       & 5 Clients               & 10 Clients              & 20 Clients              & 5 Clients               & 10 Clients              & 20 Clients              & 5 Clients               & 10 Clients              & 20 Clients              & -              \\ \hline
  \rowcolor{gray!20}
  Local         & 81.30$\pm$0.21          & 79.94$\pm$0.24          & 80.30$\pm$0.25          & 69.02$\pm$0.05          & 67.82$\pm$0.13          & 65.98$\pm$0.17          & 84.04$\pm$0.18          & 82.81$\pm$0.39          & 82.65$\pm$0.03          & -              \\ \hline
  FedAvg~\cite{mcmahan2017communication}        & 74.45$\pm$5.64          & 69.19$\pm$0.67          & 69.50$\pm$3.58          & 71.06$\pm$0.60          & 63.61$\pm$3.59          & 64.68$\pm$1.83          & 79.40$\pm$0.11          & 82.71$\pm$0.29          & 80.97$\pm$0.26          & -              \\
  \rowcolor{gray!20}
  FedProx~\cite{MLSYS2020_1f5fe839}       & 72.03$\pm$4.56          & 60.18$\pm$7.04          & 48.22$\pm$6.18          & 71.73$\pm$1.11          & 63.33$\pm$3.25          & 64.85$\pm$1.35          & 79.45$\pm$0.25          & 82.55$\pm$0.24          & 80.50$\pm$0.25          & -              \\
  FedPer~\cite{Arivazhagan2019}        & 81.68$\pm$0.40          & 79.35$\pm$0.04          & 78.01$\pm$0.32          & 70.41$\pm$0.32          & 70.53$\pm$0.28          & 66.64$\pm$0.27          & 85.80$\pm$0.21          & 84.20$\pm$0.28          & 84.72$\pm$0.31          & -              \\
  \rowcolor{gray!20}
  GCFL~\cite{NEURIPS2021_9c6947bd}          & 81.47$\pm$0.65          & 78.66$\pm$0.27          & 79.21$\pm$0.70          & 70.34$\pm$0.57          & 69.01$\pm$0.12          & 66.33$\pm$0.05          & 85.14$\pm$0.33          & 84.18$\pm$0.19          & 83.94$\pm$0.36          & -              \\
  FedGNN~\cite{wu2021fedgnn}        & 81.51$\pm$0.68          & 70.12$\pm$0.99          & 70.10$\pm$3.52          & 69.06$\pm$0.92          & 55.52$\pm$3.17          & 52.23$\pm$6.00          & 79.52$\pm$0.23          & 83.25$\pm$0.45          & 81.61$\pm$0.59          & -              \\
  \rowcolor{gray!20}
  FedSage+\cite{NEURIPS2021_34adeb8e}      & 72.97$\pm$5.94          & 69.05$\pm$1.59          & 57.97$\pm$12.60          & 70.74$\pm$0.69          & 65.63$\pm$3.10          & 65.46$\pm$0.74          & 79.57$\pm$0.24          & 82.62$\pm$0.31          & 80.82$\pm$0.25          & -              \\
  FED-PUB~\cite{baek2023personalized}       & 83.70$\pm$0.19          & 81.54$\pm$0.12          & 81.75$\pm$0.56          & 72.68$\pm$0.44          & 72.35$\pm$0.53          & 67.62$\pm$0.12          & 86.79$\pm$0.09          & 86.28$\pm$0.18          & 85.53$\pm$0.30          & -              \\
  \rowcolor{gray!20}
  FedGTA~\cite{li2023fedgta}        & 80.06$\pm$0.63          & 80.59$\pm$0.38          & 79.01$\pm$0.31          & 70.12$\pm$0.10          & 71.57$\pm$0.34          & 69.94$\pm$0.14          & 87.75$\pm$0.01          & 86.80$\pm$0.01          & 87.12$\pm$0.05          & -              \\
  AdaFGL~\cite{li2024adafgl}        & 82.01$\pm$0.51          & 80.09$\pm$0.08          & 79.74$\pm$0.05          & 71.44$\pm$0.27          & 72.34$\pm$0.09          & 70.95$\pm$0.45          & 86.91$\pm$0.28          & 86.97$\pm$0.10          & 86.59$\pm$0.21          & -              \\ 
  \rowcolor{gray!20}
  FedTAD~\cite{zhu2024fedtad}       & 80.31$\pm$0.26          & 80.87$\pm$0.11          & 80.07$\pm$0.15          & 70.34$\pm$0.37          & 69.43$\pm$0.75          & 68.09$\pm$0.69          & 84.00$\pm$0.13          & 84.61$\pm$0.17          & 84.33$\pm$0.18          & -              \\ 
  FedIIH~\cite{wentao2025fediih}    & 84.11$\pm$0.17          & 81.85$\pm$0.09          & 83.01$\pm$0.15          & 72.86$\pm$0.25          & 76.50$\pm$0.06          & 73.36$\pm$0.41          & 87.80$\pm$0.18          & 87.65$\pm$0.18          & 87.19$\pm$0.25         & -              \\
  \rowcolor{gray!20}
  FedICI~\cite{yuintegrating} & \textbf{84.99$\pm$0.04}          & \underline{83.38$\pm$0.05}          & 84.04$\pm$0.08          & \underline{73.77$\pm$0.14}          & 77.45$\pm$0.07          & 73.65$\pm$0.29          & \textbf{88.72$\pm$0.05}          & \underline{87.93$\pm$0.11}          & \underline{87.61$\pm$0.11}         & -              \\
  FedSSA~\cite{yu2026heterogeneity}                 & \underline{84.67$\pm$0.05} & 82.32$\pm$0.04 & \underline{84.13$\pm$0.09} & 73.06$\pm$0.08 & \textbf{77.65$\pm$0.10} & \underline{74.09$\pm$0.09} & 88.11$\pm$0.07 & 87.78$\pm$0.13 & 87.37$\pm$0.14 & -              \\ \hline
  \rowcolor{yellow!30}
  FedGMC (Ours)                 & 84.16$\pm$0.11 & \textbf{83.65$\pm$0.16} & \textbf{84.55$\pm$0.13} & \textbf{73.88$\pm$0.11} & \underline{77.59$\pm$0.06} & \textbf{74.35$\pm$0.10} & \underline{88.39$\pm$0.09} & \textbf{88.67$\pm$0.12} & \textbf{88.11$\pm$0.11} & -              \\ \hline
  \rowcolor{gray!50}
  & \multicolumn{3}{c}{Amazon-Computer}                                         & \multicolumn{3}{c}{Amazon-Photo}                                            & \multicolumn{3}{c}{ogbn-arxiv}                                              & Avg.            \\ \cline{2-11} 
  Methods       & 5 Clients               & 10 Clients              & 20 Clients              & 5 Clients               & 10 Clients              & 20 Clients              & 5 Clients               & 10 Clients              & 20 Clients              & All           \\ \hline
  \rowcolor{gray!20}
  Local         & 89.22$\pm$0.13          & 88.91$\pm$0.17          & 89.52$\pm$0.20          & 91.67$\pm$0.09          & 91.80$\pm$0.02          & 90.47$\pm$0.15          & 66.76$\pm$0.07          & 64.92$\pm$0.09          & 65.06$\pm$0.05          & 79.57          \\ \hline
  FedAvg~\cite{mcmahan2017communication}        & 84.88$\pm$1.96          & 79.54$\pm$0.23          & 74.79$\pm$0.24          & 89.89$\pm$0.83          & 83.15$\pm$3.71          & 81.35$\pm$1.04          & 65.54$\pm$0.07          & 64.44$\pm$0.10          & 63.24$\pm$0.13          & 74.58          \\
  \rowcolor{gray!20}
  FedProx~\cite{MLSYS2020_1f5fe839}       & 85.25$\pm$1.27          & 83.81$\pm$1.09          & 73.05$\pm$1.30          & 90.38$\pm$0.48          & 80.92$\pm$4.64          & 82.32$\pm$0.29          & 65.21$\pm$0.20          & 64.37$\pm$0.18          & 63.03$\pm$0.04          & 72.84          \\
  FedPer~\cite{Arivazhagan2019}        & 89.67$\pm$0.34          & 89.73$\pm$0.04          & 87.86$\pm$0.43          & 91.44$\pm$0.37          & 91.76$\pm$0.23          & 90.59$\pm$0.06          & 66.87$\pm$0.05          & 64.99$\pm$0.18          & 64.66$\pm$0.11          & 79.94          \\
  \rowcolor{gray!20}
  GCFL~\cite{NEURIPS2021_9c6947bd}          & 89.07$\pm$0.91          & 90.03$\pm$0.16          & 89.08$\pm$0.25          & 91.99$\pm$0.29          & 92.06$\pm$0.25          & 90.79$\pm$0.17          & 66.80$\pm$0.12          & 65.09$\pm$0.08          & 65.08$\pm$0.04          & 79.90          \\
  FedGNN~\cite{wu2021fedgnn}        & 88.08$\pm$0.15          & 88.18$\pm$0.41          & 83.16$\pm$0.13          & 90.25$\pm$0.70          & 87.12$\pm$2.01          & 81.00$\pm$4.48          & 65.47$\pm$0.22          & 64.21$\pm$0.32          & 63.80$\pm$0.05          & 75.23          \\
  \rowcolor{gray!20}
  FedSage+\cite{NEURIPS2021_34adeb8e}      & 85.04$\pm$0.61          & 80.50$\pm$1.13          & 70.42$\pm$0.85          & 90.77$\pm$0.44          & 76.81$\pm$8.24          & 80.58$\pm$1.15          & 65.69$\pm$0.09          & 64.52$\pm$0.14          & 63.31$\pm$0.20          & 73.47          \\
  FED-PUB~\cite{baek2023personalized}       & 90.74$\pm$0.05 & 90.55$\pm$0.13          & 90.12$\pm$0.09          & 93.29$\pm$0.19          & 92.73$\pm$0.18          & 91.92$\pm$0.12          & 67.77$\pm$0.09          & 66.58$\pm$0.08          & 66.64$\pm$0.12          & 81.59          \\
  \rowcolor{gray!20}
  FedGTA~\cite{li2023fedgta}        & 86.69$\pm$0.18          & 86.66$\pm$0.23          & 85.01$\pm$0.87          & 93.33$\pm$0.12          & 93.50$\pm$0.21          & 92.61$\pm$0.15          & 60.32$\pm$0.04                   & 60.22$\pm$0.09                   & 58.74$\pm$0.14                         & 79.45          \\ 
  AdaFGL~\cite{li2024adafgl}        & 80.20$\pm$0.05          & 83.62$\pm$0.26          & 84.53$\pm$0.23          & 86.69$\pm$0.19          & 89.85$\pm$0.83          & 88.11$\pm$0.05          & 52.73$\pm$0.19          & 51.77$\pm$0.36          & 50.94$\pm$0.08          & 76.97 \\     
  \rowcolor{gray!20}
  FedTAD~\cite{zhu2024fedtad}       & 82.20$\pm$1.20          & 85.50$\pm$0.33          & 83.91$\pm$1.54          & 92.29$\pm$0.39          & 90.59$\pm$0.09          & 89.18$\pm$0.84          & 65.35$\pm$0.14          & 64.06$\pm$0.25          & 64.45$\pm$0.13          & 78.87 \\    
  FedIIH~\cite{wentao2025fediih} & 90.74$\pm$0.13          & 90.86$\pm$0.23 & 90.44$\pm$0.05 & 93.42$\pm$0.02 & 94.22$\pm$0.08 & 93.55$\pm$0.09 & 70.30$\pm$0.06 & 69.34$\pm$0.02 & 68.65$\pm$0.04 & 83.10 \\
  \rowcolor{gray!20}
  FedICI~\cite{yuintegrating} & 91.07$\pm$0.12          & 91.12$\pm$0.10          & 90.37$\pm$0.08          & 93.61$\pm$0.07          & 94.34$\pm$0.17          & 93.59$\pm$0.12          & 70.81$\pm$0.09          & \underline{69.60$\pm$0.06}          & \underline{68.79$\pm$0.11}         & \underline{83.60}              \\
  FedSSA~\cite{yu2026heterogeneity}                & \underline{91.09$\pm$0.07}          & \underline{91.30$\pm$0.13} & \underline{90.62$\pm$0.09} & \underline{93.86$\pm$0.15} & \underline{94.62$\pm$0.14} & \underline{93.76$\pm$0.07} & \underline{70.86$\pm$0.13} & 69.47$\pm$0.08 & 68.77$\pm$0.13 & 83.53 \\ \hline  
  \rowcolor{yellow!30}
  FedGMC (Ours)                  & \textbf{91.38$\pm$0.10}          & \textbf{91.85$\pm$0.14} & \textbf{90.92$\pm$0.09} & \textbf{93.98$\pm$0.06} & \textbf{94.75$\pm$0.07} & \textbf{94.51$\pm$0.10} & \textbf{70.95$\pm$0.08} & \textbf{69.68$\pm$0.15} & \textbf{68.88$\pm$0.13} & \textbf{83.90} \\ \hline  
  \end{tabular}
  }
  \end{table*}

\begin{table*}[t]
\centering
\scriptsize
\caption{Comparisons on five \textbf{heterophilic} graph datasets under \textbf{non-overlapping} subgraph partitioning setting. Accuracy (\%) is reported for \textit{Roman-empire} and \textit{Amazon-ratings}, and AUC (\%) is reported for \textit{Minesweeper}, \textit{Tolokers}, and \textit{Questions}.}
\label{table3}
\renewcommand{\arraystretch}{0.9} 
  \scalebox{0.75}{
  \begin{tabular}{lcccccccccc}
  \hline
  \rowcolor{gray!50}
  \textbf{}     & \multicolumn{3}{c}{Roman-empire}                                            & \multicolumn{3}{c}{Amazon-ratings}                                          & \multicolumn{3}{c}{Minesweeper}                                             & -              \\ \cline{2-11} 
  Methods       & 5 Clients               & 10 Clients              & 20 Clients              & 5 Clients               & 10 Clients              & 20 Clients              & 5 Clients               & 10 Clients              & 20 Clients              & -              \\ \hline
  \rowcolor{gray!20}
  Local         & 33.65$\pm$0.13          & 28.42$\pm$0.26          & 23.89$\pm$0.32          & 45.03$\pm$0.31 & \underline{45.89$\pm$0.19} & 46.02$\pm$0.02 & 71.35$\pm$0.17          & 69.96$\pm$0.16          & 69.31$\pm$0.09          & -              \\ \hline
  FedAvg~\cite{mcmahan2017communication}        & 38.93$\pm$0.32          & 35.43$\pm$0.32          & 32.00$\pm$0.39          & 41.26$\pm$0.53          & 41.66$\pm$0.14          & 42.20$\pm$0.21          & 72.60$\pm$0.08          & 69.68$\pm$0.09          & 71.36$\pm$0.16          & -              \\
  \rowcolor{gray!20}
  FedProx~\cite{MLSYS2020_1f5fe839}       & 27.95$\pm$0.59          & 26.43$\pm$1.41          & 23.12$\pm$0.49          & 36.92$\pm$0.02          & 36.86$\pm$0.14          & 36.96$\pm$0.05          & 71.91$\pm$0.27          & 70.66$\pm$0.20          & 71.50$\pm$0.37          & -              \\
  FedPer~\cite{Arivazhagan2019}        & 20.75$\pm$1.75          & 15.51$\pm$1.13          & 15.45$\pm$2.76          & 36.62$\pm$0.30          & 32.34$\pm$1.01          & 36.96$\pm$0.03          & 58.73$\pm$10.45         & 65.35$\pm$7.02          & 53.80$\pm$11.40         & -              \\
  \rowcolor{gray!20}
  GCFL~\cite{NEURIPS2021_9c6947bd}          & 40.65$\pm$0.14          & 40.51$\pm$0.24          & 37.85$\pm$0.25          & 36.92$\pm$0.05          & 36.86$\pm$0.14          & 36.96$\pm$0.02          & 72.04$\pm$0.13          & 71.88$\pm$0.12          & 69.20$\pm$0.18          & -              \\
  FedGNN~\cite{wu2021fedgnn}        & 30.26$\pm$0.11          & 29.09$\pm$0.13          & 26.60$\pm$0.02          & 36.80$\pm$0.06          & 36.72$\pm$0.15          & 36.45$\pm$0.09          & 72.15$\pm$0.13          & 71.08$\pm$0.07          &  71.71$\pm$0.27    & -              \\
  \rowcolor{gray!20}
  FedSage+\cite{NEURIPS2021_34adeb8e}      & 57.26$\pm$0.13          & 49.07$\pm$0.17          & 38.36$\pm$0.12          & 36.82$\pm$0.19          & 36.71$\pm$0.20          & 37.03$\pm$0.02          & 77.74$\pm$0.26 & 72.80$\pm$0.27    & 69.70$\pm$0.23          & -              \\
  FED-PUB~\cite{baek2023personalized}       & 40.80$\pm$0.26          & 36.77$\pm$0.30          & 32.67$\pm$0.39          & 44.41$\pm$0.41    &  44.85$\pm$0.17    & 45.39$\pm$0.50    & 72.18$\pm$0.02          & 71.69$\pm$0.71          & 71.41$\pm$0.87          & -              \\
  \rowcolor{gray!20}
  FedGTA~\cite{li2023fedgta}        & 61.56$\pm$0.27    & 60.94$\pm$0.19    & 59.65$\pm$0.28    & 41.22$\pm$0.66          & 39.40$\pm$0.44          & 39.24$\pm$0.12          & 73.54$\pm$1.56          & 72.65$\pm$1.21          & 69.63$\pm$4.54          & -              \\
  AdaFGL~\cite{li2024adafgl}        & 67.64$\pm$0.18          & 64.55$\pm$0.09          & 62.42$\pm$0.26          & 41.70$\pm$0.06          & 42.30$\pm$0.22          & 42.59$\pm$0.14          & 73.24$\pm$1.13                  & 70.79$\pm$2.14                   & 71.26$\pm$1.31                         & -          \\ 
  \rowcolor{gray!20}
  FedTAD~\cite{zhu2024fedtad} & 45.26$\pm$0.19          & 44.71$\pm$0.38          & 42.04$\pm$0.13          & 43.59$\pm$0.33          & 43.35$\pm$0.29          & 44.50$\pm$0.26          & 72.39$\pm$0.43                  & 71.99$\pm$0.13                   & 72.74$\pm$0.03                         & -          \\ 
  FedIIH~\cite{wentao2025fediih} & 68.32$\pm$0.05 & 66.44$\pm$0.28 & 64.61$\pm$0.13 & 44.26$\pm$0.24          & 44.24$\pm$0.10          & 45.19$\pm$0.04          & 74.29$\pm$0.02    & 73.23$\pm$0.04 & 72.81$\pm$0.02 &   -             \\
  \rowcolor{gray!20}
  FedICI~\cite{yuintegrating} & \underline{68.96$\pm$0.08}          & \underline{67.92$\pm$0.11}          & \underline{66.16$\pm$0.10}          & \underline{45.97$\pm$0.06}          & 45.36$\pm$0.08          & 46.07$\pm$0.08          & \underline{86.35$\pm$0.08}          & \underline{85.32$\pm$0.07}          & \underline{84.25$\pm$0.09}         & -              \\
  FedSSA~\cite{yu2026heterogeneity}    & 68.67$\pm$0.10           & 66.81$\pm$0.09 & 65.14$\pm$0.15 & 45.18$\pm$0.14          & 45.11$\pm$0.15          & \underline{46.13$\pm$0.05}          & 82.26$\pm$0.14    & 82.16$\pm$0.08 & 82.60$\pm$0.11 & -               \\ \hline 
  \rowcolor{yellow!30}
FedGMC (Ours) & \textbf{71.95$\pm$0.07} & \textbf{70.71$\pm$0.15} & \textbf{66.99$\pm$0.17} & \textbf{46.11$\pm$0.10}          & \textbf{46.09$\pm$0.12}          & \textbf{46.56$\pm$0.08}          & \textbf{86.79$\pm$0.07}    & \textbf{86.04$\pm$0.08} & \textbf{84.84$\pm$0.10} & -               \\ \hline
\rowcolor{gray!50}
                & \multicolumn{3}{c}{Tolokers}                                                & \multicolumn{3}{c}{Questions}                                               & \multicolumn{4}{c}{Avg.}                                                                     \\ \cline{2-11} 
  Methods       & 5 Clients               & 10 Clients              & 20 Clients              & 5 Clients               & 10 Clients              & 20 Clients              & 5 Clients               & 10 Clients              & 20 Clients              & All            \\ \hline
  \rowcolor{gray!20}
  Local         & 67.81$\pm$0.17          & 70.04$\pm$0.23          & 62.34$\pm$0.67          & 66.73$\pm$0.57          & 57.96$\pm$0.10          & 60.11$\pm$0.21          & 56.91                   & 54.45                   & 52.31                   & 54.56          \\ \hline
  FedAvg~\cite{mcmahan2017communication}        & 60.74$\pm$0.31          & 54.73$\pm$0.50          & 56.36$\pm$0.39          & 65.68$\pm$0.23          & 58.91$\pm$0.22          & 60.33$\pm$0.15          & 55.84                   & 52.08                   & 52.45                   & 53.46          \\
  \rowcolor{gray!20}
  FedProx~\cite{MLSYS2020_1f5fe839}       & 42.90$\pm$0.24          & 41.15$\pm$0.22          & 40.42$\pm$0.62          & 47.36$\pm$0.38          & 45.46$\pm$0.34          & 46.83$\pm$0.11          & 45.41                   & 44.11                   & 43.77                   & 44.43          \\
  FedPer~\cite{Arivazhagan2019}        & 46.61$\pm$9.88          & 54.97$\pm$13.23         & 44.82$\pm$11.61         & 58.38$\pm$9.39          & 59.40$\pm$9.71          & 62.32$\pm$1.56          & 44.22                   & 45.51                   & 42.67                   & 44.13          \\
  \rowcolor{gray!20}
  GCFL~\cite{NEURIPS2021_9c6947bd}          & 64.39$\pm$1.17          & 59.90$\pm$0.85          & 58.82$\pm$0.70          & 60.51$\pm$1.18          & 59.85$\pm$0.16          & 60.31$\pm$0.48          & 54.90                   & 53.80                   & 52.63                   & 53.78          \\
  FedGNN~\cite{wu2021fedgnn}        & 43.10$\pm$0.27          & 41.57$\pm$0.07          & 40.70$\pm$0.74          & 47.55$\pm$0.02          & 45.65$\pm$0.12          & 47.39$\pm$0.13          & 45.97                   & 44.82                   & 44.57                   & 45.12          \\
  \rowcolor{gray!20}
  FedSage+\cite{NEURIPS2021_34adeb8e}      & 75.06$\pm$0.16 & 71.31$\pm$0.14          &  69.73$\pm$0.26    & 64.95$\pm$0.12          &  65.06$\pm$0.27    & 59.33$\pm$0.28          & 62.37            &  58.99             & 54.83                   &  58.73    \\
  FED-PUB~\cite{baek2023personalized}       & 70.88$\pm$0.58          & 72.46$\pm$0.68 & 65.26$\pm$0.59          &  67.71$\pm$3.99    & 59.64$\pm$0.52          & 62.48$\pm$2.92    & 59.20                   & 57.08                   &  55.44             & 57.24          \\
  \rowcolor{gray!20}
  FedGTA~\cite{li2023fedgta}        & 60.83$\pm$0.45          & 55.18$\pm$1.20          & 57.89$\pm$1.61          & 65.56$\pm$1.91          & 58.29$\pm$1.57          & 61.70$\pm$0.35          & 60.54                   & 57.29                   & 57.62                   & 58.49          \\
  AdaFGL~\cite{li2024adafgl}        & 59.26$\pm$2.18          & 54.78$\pm$2.12          & 56.61$\pm$2.93          & 64.23$\pm$2.09          & 58.82$\pm$1.14          & 62.84$\pm$0.49          & 61.21                   & 58.25                   & 59.14                         & 59.54          \\ 
  \rowcolor{gray!20} 
  FedTAD~\cite{zhu2024fedtad} & 60.91$\pm$0.25          & 53.39$\pm$1.73          & 56.47$\pm$1.58          & 68.89$\pm$1.20          & 58.44$\pm$1.06          & 61.51$\pm$1.56          & 58.21                  & 54.38                   & 55.45                         & 56.01         \\
  FedIIH~\cite{wentao2025fediih} &  71.09$\pm$0.26    & 71.32$\pm$0.09    & 70.30$\pm$0.10 & 68.32$\pm$0.03 & 67.99$\pm$0.09 & 65.40$\pm$0.07 & 65.26          & 64.64          & 63.66          & 64.52 \\
  \rowcolor{gray!20}
  FedICI~\cite{yuintegrating} & 74.05$\pm$0.10          & 73.74$\pm$0.13          & 70.79$\pm$0.08          & 69.44$\pm$0.15          & 68.54$\pm$0.16          & 65.67$\pm$0.07          & \underline{68.95}          & \underline{68.18}          & \underline{66.59}         & \underline{67.88}              \\
  FedSSA~\cite{yu2026heterogeneity}                & \underline{75.82$\pm$0.05} & \underline{73.96$\pm$0.10} & \underline{72.29$\pm$0.11} & \underline{69.51$\pm$0.15}          & \underline{68.69$\pm$0.18}          & \underline{65.74$\pm$0.07}          & 68.29    & 67.35 & 66.10 & 67.34               \\ \hline
  \rowcolor{yellow!30}
FedGMC (Ours) & \textbf{76.49$\pm$0.15} & \textbf{74.03$\pm$0.09} & \textbf{72.33$\pm$0.06} & \textbf{69.73$\pm$0.08}          & \textbf{69.72$\pm$0.14}          & \textbf{65.81$\pm$0.12}          & \textbf{70.21}    & \textbf{69.32} & \textbf{67.31} & \textbf{68.95}               \\ \hline
      \end{tabular}
       }
\end{table*}
  
\subsection{Ablation Study}
To shed light on the contributions of components in our FedGMC, we perform ablation studies on \textit{Cora} and \textit{Roman-empire} datasets, which are shown in Fig.~\ref{fig-ablation}. Specifically, we employ `w/o Semantic', `w/o Structural', and `w/o Reﬁnement' to represent the reduced methods by removing `semantic manifold calibration', `structural manifold calibration', and `global manifold reﬁnement', respectively. As shown in Fig.~\ref{fig-ablation}, disabling any one of these components consistently leads to a clear performance degradation, which confirms that each component plays a vital role. Notably, removing structural manifold calibration alone causes the accuracy on \textit{Cora} to drop by more than 8\%. Ablation studies on other datasets are presented in~\ref{ap_ablation_study}.


\begin{figure}[t]
  \centering
  \begin{minipage}[t]{0.49\linewidth}
    \centering
    \subfloat[\footnotesize{\textit{Cora}}]{\includegraphics[width=0.5\columnwidth]{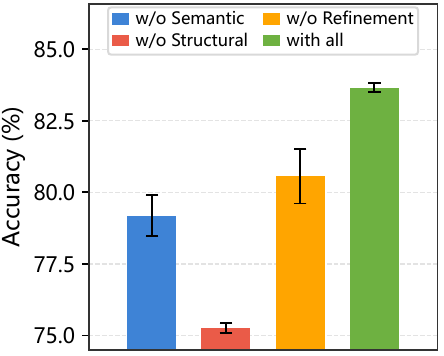}}
    \hfill
    \subfloat[\footnotesize{\textit{Roman-empire}}]{\includegraphics[width=0.5\columnwidth]{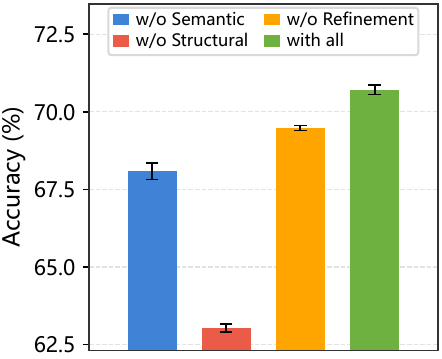}}
    \caption{Ablation studies under non-overlapping partitioning setting with 10 clients.}
    \label{fig-ablation}
  \end{minipage}
    \begin{minipage}[t]{0.49\linewidth}
    \centering
    \subfloat[\footnotesize{\textit{PubMed}}]{\includegraphics[width=0.5\columnwidth]{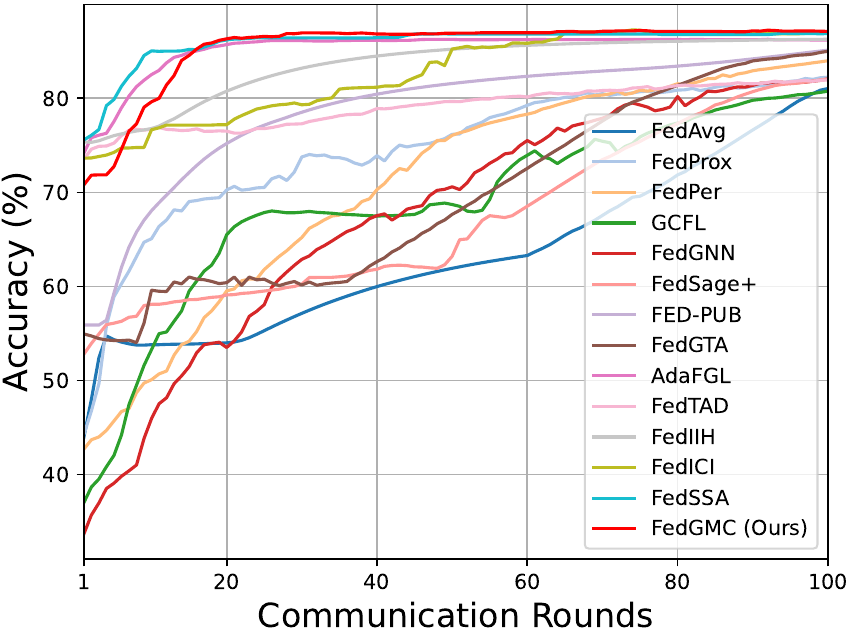}\label{fig5_1}}
    \hfill
    \subfloat[\footnotesize{\textit{Tolokers}}]{\includegraphics[width=0.5\columnwidth]{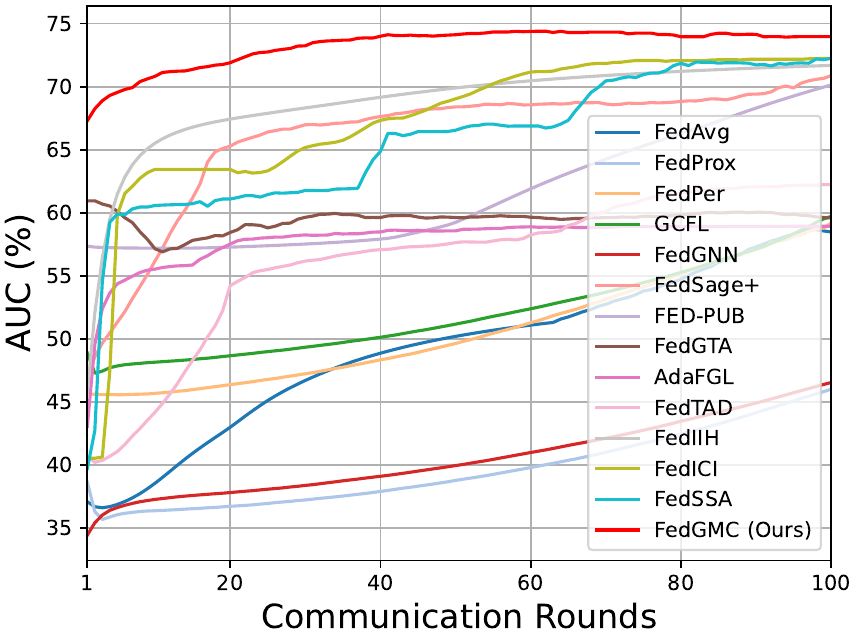}\label{fig5_2}}
    \caption{Convergence curves on two datasets under overlapping partitioning setting with 30 clients.}
    \label{fig5}
  \end{minipage}
\end{figure}

\subsection{Convergence Curves}
\label{Convergence_Curves}
To evaluate the convergence behavior of our proposed FedGMC and baseline methods, we plot the convergence curves in Fig.~\ref{fig5}. The convergence curves show that our FedGMC exhibits only minor fluctuations, which confirms its strong stability. Moreover, our FedGMC converges rapidly within a small number of communication rounds, which highlights its practical efficiency for real-world applications. For example, FedGMC achieves convergence after an average of only 40 communication rounds, while baseline methods like FedIIH need around 80 communication rounds on average. This faster convergence property is due to our semantic manifold calibration and structural manifold calibration, which effectively reduce the impact of heterogeneity. Additional convergence curves are available in Appendix~\ref{ap_convergence_curves}.

\subsection{Sensitivity Analysis on Hyperparameters}
Here we perform a comprehensive sensitivity analysis of hyperparameters involved in our proposed FedGMC. Since our FedGMC includes four key hyperparameters (\textit{i.e.}, the number of global structural manifold $Q$, the number of local structural manifold $B$, temperature hyperparameter $\tau$, and step size $\eta$), we plot accuracy curves accompanied by variance bars under different values of hyperparameters on \textit{Amazon-Computer} dataset. As shown in Fig.~\ref{fig6}, our FedGMC exhibits only minor performance variations under different hyperparameter values, which demonstrates the robustness of FedGMC to hyperparameter values. More sensitivity analyses are provided in Appendix~\ref{ap_sensitivity_analysis}.

\begin{figure}[]
  \centering
  \subfloat[\footnotesize{$Q$}]{\includegraphics[width=0.25\columnwidth]{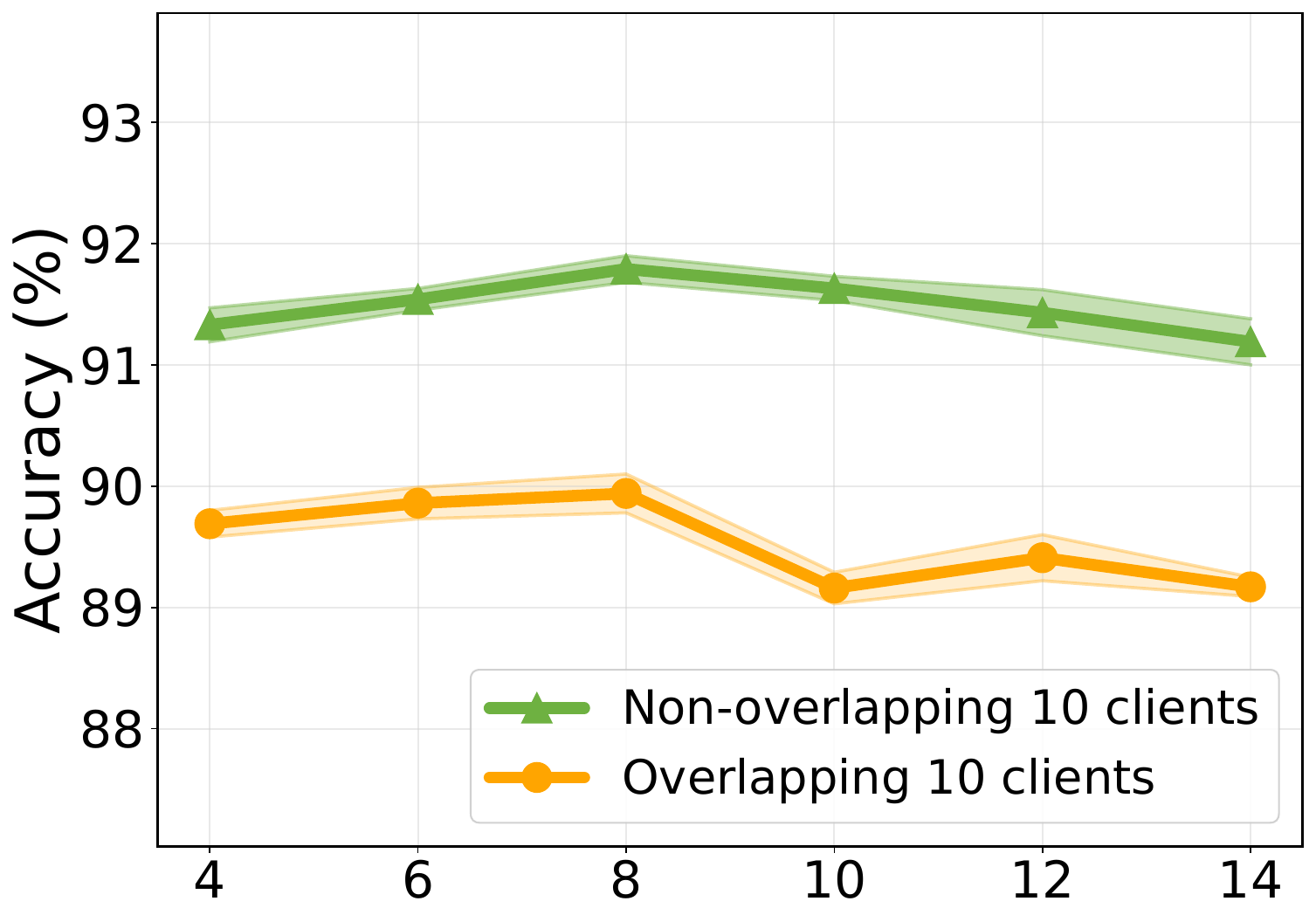}\label{fig6_1}}
  \hfill
  \subfloat[\footnotesize{$B$}]{\includegraphics[width=0.25\columnwidth]{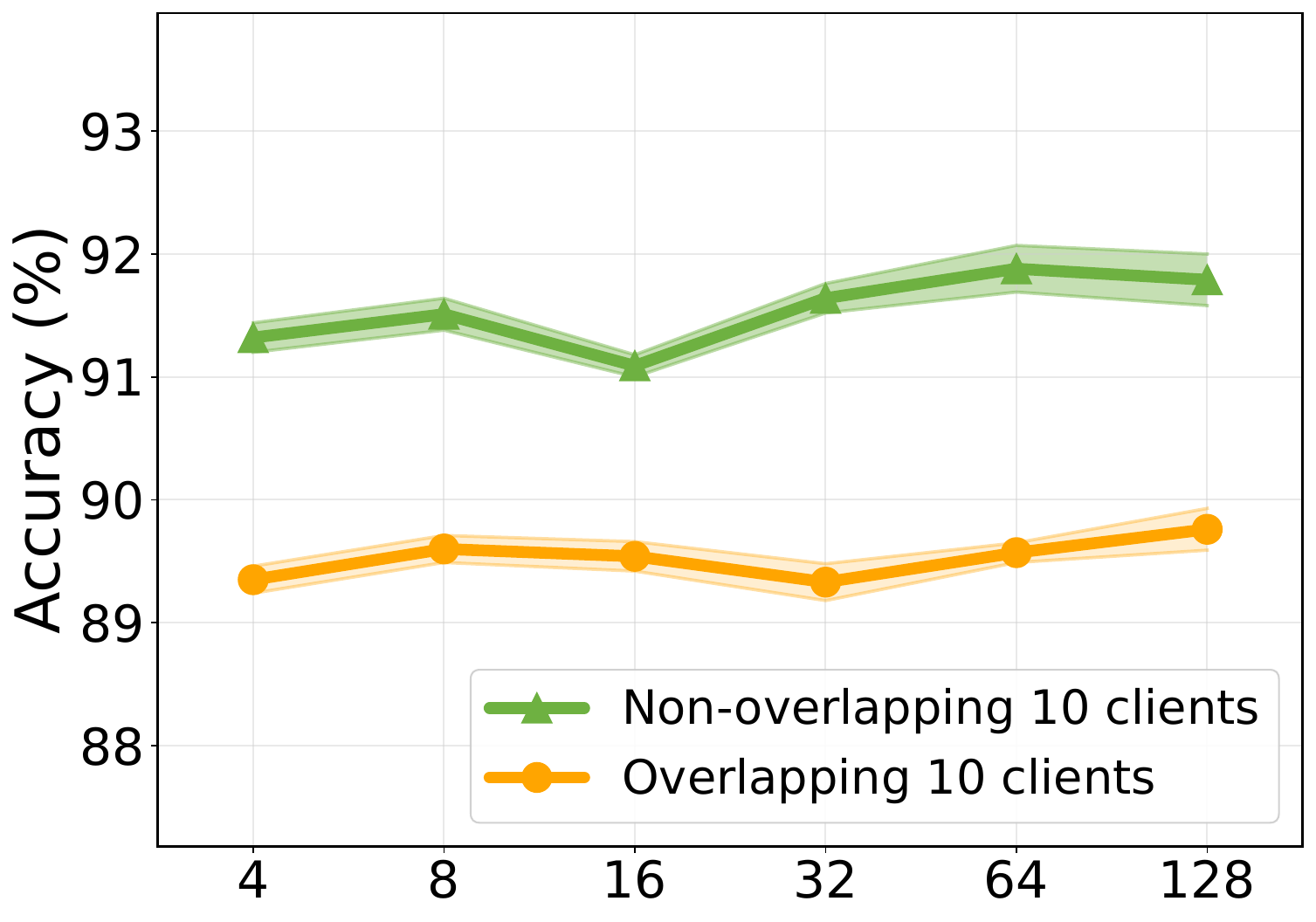}\label{fig6_2}}
  \hfill
  \subfloat[\footnotesize{$\tau$}]{\includegraphics[width=0.25\columnwidth]{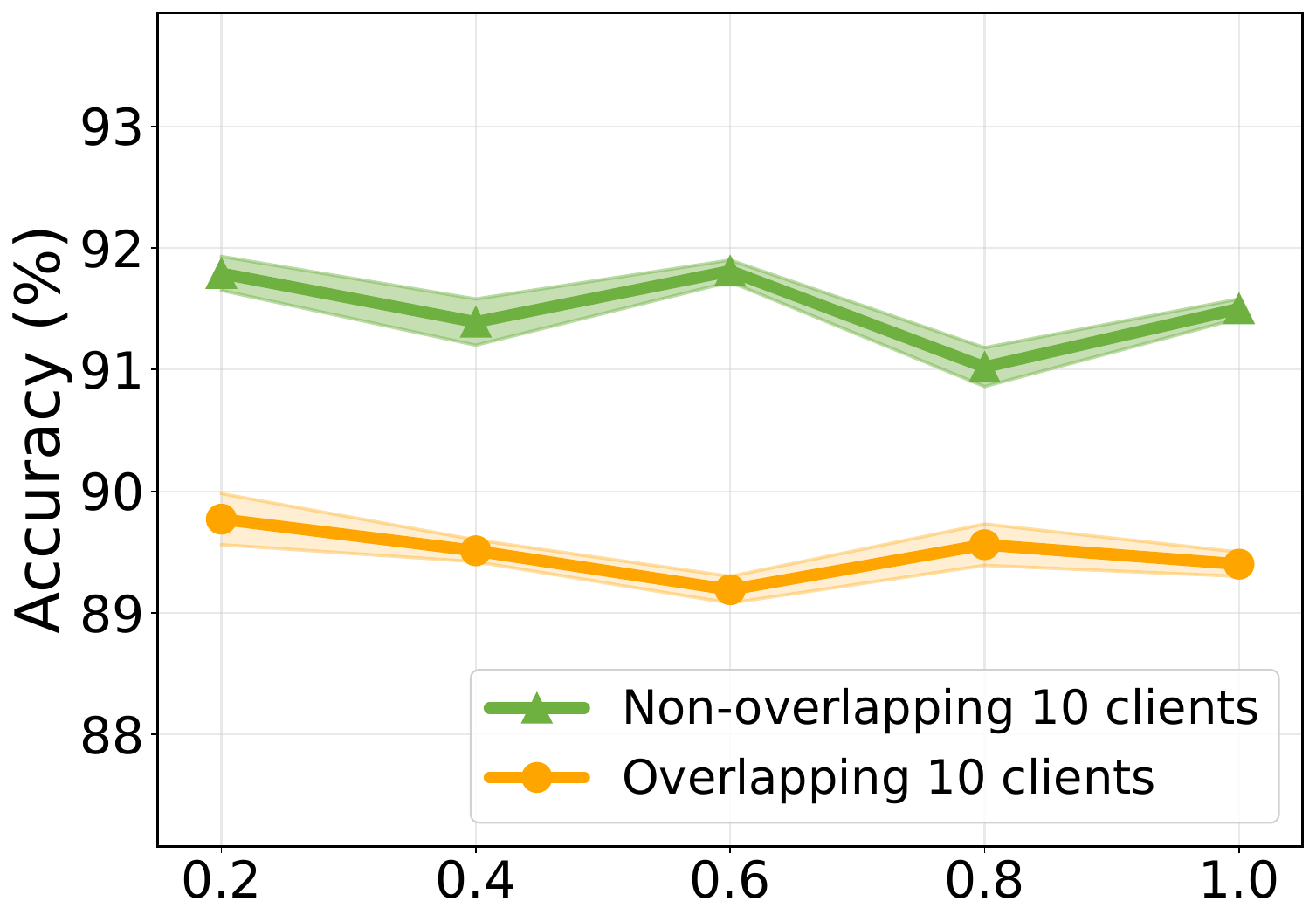}\label{fig6_3}}
  \hfill
    \subfloat[\footnotesize{$\eta$}]{\includegraphics[width=0.25\columnwidth]{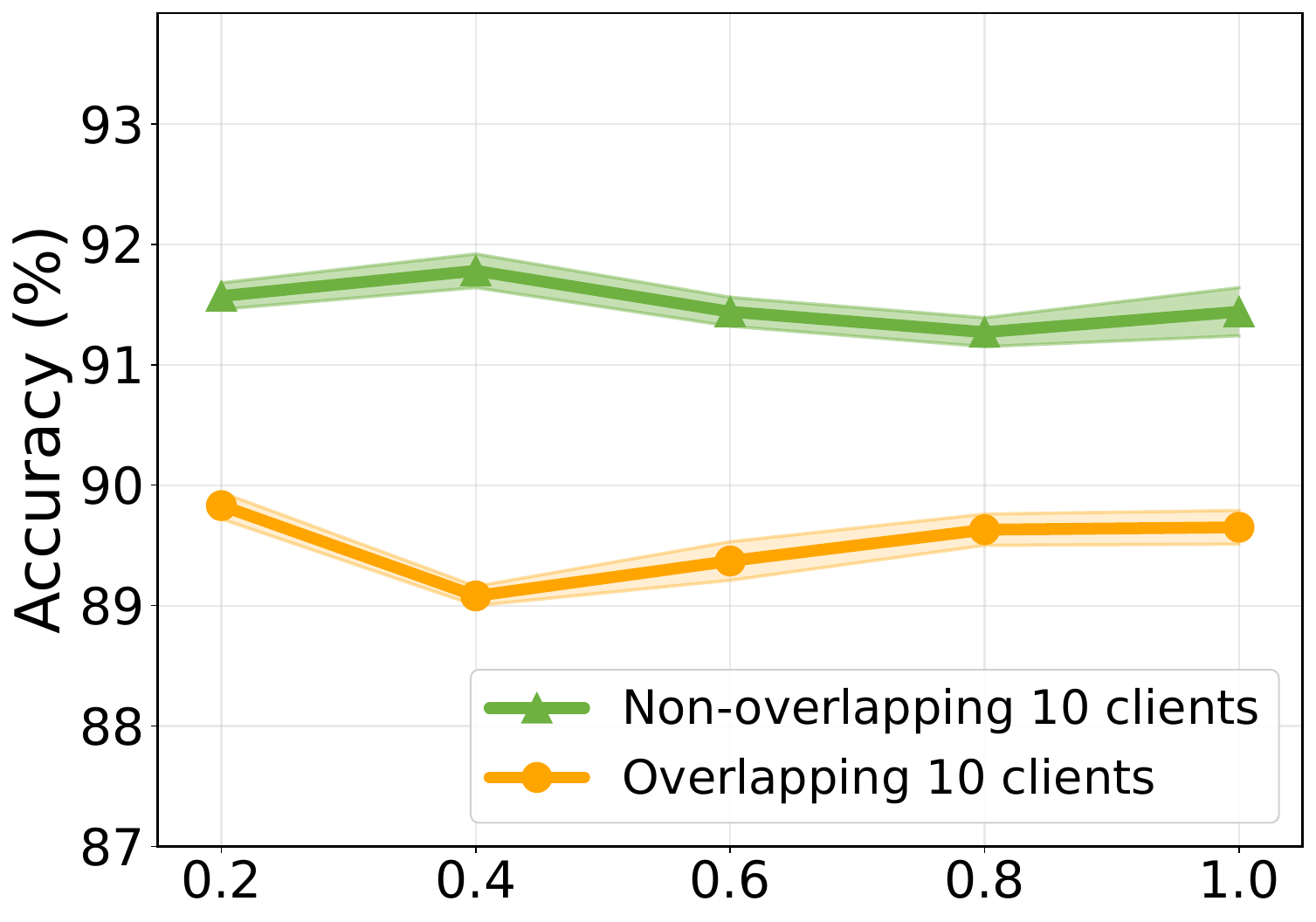}\label{fig6_4}}
  \caption{Accuracy curves accompanied by standard deviation bands on \textit{Amazon-Computer} dataset under different values of $Q$, $B$, $\tau$, and $\eta$.}
  \label{fig6}
\end{figure}

\subsection{Case Study}
To illustrate the effectiveness of FedGMC in manifold calibration, we conduct case studies on semantic and structural manifolds using \textit{Cora} dataset under the non-overlapping partitioning setting with 10 clients. We first visualize the global semantic manifolds of various clients via t-SNE~\cite{van2008visualizing}. As illustrated in Fig.~\ref{fig-case-study}(a), the 2D projections reveal that the global semantic manifolds (\textit{i.e.}, purple stars) are positioned with maximal equidistance across the sphere, which demonstrates the effectiveness of FedGMC in mitigating semantic heterogeneity. After global refinement (\textit{i.e.}, Fig.~\ref{fig-case-study}(b)), the local semantic embeddings (\textit{i.e.}, colored points) show only minor deviations from equidistant anchors while preserving the intrinsic geometric structure. Similarly, the global structural manifolds in Fig.~\ref{fig-case-study}(c) form compact clusters, which validates the effectiveness of addressing structural heterogeneity. After refinement (\textit{i.e.}, Fig.~\ref{fig-case-study}(d)), these clusters exhibit reduced variations yet retain topological structures. More case studies are provided in Appendix~\ref{ap_case_study}.

\begin{figure}[t]
	\centering
	\includegraphics[width=13.5cm]{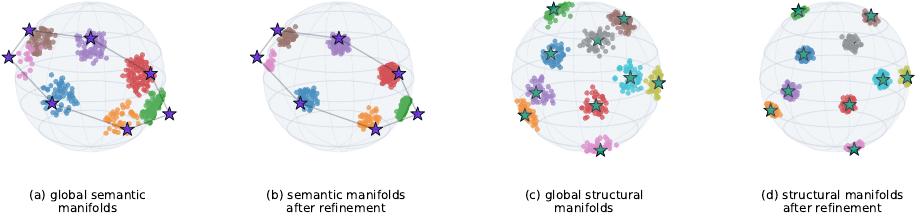}
	\caption{Case studies on \textit{Cora} dataset under non-overlapping partitioning setting with 10 clients.}
	\label{fig-case-study}
\end{figure}

\section{Conclusion}\label{conclusion}
In this paper, we propose \underline{\textbf{Fed}}erated \underline{\textbf{G}}raph \underline{\textbf{M}}anifold \underline{\textbf{C}}alibration (FedGMC) to address heterogeneity in Graph Federated Learning (GFL). Instead of enforcing rigid alignment, FedGMC introduces a dual manifold calibration mechanism that preserves global commonalities while maximizing the personalized representation space of local clients. For semantic heterogeneity, we construct a geometrically optimal semantic manifold via equidistant semantic anchors, so as to guide the calibration of local semantic manifolds. For structural heterogeneity, we construct a global structural manifold by building global structural templates, so as to guide the calibration of local structural manifolds. Since FedGMC effectively balances global commonality and local personalization, it outperforms thirteen state-of-the-art methods on eleven datasets.

\bibliographystyle{unsrt}
\bibliography{cite}

\newpage
\section*{NeurIPS Paper Checklist}

\begin{enumerate}

\item {\bf Claims}
    \item[] Question: Do the main claims made in the abstract and introduction accurately reflect the paper's contributions and scope?
    \item[] Answer: \answerYes{} 
    \item[] Justification: We have summarized the position and key contributions of the paper in the abstract and introduction parts.

\item {\bf Limitations}
    \item[] Question: Does the paper discuss the limitations of the work performed by the authors?
    \item[] Answer: \answerYes{} 
    \item[] Justification: The limitations are discussed in Appendix~\ref{appendix: further discussion}.

\item {\bf Theory assumptions and proofs}
    \item[] Question: For each theoretical result, does the paper provide the full set of assumptions and a complete (and correct) proof?
    \item[] Answer: \answerYes{} 
    \item[] Justification: The assumptions and proof are provided in Appendix~\ref{Theoretical_Analysis}.

    \item {\bf Experimental result reproducibility}
    \item[] Question: Does the paper fully disclose all the information needed to reproduce the main experimental results of the paper to the extent that it affects the main claims and/or conclusions of the paper (regardless of whether the code and data are provided or not)?
    \item[] Answer: \answerYes{} 
    \item[] Justification: The technical details of implementation are introduced in Appendix~\ref{implementation_details}.

\item {\bf Open access to data and code}
    \item[] Question: Does the paper provide open access to the data and code, with sufficient instructions to faithfully reproduce the main experimental results, as described in supplemental material?
    \item[] Answer: \answerYes{} 
    \item[] Justification: The source files are publicly available.

\item {\bf Experimental setting/details}
    \item[] Question: Does the paper specify all the training and test details (e.g., data splits, hyperparameters, how they were chosen, type of optimizer) necessary to understand the results?
    \item[] Answer: \answerYes{} 
    \item[] Justification: The experiment setups and the technical details of implementation are introduced in Appendix~\ref{implementation_details}. 

\item {\bf Experiment statistical significance}
    \item[] Question: Does the paper report error bars suitably and correctly defined or other appropriate information about the statistical significance of the experiments?
    \item[] Answer: \answerYes{} 
    \item[] Justification: We repeat each experiment ten times to obtain consistent and reliable results. The experimental results with mean and standard deviation are provided in Tab.~\ref{table1} and Tab.~\ref{table3}.

\item {\bf Experiments compute resources}
    \item[] Question: For each experiment, does the paper provide sufficient information on the computer resources (type of compute workers, memory, time of execution) needed to reproduce the experiments?
    \item[] Answer: \answerYes{} 
    \item[] Justification: All experiments are conducted on a Linux server with a 2.90\,GHz Intel Xeon Gold 6326 CPU, 64\,GB RAM, and two NVIDIA GeForce RTX 4090 GPUs with 48\,GB memory. Our proposed FedGMC is implemented in Python 3.8.8, PyTorch 1.12.0, and PyTorch Geometric (PyG) 2.5.1.
    
\item {\bf Code of ethics}
    \item[] Question: Does the research conducted in the paper conform, in every respect, with the NeurIPS Code of Ethics \url{https://neurips.cc/public/EthicsGuidelines}?
    \item[] Answer: \answerYes{} 
    \item[] Justification: We have carefully checked the NeurIPS Code of Ethics and confirmed that our paper obeys it. 

\item {\bf Broader impacts}
    \item[] Question: Does the paper discuss both potential positive societal impacts and negative societal impacts of the work performed?
    \item[] Answer: \answerYes{} 
    \item[] Justification: The broader impacts are introduced in Appendix~\ref{appendix: further discussion}. We believe that this paper does not raise any negative societal impacts or ethical concerns.

\item {\bf Safeguards}
    \item[] Question: Does the paper describe safeguards that have been put in place for responsible release of data or models that have a high risk for misuse (e.g., pre-trained language models, image generators, or scraped datasets)?
    \item[] Answer: \answerNA{} 
    \item[] Justification: This paper does not release high-risk models. The proposed FedGMC does not involve such risks.

\item {\bf Licenses for existing assets}
    \item[] Question: Are the creators or original owners of assets (e.g., code, data, models), used in the paper, properly credited and are the license and terms of use explicitly mentioned and properly respected?
    \item[] Answer: \answerYes{} 
    \item[] Justification: In the paper, we have introduced the resources of models and datasets.

\item {\bf New assets}
    \item[] Question: Are new assets introduced in the paper well documented and is the documentation provided alongside the assets?
    \item[] Answer: \answerNA{} 
    \item[] Justification: We did not release any new assets such as datasets, benchmarks, or standalone software packages.

\item {\bf Crowdsourcing and research with human subjects}
    \item[] Question: For crowdsourcing experiments and research with human subjects, does the paper include the full text of instructions given to participants and screenshots, if applicable, as well as details about compensation (if any)? 
    \item[] Answer: \answerNA{} 
    \item[] Justification: The paper is not about crowdsourcing experiments or research with human subjects.

\item {\bf Institutional review board (IRB) approvals or equivalent for research with human subjects}
    \item[] Question: Does the paper describe potential risks incurred by study participants, whether such risks were disclosed to the subjects, and whether Institutional Review Board (IRB) approvals (or an equivalent approval/review based on the requirements of your country or institution) were obtained?
    \item[] Answer: \answerNA{} 
    \item[] Justification: The paper is not about research with human subjects.

\item {\bf Declaration of LLM usage}
    \item[] Question: Does the paper describe the usage of LLMs if it is an important, original, or non-standard component of the core methods in this research? Note that if the LLM is used only for writing, editing, or formatting purposes and does \emph{not} impact the core methodology, scientific rigor, or originality of the research, declaration is not required.
    \item[] Answer: \answerNA{} 
    \item[] Justification: The proposed method development in this research does not involve LLMs as any important, original, or non-standard components.

\end{enumerate}

\newpage
\appendix
\section{Theoretical Analysis}\label{Theoretical_Analysis}
In this section, we provide rigorous theoretical analysis for the convergence properties of our proposed dual manifold calibration mechanism. We first establish the convergence guarantees for semantic manifold calibration, followed by the analysis of structural manifold calibration.

\subsection{Convergence Analysis of Semantic Manifold Calibration}
\subsubsection{Preliminaries and Assumptions}
Before presenting our main theoretical results, we introduce the necessary assumptions and definitions.
\begin{assumption}[Bounded Semantic Representations]
\label{assump:bounded_semantic}
For each client $m \in \{1, 2, \cdots, M\}$ and any node $v \in \mathcal{V}_m$, the ego-embedding satisfies $\|\mathbf{h}_{m,v}^{\mathrm{ego}}\|_2 \leq L_{\mathrm{ego}}$ for some constant $L_{\mathrm{ego}} > 0$.
\end{assumption}
\begin{assumption}[Lipschitz Continuity of Local Models]
\label{assump:lipschitz_local}
The local model that generates ego-embeddings is $L_{\mathrm{lip}}$-Lipschitz continuous with respect to its parameters $\mathbf{W}_m^{\mathrm{ego}}$, i.e., for any two parameter matrices $\mathbf{W}_1, \mathbf{W}_2$:
\begin{equation}
\left\| \sigma(\mathbf{X}_m \mathbf{W}_1) - \sigma(\mathbf{X}_m \mathbf{W}_2) \right\|_\mathrm{F} \leq L_{\mathrm{lip}} \left\| \mathbf{W}_1 - \mathbf{W}_2 \right\|_\mathrm{F}.
\end{equation}
\end{assumption}
\begin{assumption}[Bounded Gradient Variance]
\label{assump:bounded_variance}
The variance of stochastic gradients computed on each client is bounded, i.e., for any client $m$:
\begin{equation}
\mathbb{E}\left[ \left\| \nabla \mathcal{L}_m(\mathbf{W}_m^{\mathrm{ego}}) - \nabla \mathcal{L}(\mathbf{W}_m^{\mathrm{ego}}) \right\|_2^2 \right] \leq \sigma^2,
\end{equation}
where $\mathcal{L}_m$ denotes the local loss on client $m$ and $\mathcal{L}$ denotes the global loss.
\end{assumption}
\begin{definition}[Semantic Manifold Alignment Error]
\label{def:semantic_alignment_error}
For client $m$, we define the semantic manifold alignment error as:
\begin{equation}
\mathcal{E}_m^{\mathrm{sem}} = \left\| \mathbf{R}_m \mathbf{P}_m  - \Delta \right\|_\mathrm{F}^2,
\end{equation}
where $\mathbf{R}_m$ is the optimal orthogonal transformation, $\mathbf{P}_m$ is the local semantic manifold, and $\Delta$ is the global semantic manifold.
\end{definition}

\subsubsection{Properties of Orthogonal Procrustes Transformation}
We first establish fundamental properties of the orthogonal Procrustes transformation used in semantic manifold calibration.
\begin{lemma}[Optimality of Orthogonal Procrustes Solution]
\label{lemma:procrustes_optimal}
Let $\mathbf{P}_m \in \mathbb{R}^{d \times C}$ and $\Delta \in \mathbb{R}^{d \times C}$ be two matrices. The orthogonal matrix $\mathbf{R}_m^* = \mathbf{U}_m \mathbf{V}_m^{\top}$ obtained from the SVD $\Delta \mathbf{P}_m^{\top} = \mathbf{U}_m \mathbf{\Sigma}_m \mathbf{V}_m^{\top}$ is the unique solution to:
\begin{equation}
\mathbf{R}_m^* = \arg\min_{\mathbf{R}_m^{\top} \mathbf{R}_m = \mathbf{I}} \left\| \mathbf{R}_m \mathbf{P}_m - \Delta \right\|_\mathrm{F}^2.
\end{equation}
Moreover, the minimum alignment error satisfies:
\begin{equation}
\left\| \mathbf{R}_m^* \mathbf{P}_m - \Delta \right\|_\mathrm{F}^2 = \|\mathbf{P}_m\|_\mathrm{F}^2 + \|\Delta\|_\mathrm{F}^2 - 2\mathrm{tr}(\mathbf{\Sigma}_m).
\end{equation}
\end{lemma}

\begin{proof}
We expand the objective function:
\begin{align}
\left\| \mathbf{R}_m \mathbf{P}_m - \Delta \right\|_\mathrm{F}^2 &= \mathrm{tr}\left[(\mathbf{R}_m \mathbf{P}_m - \Delta)^{\top} (\mathbf{R}_m \mathbf{P}_m - \Delta)\right] \\
&= \mathrm{tr}(\mathbf{P}_m^{\top} \mathbf{R}_m^{\top} \mathbf{R}_m \mathbf{P}_m) - 2\mathrm{tr}(\Delta \mathbf{P}_m^{\top} \mathbf{R}_m^{\top}) + \mathrm{tr}(\Delta^{\top} \Delta) \\
&= \mathrm{tr}(\mathbf{P}_m^{\top} \mathbf{P}_m) - 2\mathrm{tr}(\Delta \mathbf{P}_m^{\top} \mathbf{R}_m^{\top}) + \mathrm{tr}(\Delta^{\top} \Delta),
\end{align}
where the last equality uses the orthogonality constraint $\mathbf{R}_m^{\top} \mathbf{R}_m = \mathbf{I}$ and the cyclic property of trace.

Minimizing the above expression is equivalent to maximizing $\mathrm{tr}(\Delta \mathbf{P}_m^{\top} \mathbf{R}_m^{\top})$. Let $\Delta \mathbf{P}_m^{\top} = \mathbf{U}_m \mathbf{\Sigma}_m \mathbf{V}_m^{\top}$ be the SVD. Then, we have
\begin{align}
\mathrm{tr}(\Delta \mathbf{P}_m^{\top} \mathbf{R}_m^{\top}) &= \mathrm{tr}(\mathbf{U}_m \mathbf{\Sigma}_m \mathbf{V}_m^{\top} \mathbf{R}_m^{\top} ) \\
&= \mathrm{tr}(\mathbf{\Sigma}_m \mathbf{V}_m^{\top} \mathbf{R}_m^{\top} \mathbf{U}_m).
\end{align}

Let $\mathbf{Z} = \mathbf{V}_m^{\top} \mathbf{R}_m^{\top} \mathbf{U}_m$. Since $\mathbf{R}_m$ is orthogonal, $\mathbf{Z}$ is also orthogonal. By the von Neumann trace inequality, for any orthogonal matrix $\mathbf{Z}$ and diagonal matrix $\mathbf{\Sigma}_m$ with non-negative entries:
\begin{equation}
\mathrm{tr}(\mathbf{\Sigma}_m \mathbf{Z}) \leq \mathrm{tr}(\mathbf{\Sigma}_m),
\end{equation}
with equality if and only if $\mathbf{Z} = \mathbf{I}$. This occurs when $\mathbf{R}_m = \mathbf{U}_m \mathbf{V}_m^{\top}$, completing the proof.
\end{proof}

\begin{lemma}[Isometric Property of Orthogonal Transformation]
\label{lemma:isometric}
The orthogonal Procrustes transformation $\mathbf{R}_m$ preserves all pairwise distances within the local semantic manifold, i.e., for any two class prototypes $\mathbf{p}_{m,i}, \mathbf{p}_{m,j} \in \mathbb{R}^d$ (columns of $\mathbf{P}_m$):
\begin{equation}
\left\| \mathbf{R}_m \mathbf{p}_{m,i} - \mathbf{R}_m \mathbf{p}_{m,j} \right\|_2 = \left\| \mathbf{p}_{m,i} - \mathbf{p}_{m,j} \right\|_2.
\end{equation}
\end{lemma}

\begin{proof}
Since \(\mathbf{R}_m\) is orthogonal, we have
\begin{align}
\left\| \mathbf{R}_m \mathbf{p}_{m,i} - \mathbf{R}_m \mathbf{p}_{m,j} \right\|_2^2 
&= \left\| \mathbf{R}_m (\mathbf{p}_{m,i} - \mathbf{p}_{m,j}) \right\|_2^2 \\
&= (\mathbf{p}_{m,i} - \mathbf{p}_{m,j})^\top \mathbf{R}_m^\top \mathbf{R}_m (\mathbf{p}_{m,i} - \mathbf{p}_{m,j}) \\
&= (\mathbf{p}_{m,i} - \mathbf{p}_{m,j})^\top (\mathbf{p}_{m,i} - \mathbf{p}_{m,j}) \\
&= \left\| \mathbf{p}_{m,i} - \mathbf{p}_{m,j} \right\|_2^2.
\end{align}
This completes the proof.
\end{proof}

\subsubsection{Convergence of Semantic Calibration Loss}
We now establish the convergence of semantic calibration loss under the proposed calibration mechanism.
\begin{theorem}[Convergence of Semantic Calibration]
\label{thm:semantic_convergence}
Under Assumptions~\ref{assump:bounded_semantic},~\ref{assump:lipschitz_local}, and~\ref{assump:bounded_variance}, let \(\{\mathbf{W}_m^{\mathrm{ego}}(t)\}_{t=0}^{\infty}\) be the sequence of local model parameters updated via deterministic gradient descent with learning rates \(\{\eta_t\}\) satisfying
\begin{equation}
\eta_t > 0, \quad \sum_{t=0}^{\infty} \eta_t = \infty, \quad \sum_{t=0}^{\infty} \eta_t^2 < \infty.
\end{equation}
Then the semantic calibration loss converges to a finite limit:
\begin{equation}
\lim_{t \to \infty} \mathcal{L}_{\mathrm{se\_calibration}}(t) = \mathcal{L}_{\mathrm{se\_calibration}}^*,
\end{equation}
and the gradient norm satisfies
\begin{equation}
\liminf_{t \to \infty} \left\| \nabla \mathcal{L}_{\mathrm{se\_calibration}}(t) \right\|_\mathrm{F} = 0.
\end{equation}
\end{theorem}

\begin{proof}
Assume the step sizes \(\{\eta_t\}_{t \ge 0}\) satisfy the standard Robbins--Monro conditions:
\begin{align}
\eta_t > 0, \qquad
\sum_{t=0}^\infty \eta_t = \infty, \qquad
\sum_{t=0}^\infty \eta_t^2 < \infty.
\end{align}
We analyze deterministic gradient descent on the semantic calibration loss \(\mathcal{L}_{\mathrm{se\_calibration}}\) with respect to the ego-embedding parameters \(\mathbf{W}^{\mathrm{ego}}\).

\textbf{Step 1: Descent Lemma.}
Assume that \(\mathcal{L}_{\mathrm{se\_calibration}}\) is \(L_{\mathrm{smooth}}\)-smooth with respect to \(\mathbf{W}^{\mathrm{ego}}\). Under Assumptions~\ref{assump:bounded_semantic} and~\ref{assump:lipschitz_local}, the smoothness constant is
\begin{align}
L_{\mathrm{smooth}} = 2 L_{\mathrm{lip}}^2 (L_{\mathrm{ego}} + 1).
\end{align}
The descent lemma yields
\begin{align}
\mathcal{L}_{\mathrm{se\_calibration}}(t+1)
&\leq \mathcal{L}_{\mathrm{se\_calibration}}(t)
- \eta_t \left(1 - \frac{L_{\mathrm{smooth}} \eta_t}{2}\right) \|\nabla \mathcal{L}_{\mathrm{se\_calibration}}(t)\|_\mathrm{F}^2,
\end{align}
provided \(\eta_t < \frac{2}{L_{\mathrm{smooth}}}\).

\textbf{Step 2: Bounded Gradient Norm.}
By the chain rule and Assumptions~\ref{assump:bounded_semantic} and~\ref{assump:lipschitz_local},
\begin{equation}
\left\| \nabla_{\mathbf{W}^{\mathrm{ego}}} \mathcal{L}_{\mathrm{se\_calibration}}(t) \right\|_\mathrm{F} \le G_{\max} := 2 M N L_{\mathrm{ego}} (L_{\mathrm{ego}} + 1) L_{\mathrm{lip}}.
\end{equation}

\textbf{Step 3: Summability of Descent.}
From \(\sum_t \eta_t^2 < \infty\), we have \(\eta_t \to 0\). Therefore, there exists \(T_0\) such that for all \(t \ge T_0\),
\begin{equation}
\eta_t \le \frac{1}{L_{\mathrm{smooth}}}.
\end{equation}
Consequently,
\begin{equation}
1 - \frac{L_{\mathrm{smooth}} \eta_t}{2} \ge \frac{1}{2}.
\end{equation}
The descent inequality simplifies to
\begin{equation}
\mathcal{L}_{\mathrm{se\_calibration}}(t+1) \le \mathcal{L}_{\mathrm{se\_calibration}}(t) - \frac{\eta_t}{2} \|\nabla \mathcal{L}_{\mathrm{se\_calibration}}(t)\|_\mathrm{F}^2, \quad t \ge T_0.
\end{equation}
Summing from \(t = T_0\) to \(T-1\) and letting \(T \to \infty\) (using \(\mathcal{L}_{\mathrm{se\_calibration}} \ge 0\)) gives
\begin{equation}
\sum_{t=T_0}^\infty \eta_t \|\nabla \mathcal{L}_{\mathrm{se\_calibration}}(t)\|_\mathrm{F}^2 < \infty.
\end{equation}

\textbf{Step 4: Liminf of Gradient Norm.}
Suppose for contradiction that
\begin{equation}
\liminf_{t \to \infty} \|\nabla \mathcal{L}_{\mathrm{se\_calibration}}(t)\|_\mathrm{F} \ge \delta > 0.
\end{equation}
Afterwards, there exists \(T_1 \ge T_0\) such that
\begin{equation}
\|\nabla \mathcal{L}_{\mathrm{se\_calibration}}(t)\|_\mathrm{F} \ge \frac{\delta}{2} \quad \text{for all } t \ge T_1.
\end{equation}
This implies
\begin{equation}
\sum_{t=T_1}^\infty \eta_t \|\nabla \mathcal{L}_{\mathrm{se\_calibration}}(t)\|_\mathrm{F}^2 \ge \left(\frac{\delta}{2}\right)^2 \sum_{t=T_1}^\infty \eta_t = \infty,
\end{equation}
where the last equality uses the assumption \(\sum_t \eta_t = \infty\). This contradicts the summability result of Step 3. Therefore,
\begin{equation}
\liminf_{t \to \infty} \|\nabla \mathcal{L}_{\mathrm{se\_calibration}}(t)\|_\mathrm{F} = 0.
\end{equation}

\textbf{Step 5: Convergence of the Loss Value.}
For all \(t \ge T_0\), the inequality in Step 3 shows that \(\{\mathcal{L}_{\mathrm{se\_calibration}}(t)\}_{t \ge T_0}\) is monotonically non-increasing and bounded from below by 0. By the monotone convergence theorem,
\begin{equation}
\mathcal{L}_{\mathrm{se\_calibration}}(t) \to \mathcal{L}_{\mathrm{se\_calibration}}^*.
\end{equation}
\end{proof}

\subsubsection{Convergence Rate of Global Semantic Manifold Refinement}
We now analyze the convergence rate of the dynamic refinement process for the global semantic manifold.

\begin{theorem}[Convergence Rate on the Global Semantic Manifold]
\label{thm:semantic_manifold_rate}
Let \(\Delta(t)=\{\delta_1(t), \delta_2(t), \dots,\delta_C(t)\}\) with \(\|\delta_i(t)\|_2=1\) for all \(i,t\). Assume:
\begin{itemize}
    \item[(A1)] \(\mathcal{L}_{\mathrm{global}}(\Delta)\ge\mathcal{L}_{\mathrm{global}}^*\) for some finite \(\mathcal{L}_{\mathrm{global}}^* > -\infty\);
    \item[(A2)] The stochastic Riemannian gradient is unbiased and tangent:
    \begin{equation}
    \mathbb{E}_t\bigl[\widehat{\operatorname{grad}}_{\mathcal{S}}\mathcal{L}_{\mathrm{global}}(\Delta(t))\bigr]=\operatorname{grad}_{\mathcal{S}}\mathcal{L}_{\mathrm{global}}(\Delta(t)),
    \end{equation}
    and lies in the tangent space \(T_{\Delta(t)}\mathcal{S}\);
    \item[(A3)] Bounded variance: \(\mathbb{E}_t\bigl[\|\widehat{\operatorname{grad}}_{\mathcal{S}}\mathcal{L}_{\mathrm{global}}(\Delta(t))-\operatorname{grad}_{\mathcal{S}}\mathcal{L}_{\mathrm{global}}(\Delta(t))\|_\mathrm{F}^2\bigr]\le\sigma_{\mathrm{agg}}^2\);
    \item[(A4)] Retraction-smoothness: there exists \(L_{\mathcal{S}}>0\) such that for all \(\Delta\in\mathcal{S}\) and all tangent vectors \(\xi\in T_\Delta\mathcal{S}\) in the domain of the retraction, namely
    \begin{equation}
    \mathcal{L}_{\mathrm{global}}\bigl(\operatorname{Retr}_\Delta(\xi)\bigr)
    \le\mathcal{L}_{\mathrm{global}}(\Delta)
    +\langle\operatorname{grad}_{\mathcal{S}}\mathcal{L}_{\mathrm{global}}(\Delta),\xi\rangle
    +\frac{L_{\mathcal{S}}}{2}\|\xi\|_\mathrm{F}^2.
    \end{equation}
\end{itemize}

Consider the Riemannian stochastic gradient update
\begin{equation}
\Delta(t+1)=\operatorname{Retr}_{\Delta(t)}\bigl(-\eta\,\widehat{\operatorname{grad}}_{\mathcal{S}}\mathcal{L}_{\mathrm{global}}(\Delta(t))\bigr),
\end{equation}
with step size
\begin{equation}
\eta=\min\left\{\frac{1}{L_{\mathcal{S}}},\frac{c}{\sqrt{T}}\right\},
\end{equation}
for some constant \(c>0\).

Subsequently, we have
\begin{equation}
\min_{0\le t<T}\mathbb{E}\Bigl[\bigl\|\operatorname{grad}_{\mathcal{S}}\mathcal{L}_{\mathrm{global}}(\Delta(t))\bigr\|_\mathrm{F}^2\Bigr]
\le\frac{2\bigl(\mathbb{E}[\mathcal{L}_{\mathrm{global}}(\Delta(0))]-\mathcal{L}_{\mathrm{global}}^*\bigr)}{\eta T}+L_{\mathcal{S}}\eta\sigma_{\mathrm{agg}}^2.
\end{equation}
This yields the convergence rate \(\mathcal{O}(1/\sqrt{T})\).
\end{theorem}

\begin{proof}
By retraction-smoothness (A4) with update direction \(\xi=-\eta\hat g_t\) (\(\hat g_t=\widehat{\operatorname{grad}}_{\mathcal{S}}\mathcal{L}_{\mathrm{global}}(\Delta(t))\) and \(g_t=\operatorname{grad}_{\mathcal{S}}\mathcal{L}_{\mathrm{global}}(\Delta(t))\)),
\begin{equation}
\mathcal{L}_{\mathrm{global}}(\Delta(t+1))
\le\mathcal{L}_{\mathrm{global}}(\Delta(t))
-\eta\langle g_t,\hat g_t\rangle
+\frac{L_{\mathcal{S}}\eta^2}{2}\|\hat g_t\|_\mathrm{F}^2.
\end{equation}

Taking conditional expectation,
\begin{equation}
\mathbb{E}_t\bigl[\mathcal{L}_{\mathrm{global}}(\Delta(t+1))\bigr]
\le\mathcal{L}_{\mathrm{global}}(\Delta(t))
-\eta\mathbb{E}_t[\langle g_t,\hat g_t\rangle]
+\frac{L_{\mathcal{S}}\eta^2}{2}\mathbb{E}_t[\|\hat g_t\|_\mathrm{F}^2].
\end{equation}

By unbiasedness (A2), we have
\begin{align}
\mathbb{E}_t[\langle g_t,\hat g_t\rangle]=\|g_t\|_\mathrm{F}^2.
\end{align}
By the variance bound (A3) and unbiasedness,
\begin{align}
\mathbb{E}_t[\|\hat g_t\|_\mathrm{F}^2]=\|g_t\|_\mathrm{F}^2+\mathbb{E}_t[\|\hat g_t-g_t\|_\mathrm{F}^2]\le\|g_t\|_\mathrm{F}^2+\sigma_{\mathrm{agg}}^2.
\end{align}

Substituting these relations yields
\begin{align}
\mathbb{E}_t\bigl[\mathcal{L}_{\mathrm{global}}(\Delta(t+1))\bigr]
&\le\mathcal{L}_{\mathrm{global}}(\Delta(t))
-\eta\Bigl(1-\frac{L_{\mathcal{S}}\eta}{2}\Bigr)\|g_t\|_\mathrm{F}^2
+\frac{L_{\mathcal{S}}\eta^2}{2}\sigma_{\mathrm{agg}}^2.
\end{align}

Since \(\eta\le \frac{1}{L_{\mathcal{S}}}\), we have \(1-\frac{L_{\mathcal{S}}\eta}{2}\ge \frac{1}{2}\),
\begin{equation}
\mathbb{E}_t\bigl[\mathcal{L}_{\mathrm{global}}(\Delta(t+1))\bigr]
\le\mathcal{L}_{\mathrm{global}}(\Delta(t))
-\frac{\eta}{2}\|g_t\|_\mathrm{F}^2
+\frac{L_{\mathcal{S}}\eta^2}{2}\sigma_{\mathrm{agg}}^2.
\end{equation}

Taking full expectation, summing from \(t=0\) to \(T-1\), and using the lower bound (A1), we obtain after telescoping and dividing by \(\frac{\eta T}{2}\):
\begin{equation}
\frac{1}{T}\sum_{t=0}^{T-1}\mathbb{E}\bigl[\|g_t\|_\mathrm{F}^2\bigr]
\le\frac{2\bigl(\mathbb{E}[\mathcal{L}_{\mathrm{global}}(\Delta(0))]-\mathcal{L}_{\mathrm{global}}^*\bigr)}{\eta T}+L_{\mathcal{S}}\eta\sigma_{\mathrm{agg}}^2.
\end{equation}

Finally, we have
\begin{equation}
\min_{0\le t<T}\mathbb{E}\bigl[\|g_t\|_\mathrm{F}^2\bigr]
\le\frac{1}{T}\sum_{t=0}^{T-1}\mathbb{E}\bigl[\|g_t\|_\mathrm{F}^2\bigr].
\end{equation}
This completes the proof.
\end{proof}

\subsection{Convergence Analysis of Structural Manifold Calibration}
\begin{theorem}[Convergence to a Critical Point]
\label{thm:structural_convergence}
Let \( x^t = (\mathcal{T}^t, \mathbf{F}_m^t) \) and define the extended-value function
\begin{equation}
\Psi(x) = \Phi(x) + \delta_{\mathcal{C}}(x),
\end{equation}
where
\begin{equation}
\Phi(\mathcal{T}, \mathbf{F}_m)
= \mathcal{L}_{\mathrm{str\_calibration}}(\mathcal{T}, \mathbf{F}_m)
+ \frac{1}{\tau} \sum_{m=1}^{M} \sum_{b=1}^{B} \sum_{q=1}^{Q} (\mathbf{F}_m)_{b,q} \log (\mathbf{F}_m)_{b,q}
\end{equation}
is the entropy-regularized structural calibration objective, and \(\mathcal{C}\) is the feasible set (\textit{i.e.}, product of the template space and the transport polytope).
Assume the following conditions hold:
\begin{enumerate}
    \item[(i)] The feasible set \(\mathcal{C}\) is compact;
    \item[(ii)] \(\Psi\) is proper and lower semicontinuous, and satisfies the Kurdyka-Łojasiewicz (KL) property;
    \item[(iii)] The sequence satisfies the sufficient decrease condition: there exists \( a > 0 \) such that
    \begin{equation}
    \Psi(x^{t+1}) \leq \Psi(x^t) - a \|x^{t+1} - x^t\|^2;
    \end{equation}
    \item[(iv)] The relative error condition holds: there exist \( b > 0 \) and \( w^{t+1} \in \partial \Psi(x^{t+1}) \) such that
    \begin{equation}
    \|w^{t+1}\| \leq b \|x^{t+1} - x^t\|;
    \end{equation}
    \item[(v)] \(\Phi\) is continuous on \(\mathcal{C}\) and continuously differentiable on the relative interior of \(\mathcal{C}\).
\end{enumerate}
Then the sequence \(\{x^t\}\) has finite length and converges to a critical point \(x^*\) of \(\Psi\), i.e.,
\begin{equation}
0 \in \partial \Psi(x^*).
\end{equation}
Moreover, if \(\Phi\) is continuously differentiable in a neighborhood of \(x^*\), then
\begin{equation}
0 \in \nabla \Phi(x^*) + N_{\mathcal{C}}(x^*),
\end{equation}
where \(N_{\mathcal{C}}(x^*)\) denotes the normal cone to \(\mathcal{C}\) at \(x^*\).
\end{theorem}

\begin{proof}
We verify the hypotheses of the KL convergence theorem of Attouch, Bolte and Svaiter.
By the sufficient decrease condition (iii),
\begin{equation}
\Psi(x^{t+1}) \leq \Psi(x^t) - a \|x^{t+1} - x^t\|^2.
\end{equation}
Since \(\mathcal{C}\) is compact and \(\Phi\) is continuous on \(\mathcal{C}\), \(\Psi = \Phi\) is bounded from below on \(\mathcal{C}\). Therefore, \(\{\Psi(x^t)\}\) is nonincreasing and converges to some finite value \(\Psi_\infty\).
Summing the sufficient decrease inequality yields
\begin{equation}
a \sum_{t=0}^{\infty} \|x^{t+1} - x^t\|^2 \leq \Psi(x^0) - \Psi_\infty < \infty,
\end{equation}
which immediately implies
\begin{equation}
\|x^{t+1} - x^t\| \to 0.
\end{equation}
Since \(\mathcal{C}\) is compact, the sequence \(\{x^t\}\) admits at least one cluster point. Let \(\bar x\) be such a cluster point and let \(x^{t_j} \to \bar x\). By lower semicontinuity of \(\Psi\),
\begin{equation}
\Psi(\bar x) \leq \liminf_{j\to\infty} \Psi(x^{t_j}) = \Psi_\infty.
\end{equation}
Moreover, since \(\Phi\) is continuous on the compact set \(\mathcal{C}\), we have \(\Psi(x^{t_j}) \to \Psi(\bar x)\). Therefore,
\begin{equation}
\Psi(\bar x) = \Psi_\infty.
\end{equation}
Afterwards, by the relative error condition (iv), there exists \(w^{t+1} \in \partial \Psi(x^{t+1})\) such that
\begin{equation}
\|w^{t+1}\| \leq b \|x^{t+1} - x^t\|.
\end{equation}
Since \(\|x^{t+1} - x^t\| \to 0\), it follows that \(\|w^{t+1}\| \to 0\).
Because \(\|x^{t+1} - x^t\| \to 0\) and \(x^{t_j} \to \bar x\), we also have
\begin{equation}
x^{t_j+1} \to \bar x.
\end{equation}
Passing to the limit along this subsequence and using the closed graph property of the limiting subdifferential, we obtain
\begin{equation}
0 \in \partial \Psi(\bar x).
\end{equation}
Therefore, every cluster point is a critical point of \(\Psi\). Since \(\Psi\) satisfies the KL property and the sequence satisfies both the sufficient decrease condition and the relative error condition, the standard KL convergence theorem of Attouch-Bolte-Svaiter implies that \(\{x^t\}\) has finite length:
\begin{equation}
\sum_{t=0}^{\infty} \|x^{t+1} - x^t\| < \infty.
\end{equation}
Consequently, \(\{x^t\}\) is a Cauchy sequence. Since \(\mathcal{C}\) is compact, there exists \(x^* \in \mathcal{C}\) such that \(x^t \to x^*\).
By the closedness of the limiting subdifferential and \(\|w^{t+1}\| \to 0\), we conclude
\begin{equation}
0 \in \partial \Psi(x^*).
\end{equation}
Therefore, \(x^*\) is a critical point of \(\Psi\).
Finally, suppose that \(\Phi\) is continuously differentiable in a neighborhood of \(x^*\). Since
\begin{equation}
\Psi = \Phi + \delta_{\mathcal{C}},
\end{equation}
the standard subdifferential sum rule gives
\begin{equation}
\partial \Psi(x^*) = \nabla \Phi(x^*) + \partial \delta_{\mathcal{C}}(x^*) = \nabla \Phi(x^*) + N_{\mathcal{C}}(x^*).
\end{equation}
Therefore, we have
\begin{equation}
0 \in \nabla \Phi(x^*) + N_{\mathcal{C}}(x^*),
\end{equation}
which is the desired KKT-type stationarity condition. This completes the proof.
\end{proof}

\section{Implementation Details}\label{implementation_details}
This section details our experimental setup, encompassing the local training objective, computational platform, graph datasets, subgraph partitioning, baseline methods, and training procedures.
\subsection{Local Training Objective}\label{ap_local_objective}
For each client $m$, the local training objective integrates the cross-entropy loss for node classification with the semantic manifold calibration loss and the structural manifold calibration loss. Specifically, each client minimizes
\begin{equation}
\mathcal{L}_m = \mathcal{L}_\mathrm{ce}^{(m)} + \mathcal{L}_\mathrm{se\_calibration}^{(m)} + \mathcal{L}_\mathrm{str\_calibration}^{(m)},
\end{equation}
where $\mathcal{L}_\mathrm{ce}^{(m)}$ denotes the cross-entropy loss, $\mathcal{L}_\mathrm{se\_calibration}^{(m)}$ is the semantic manifold calibration loss, and $\mathcal{L}_\mathrm{str\_calibration}^{(m)}$ is the structural manifold calibration loss.
\subsection{Experimental Platform}\label{experimental_platform}
All experiments are conducted on a Linux server equipped with a 2.90\,GHz Intel Xeon Gold 6326 CPU, 64\,GB of RAM, and two NVIDIA GeForce RTX 4090 GPUs with 48\,GB memory. Our FedGMC is implemented in Python 3.8.8, PyTorch 1.12.0, and PyTorch Geometric (PyG) 2.5.1.
\subsection{Datasets}\label{dataset_info}
We evaluate our proposed FedGMC on eleven widely adopted benchmark datasets, comprising six homophilic graphs and five heterophilic graphs. The homophilic datasets include four citation networks (\textit{Cora}, \textit{CiteSeer}, \textit{PubMed}, and \textit{ogbn-arxiv}) as well as two Amazon product co-purchasing graphs (\textit{Amazon-Computer} and \textit{Amazon-Photo}). The heterophilic datasets consist of \textit{Roman-empire}, \textit{Amazon-ratings}, \textit{Minesweeper}, \textit{Tolokers}, and \textit{Questions}~\cite{platonov2023a}. The statistical characteristics of all datasets are summarized in Tab.~\ref{datasets_statistics}. Following prior work~\cite{platonov2023a, wentao2025fediih, yu2026heterogeneity}, we employ Area Under the ROC Curve (AUC) as the evaluation metric for the binary classification tasks on \textit{Minesweeper}, \textit{Tolokers}, and \textit{Questions}. For the remaining multi-class datasets, classification accuracy is adopted as the evaluation metric.

\begin{table*}[t]
    \centering
    \scriptsize
    \caption{The statistical information of eleven graph datasets.}
    \label{datasets_statistics}
    \renewcommand{\arraystretch}{1.1}
    \begin{tabular}{lrrrr}
    \toprule
    \textbf{Dataset} & \textbf{\# Nodes} & \textbf{\# Edges} & \textbf{\# Classes} & \textbf{\# Features} \\ 
    \midrule
    \rowcolor{gray!50}
    \multicolumn{5}{l}{\textbf{Homophilic Graph Datasets}} \\
    \textit{Cora}            & 2,708    & 5,429     & 7  & 1,433 \\
    \textit{CiteSeer}        & 3,327    & 4,732     & 6  & 3,703 \\
    \textit{PubMed}          & 19,717   & 44,324    & 3  & 500   \\
    \textit{Amazon-Computer} & 13,752   & 491,722   & 10 & 767   \\
    \textit{Amazon-Photo}    & 7,650    & 238,162   & 8  & 745   \\
    \textit{ogbn-arxiv}      & 169,343  & 1,166,243 & 40 & 128   \\
    \hline
    \rowcolor{gray!50}
    \multicolumn{5}{l}{\textbf{Heterophilic Graph Datasets}} \\
    \textit{Roman-empire}    & 22,662   & 32,927    & 18 & 300   \\
    \textit{Amazon-ratings}  & 24,492   & 93,050    & 5  & 300   \\
    \textit{Minesweeper}     & 10,000   & 39,402    & 2  & 7     \\
    \textit{Tolokers}        & 11,758   & 519,000   & 2  & 10    \\
    \textit{Questions}       & 48,921   & 153,540   & 2  & 301   \\
    \bottomrule
    \end{tabular}
\end{table*}

For all datasets except \textit{ogbn-arxiv}, we randomly partition the nodes into training, validation, and test sets using the ratios 20\%, 40\%, and 40\%, respectively. For \textit{ogbn-arxiv}, which contains over 0.1 million nodes and more than 1 million edges, we follow prior work~\cite{baek2023personalized, wentao2025fediih, yu2026heterogeneity} and allocate only 5\% of nodes for training, 47.5\% for validation, and 47.5\% for testing.

\subsection{Subgraph Partitioning}\label{subgraph_partitioning_detail}
To simulate realistic federated scenarios with varying degrees of graph data isolation, we consider two subgraph partitioning schemes following prior work~\cite{baek2023personalized, wentao2025fediih, yu2026heterogeneity}: non-overlapping and overlapping. In the non-overlapping setting, the global node set $\mathcal{V}$ is partitioned into $M$ disjoint client-specific subsets $\{\mathcal{V}_m\}_{m=1}^M$ such that $\bigcup_{m=1}^M \mathcal{V}_m = \mathcal{V}$ and $\mathcal{V}_m \cap \mathcal{V}_n = \emptyset$ for all $m \neq n$. Any partitioning scheme that fails to satisfy this strict disjointness condition is referred to as overlapping. The concrete procedures for generating both types of subgraphs are detailed below.

\subsubsection{Non-overlapping Partitioning}
Given $M$ clients, we employ the METIS graph partitioning algorithm~\cite{karypis1997metis} to divide the global graph into exactly $M$ non-overlapping subgraphs. Each client is then assigned a unique subgraph corresponding to one of the partitions produced by METIS. This scheme effectively enforces graph data partition across clients, reflecting many real-world privacy-sensitive applications.

\subsubsection{Overlapping Partitioning}
Given $M$ clients, we first apply METIS graph partitioning algorithm~\cite{karypis1997metis} to partition the global graph into $\lfloor M/5 \rfloor$ temporary subgraphs. For each temporary subgraph, we randomly sample half of its nodes together with their incident edges and repeat this sampling process independently five times. This yields five distinct but partially overlapping subgraphs per coarse partition. As a result, the total number of generated subgraphs precisely equals the number of clients $M$.

\subsection{Baseline Methods}\label{baseline_methods_info}
We compare our proposed FedGMC with thirteen baseline methods, comprising one classic federated learning (FL) method (\textit{i.e.}, FedAvg~\cite{mcmahan2017communication}), two personalized FL methods (\textit{i.e.}, FedProx~\cite{MLSYS2020_1f5fe839} and FedPer~\cite{Arivazhagan2019}), three general graph federated learning (GFL) methods (\textit{i.e.}, GCFL~\cite{NEURIPS2021_9c6947bd}, FedGNN~\cite{wu2021fedgnn}, and FedSage+~\cite{NEURIPS2021_34adeb8e}), and seven personalized GFL approaches (\textit{i.e.}, FED-PUB~\cite{baek2023personalized}, FedGTA~\cite{li2023fedgta}, AdaFGL~\cite{li2024adafgl}, FedTAD~\cite{zhu2024fedtad}, FedIIH~\cite{wentao2025fediih}, FedICI~\cite{yuintegrating}, and FedSSA~\cite{yu2026heterogeneity}). In addition, we include a purely local training baseline in which each client trains independently without any federated aggregation. The details of these methods are summarized below.

\textbf{FedAvg}~\cite{mcmahan2017communication}: A foundational FL method in which clients train local models independently and periodically upload updates to a central server. The server aggregates the received model parameters via simple averaging and broadcasts the resulting global model back to all clients.

\textbf{FedProx}~\cite{MLSYS2020_1f5fe839}: A personalized FL method that augments the local objective with a proximal term penalizing divergence between local and global model parameters. This regularization stabilizes local updates and enables clients to learn personalized models while still benefiting from global knowledge.

\textbf{FedPer}~\cite{Arivazhagan2019}: A personalized FL method that aggregates only the backbone network parameters across clients during federated rounds, while the classification layer remains personalized and is updated locally on each client.

\textbf{GCFL}~\cite{NEURIPS2021_9c6947bd}: A representative GFL baseline method originally designed for vertical federated scenarios such as molecular property prediction. GCFL employs a bi-partitioning strategy that recursively divides clients into two disjoint groups based on gradient similarity. This procedure is analogous to clustered federated learning~\cite{sattler2021clustered}. Model aggregation is subsequently performed independently within each group.

\textbf{FedGNN}~\cite{wu2021fedgnn}: A general GFL method that improves local performance by exchanging node embeddings across clients. When nodes in different clients share identical neighborhoods, FedGNN transfers the corresponding embeddings to expand local subgraphs, thereby enriching local information through cross-client knowledge transfer.

\textbf{FedSage+}~\cite{NEURIPS2021_34adeb8e}: A GFL baseline method that reconstructs missing edges between subgraphs. Each client receives node representations from other clients and computes gradients based on the distance between local features and received representations. These gradients are sent back to train the neighbor generators on other clients, facilitating reconstruction of missing inter-subgraph edges.

\textbf{FED-PUB}~\cite{baek2023personalized}: A personalized GFL method that performs similarity-aware model aggregation. It estimates inter-subgraph similarity by evaluating local model outputs on a shared test graph and performs weighted aggregation accordingly. In addition, each client learns a personalized sparse mask to selectively update only parameters relevant to its local subgraph.

\textbf{FedGTA}~\cite{li2023fedgta}: A personalized GFL method that estimates inter-subgraph similarity through topology-aware local smoothing confidence and mixed moments of neighbor features. In each communication round, clients upload these statistics together with their model parameters to the server, which then performs client-specific weighted aggregation based on the estimated similarities.

\textbf{AdaFGL}~\cite{li2024adafgl}: A personalized GFL method that adopts a decoupled two-stage personalization strategy. In the first stage, standard federated training is performed to obtain a global knowledge extractor via parameter aggregation in the final round. In the second stage, each client conducts personalized training on its local subgraph by leveraging the extracted federated knowledge.

\textbf{FedTAD}~\cite{zhu2024fedtad}: A personalized GFL method that introduces a generator to synthesize pseudo-graphs for data-free knowledge distillation. This enables effective knowledge transfer from local models to the global model, thereby alleviating the adverse effects of data heterogeneity.

\textbf{FedIIH}~\cite{wentao2025fediih}: A personalized GFL method that jointly models inter-heterogeneity and intra-heterogeneity. Inter-heterogeneity is captured from a multi-level global perspective via hierarchical variational inference, which enables accurate estimation of inter-subgraph similarity based on graph data distributions. Intra-heterogeneity is addressed by disentangling each subgraph into multiple latent factors, which allows for fine-grained personalization.

\textbf{FedICI}~\cite{yuintegrating}: A personalized GFL method that integrates commonality and individuality to address heterogeneity by exploiting the spectral properties of graphs. Low-frequency components are aggregated in a low-rank space to maximize commonality consistency across clients, while high-frequency components are projected onto an orthogonal space to minimize correlations among client-specific individualities.

\textbf{FedSSA}~\cite{yu2026heterogeneity}: A personalized GFL method that mitigates both node-feature and structural heterogeneity through semantic and structural alignment. On one hand, it construct cluster-level representative distributions to minimize the divergence between local and cluster-level distributions, thereby enabling semantic knowledge sharing. On the other hand, local spectral characteristics are aligned with cluster-level counterparts for structural knowledge sharing.

\textbf{Local}: A non-federated baseline in which each client trains its local model independently using only its own graph data, without any federated aggregation or inter-client collaboration.
\subsection{Training Details}
For all baseline methods except FedAvg, FedSage+, FedGTA, FedIIH, FedICI, FedSSA, and our proposed FedGMC, we employ a two-layer Graph Convolutional Network (GCN)~\cite{kipf2017semisupervised} followed by a classifier as the backbone architecture. The hyperparameters of each baseline are configured exactly as reported in their original papers to ensure faithful reproduction of results. For FedSage+, we adopt GraphSAGE~\cite{NIPS2017_5dd9db5e} as the encoder and additionally train a missing-neighbor generator to reconstruct edges across subgraphs. FedGTA utilizes a Graph Attention Multi-Layer Perceptron (GAMLP)~\cite{35346783539121} as its backbone, paired with a linear classifier. FedIIH leverages a node feature projection layer from DisenGCN~\cite{pmlrv97ma19a} to obtain disentangled node representations, which are subsequently classified by a Multi-Layer Perceptron (MLP). Furthermore, both FedICI and FedSSA employ a spectral GNN (\textit{i.e.}, UniFilter~\cite{huanguniversal}) to extract node representations before feeding them into an MLP for node classification. In contrast, our proposed FedGMC utilizes H\textsubscript{2}GCN~\cite{zhu2020beyond} to simultaneously extract ego-embeddings and multi-hop neighbor embeddings within a unified framework. To guarantee a fair comparison, we adopt the identical GNN backbone for both the local training baseline and FedAvg as the one used in FedGMC. All experiments are conducted under identical conditions, including the same embedding dimension, optimizer settings, and learning rate based on the validation performance.

\section{Additional Experiments}
In this section, we present supplementary experiments to further validate the effectiveness and robustness of our proposed FedGMC. Specifically, we provide (i) additional experimental results on eleven datasets under the overlapping partitioning setting, (ii) comprehensive ablation studies across eleven datasets under both non-overlapping and overlapping partitioning settings, (iii) additional convergence curves on representative datasets under both non-overlapping and overlapping partitioning settings, (iv) extended sensitivity analyses on four key hyperparameters, (v) further case studies on additional datasets, and (vi) computational efficiency analysis by reporting time consumption (seconds) per communication round.

\subsection{Additional Experiments on Overlapping Subgraph Partitioning Setting}\label{additional_tables}
Owing to space constraints, the main paper reports results exclusively under the non-overlapping partitioning setting. Here, we present the corresponding experimental results on the same eleven datasets under the overlapping partitioning setting in Tab.~\ref{table2} and Tab.~\ref{table4}. As shown in Tab.~\ref{table2} and Tab.~\ref{table4}, FedGMC consistently achieves the competitive performance across all datasets and varying numbers of clients. For instance, in Tab.~\ref{table4}, FedGMC obtains an average accuracy of 64.92\%, obviously surpassing the second-best method (\textit{i.e.}, FedICI). These results further confirm the strong capability of FedGMC in addressing heterogeneity within federated graph learning scenarios.

\begin{table*}[t]
      \centering
          \scriptsize
          \caption{Accuracy (\%) of methods on six \textbf{homophilic} graph datasets under \textbf{overlapping} subgraph partitioning setting.}
          \label{table2}
    \renewcommand{\arraystretch}{0.9} 
           \scalebox{0.75}{
      \begin{tabular}{lcccccccccc}
      \hline
      \rowcolor{gray!50}
      \multicolumn{1}{c}{} & \multicolumn{3}{c}{Cora}                                                    & \multicolumn{3}{c}{CiteSeer}                                                      & \multicolumn{3}{c}{PubMed}                                                  & -              \\ \cline{2-11} 
      Methods              & 10 Clients              & 30 Clients              & 50 Clients              & 10 Clients                    & 30 Clients              & 50 Clients              & 10 Clients              & 30 Clients              & 50 Clients              & -              \\ \hline
      \rowcolor{gray!20}
      Local                & 73.98$\pm$0.25          & 71.65$\pm$0.12          & 76.63$\pm$0.10          & 65.12$\pm$0.08                & 64.54$\pm$0.42          & 66.68$\pm$0.44          & 82.32$\pm$0.07          & 80.72$\pm$0.16          & 80.54$\pm$0.11          & -              \\ \hline
      FedAvg~\cite{mcmahan2017communication}               & 76.48$\pm$0.36          & 53.99$\pm$0.98          & 53.99$\pm$4.53          & 69.48$\pm$0.15                & 66.15$\pm$0.64          & 66.51$\pm$1.00          & 82.67$\pm$0.11          & 82.05$\pm$0.12          & 80.24$\pm$0.35          & -              \\
      \rowcolor{gray!20}
      FedProx~\cite{MLSYS2020_1f5fe839}              & 77.85$\pm$0.50          & 51.38$\pm$1.74          & 56.27$\pm$9.04          & 69.39$\pm$0.35                & 66.11$\pm$0.75          & 66.53$\pm$0.43          & 82.63$\pm$0.17          & 82.13$\pm$0.13          & 80.50$\pm$0.46          & -              \\
      FedPer~\cite{Arivazhagan2019}               & 78.73$\pm$0.31          & 74.18$\pm$0.24          & 74.42$\pm$0.37          & 69.81$\pm$0.28                & 65.19$\pm$0.81          & 67.64$\pm$0.44          & 85.31$\pm$0.06          & 84.35$\pm$0.38          & 83.94$\pm$0.10          & -              \\
      \rowcolor{gray!20}
      GCFL~\cite{NEURIPS2021_9c6947bd}                 & 78.84$\pm$0.26          & 73.41$\pm$0.27          & 76.63$\pm$0.16          & 69.48$\pm$0.39                & 64.92$\pm$0.18          & 65.98$\pm$0.30          & 83.59$\pm$0.25          & 80.77$\pm$0.12          & 81.36$\pm$0.11          & -              \\
      FedGNN~\cite{wu2021fedgnn}               & 70.63$\pm$0.83          & 61.38$\pm$2.33          & 56.91$\pm$0.82          & 68.72$\pm$0.39                & 59.98$\pm$1.52          & 58.98$\pm$0.98          & 84.25$\pm$0.07          & 82.02$\pm$0.22          & 81.85$\pm$0.10          & -              \\
      \rowcolor{gray!20}
      FedSage+\cite{NEURIPS2021_34adeb8e}             & 77.52$\pm$0.46          & 51.99$\pm$0.42          & 55.48$\pm$11.5          & 68.75$\pm$0.48                & 65.97$\pm$0.02          & 65.93$\pm$0.30          & 82.77$\pm$0.08          & 82.14$\pm$0.11          & 80.31$\pm$0.68          & -              \\
      FED-PUB~\cite{baek2023personalized}              & 79.60$\pm$0.12          & 75.40$\pm$0.54          & 77.84$\pm$0.23          & 70.58$\pm$0.20                & 68.33$\pm$0.45          & 69.21$\pm$0.30          & 85.70$\pm$0.08          & 85.16$\pm$0.10          & 84.84$\pm$0.12          & -              \\
      \rowcolor{gray!20}
      FedGTA~\cite{li2023fedgta}               & 76.42$\pm$0.62          & 75.63$\pm$0.33          & 77.69$\pm$0.14          & 70.43$\pm$0.08 & 71.71$\pm$0.33          & 69.19$\pm$0.32          & 85.34$\pm$0.42          & 84.99$\pm$0.05          & 84.47$\pm$0.06          & -              \\
      AdaFGL~\cite{li2024adafgl}               & 78.50$\pm$0.19          & 75.80$\pm$0.23          & 74.41$\pm$0.11          & 72.63$\pm$0.15 & 68.18$\pm$0.31          & 62.90$\pm$0.75          & 85.58$\pm$0.23          & 85.85$\pm$0.41          & 84.45$\pm$0.07          & -              \\
	    \rowcolor{gray!20}
      FedTAD~\cite{zhu2024fedtad}              & 79.29$\pm$0.78          & 60.92$\pm$2.17          & 68.08$\pm$0.44          & \textbf{73.47$\pm$0.16}            & 67.74$\pm$0.57          & 63.51$\pm$0.68          & 82.98$\pm$0.20          & 82.11$\pm$0.15          & 81.63$\pm$0.19          & -              \\ 
      FedIIH~\cite{wentao2025fediih}        & 80.57$\pm$0.23  & 76.82$\pm$0.24 & \textbf{78.58$\pm$0.25} & \underline{73.16$\pm$0.18}       & 72.27$\pm$0.21    & 69.56$\pm$0.11 & 85.87$\pm$0.03 & 86.65$\pm$0.11 & \underline{85.65$\pm$0.12} & -              \\
      \rowcolor{gray!20}
      FedICI~\cite{yuintegrating} & \textbf{81.05$\pm$0.11}          & \textbf{76.92$\pm$0.09}          & \underline{78.46$\pm$0.11}          & 72.27$\pm$0.06          & \textbf{72.66$\pm$0.14}          & \underline{69.83$\pm$0.14}          & \underline{87.30$\pm$0.08}          & \underline{86.90$\pm$0.16}          & 85.33$\pm$0.12         & -              \\
      FedSSA~\cite{yu2026heterogeneity}                & \underline{81.02$\pm$0.09}    & \underline{76.89$\pm$0.06} & \textbf{78.58$\pm$0.15}          & 72.19$\pm$0.08                & 72.38$\pm$0.13    & 69.68$\pm$0.10 & 86.40$\pm$0.09 & 86.83$\pm$0.06 & 85.20$\pm$0.04 & -              \\ \hline  
      \rowcolor{yellow!30}
      FedGMC (Ours)                         & 80.20$\pm$0.09    & 76.25$\pm$0.09 & 78.21$\pm$0.10          & 72.23$\pm$0.13                & \underline{72.43$\pm$0.12}    & \textbf{69.88$\pm$0.16} & \textbf{87.42$\pm$0.07} & \textbf{87.11$\pm$0.07} & \textbf{85.70$\pm$0.17} & -              \\ \hline
      \rowcolor{gray!50}
                           & \multicolumn{3}{c}{Amazon-Computer}                                         & \multicolumn{3}{c}{Amazon-Photo}                                                  & \multicolumn{3}{c}{ogbn-arxiv}                                              & Avg.            \\ \cline{2-11} 
      Methods              & 10 Clients              & 30 Clients              & 50 Clients              & 10 Clients                    & 30 Clients              & 50 Clients              & 10 Clients              & 30 Clients              & 50 Clients              & All              \\ \hline
      \rowcolor{gray!20}
      Local                & 88.50$\pm$0.20          & 86.66$\pm$0.09          & 87.04$\pm$0.02          & 92.17$\pm$0.12                & 90.16$\pm$0.12          & 90.42$\pm$0.15          & 62.52$\pm$0.07          & 61.32$\pm$0.04          & 60.04$\pm$0.04          & 76.72          \\ \hline
      FedAvg~\cite{mcmahan2017communication}               & 88.99$\pm$0.19          & 83.37$\pm$0.47          & 76.34$\pm$0.12          & 92.91$\pm$0.07                & 89.30$\pm$0.22          & 74.19$\pm$0.57          & 63.56$\pm$0.02          & 59.72$\pm$0.06          & 60.94$\pm$0.24          & 73.38          \\
      \rowcolor{gray!20}
      FedProx~\cite{MLSYS2020_1f5fe839}              & 88.84$\pm$0.20          & 83.84$\pm$0.89          & 76.60$\pm$0.47          & 92.67$\pm$0.19                & 89.17$\pm$0.40          & 72.36$\pm$2.06          & 63.52$\pm$0.11          & 59.86$\pm$0.16          & 61.12$\pm$0.04          & 73.38          \\
      FedPer~\cite{Arivazhagan2019}               & 89.30$\pm$0.04          & 87.99$\pm$0.23          & 88.22$\pm$0.27          & 92.88$\pm$0.24                & 91.23$\pm$0.16          & 90.92$\pm$0.38          & 63.97$\pm$0.08          & 62.29$\pm$0.04          & 61.24$\pm$0.11          & 78.42          \\
      \rowcolor{gray!20}
      GCFL~\cite{NEURIPS2021_9c6947bd}                 & 89.01$\pm$0.22          & 87.24$\pm$0.09          & 87.02$\pm$0.22          & 92.45$\pm$0.10                & 90.58$\pm$0.11          & 90.54$\pm$0.08          & 63.24$\pm$0.02          & 61.66$\pm$0.10          & 60.32$\pm$0.01          & 77.61          \\
      FedGNN~\cite{wu2021fedgnn}               & 88.15$\pm$0.09          & 87.00$\pm$0.10          & 83.96$\pm$0.88          & 91.47$\pm$0.11                & 87.91$\pm$1.34          & 78.90$\pm$6.46          & 63.08$\pm$0.19          & 60.09$\pm$0.04          & 60.51$\pm$0.11          & 73.66          \\
      \rowcolor{gray!20}
      FedSage+\cite{NEURIPS2021_34adeb8e}             & 89.24$\pm$0.15          & 81.33$\pm$1.20          & 76.72$\pm$0.39          & 92.76$\pm$0.05                & 88.69$\pm$0.99          & 72.41$\pm$1.36          & 63.24$\pm$0.02          & 59.90$\pm$0.12          & 60.95$\pm$0.09          & 73.12          \\
      FED-PUB~\cite{baek2023personalized}              & 89.98$\pm$0.08          & 89.15$\pm$0.06          & 88.76$\pm$0.14          & 93.22$\pm$0.07                & 92.01$\pm$0.07          & 91.71$\pm$0.11          & 64.18$\pm$0.04          & 63.34$\pm$0.12          & 62.55$\pm$0.12          & 79.53          \\
      \rowcolor{gray!20}
      FedGTA~\cite{li2023fedgta}               & 90.10$\pm$0.18          & 88.79$\pm$0.27          & 88.15$\pm$0.21          & 93.13$\pm$0.14                & 92.49$\pm$0.06          & 91.77$\pm$0.06          & 55.98$\pm$0.09          & 56.76$\pm$0.07          & 57.89$\pm$0.09          & 78.39          \\
      AdaFGL~\cite{li2024adafgl}        & 80.49$\pm$0.17          & 80.42$\pm$0.20          & 82.12$\pm$0.10          & 89.24$\pm$0.18          & 88.34$\pm$0.22          & 87.68$\pm$0.10          & 56.81$\pm$0.06                   & 55.17$\pm$0.18                   & 54.82$\pm$0.26                         & 75.74          \\ 
      \rowcolor{gray!20}
      FedTAD~\cite{zhu2024fedtad}        & 79.09$\pm$5.63          & 79.48$\pm$0.85          & 77.05$\pm$0.07          & 81.94$\pm$3.09          & 86.58$\pm$1.75          & 84.38$\pm$1.33          & 58.45$\pm$0.15                   & 57.75$\pm$0.54                   & 56.52$\pm$0.14                         & 73.39          \\
      FedIIH~\cite{wentao2025fediih}        & 90.15$\pm$0.04 & 89.56$\pm$0.19                 & \textbf{89.99$\pm$0.11} & 93.38$\pm$0.08                & 94.17$\pm$0.04 & 93.25$\pm$0.16 & 66.69$\pm$0.09          & 66.10$\pm$0.03         & 65.67$\pm$0.06 & 81.01 \\
      \rowcolor{gray!20}
      FedICI~\cite{yuintegrating} & 90.39$\pm$0.08          & \underline{89.73$\pm$0.13}          & 89.41$\pm$0.05          & 93.68$\pm$0.07          & 94.38$\pm$0.09          & 93.32$\pm$0.09          & 67.36$\pm$0.06          & 66.25$\pm$0.02          & 65.69$\pm$0.09         & \underline{81.16}              \\
      FedSSA~\cite{yu2026heterogeneity}                & \underline{90.41$\pm$0.07} & 89.65$\pm$0.12 & 89.34$\pm$0.11 & \underline{93.72$\pm$0.15}       & \underline{94.44$\pm$0.11} & \underline{93.36$\pm$0.14} & \underline{67.44$\pm$0.12} & \underline{66.38$\pm$0.13} & \underline{65.73$\pm$0.07} & 81.09 \\ \hline
      \rowcolor{yellow!30}
      FedGMC (Ours)                       & \textbf{91.21$\pm$0.10} & \textbf{89.80$\pm$0.09} & \underline{89.67$\pm$0.10} & \textbf{93.78$\pm$0.15}       & \textbf{94.79$\pm$0.13} & \textbf{93.54$\pm$0.16} & \textbf{67.73$\pm$0.17} & \textbf{66.44$\pm$0.13} & \textbf{65.87$\pm$0.09} & \textbf{81.24} \\ \hline
      \end{tabular}
      }
  \end{table*}

\begin{table*}[t]
      \centering
      \scriptsize
      \caption{Comparisons on five \textbf{heterophilic} graph datasets under \textbf{overlapping} subgraph partitioning setting. Accuracy (\%) is reported for \textit{Roman-empire} and \textit{Amazon-ratings}, and AUC (\%) is reported for \textit{Minesweeper}, \textit{Tolokers}, and \textit{Questions}.}
      \label{table4}
    \renewcommand{\arraystretch}{0.9} 
       \scalebox{0.75}{
  \begin{tabular}{lcccccccccc}
  \hline
  \rowcolor{gray!50}
  \textbf{}     & \multicolumn{3}{c}{Roman-empire}                                            & \multicolumn{3}{c}{Amazon-ratings}                                          & \multicolumn{3}{c}{Minesweeper}                                                   & -              \\ \cline{2-11} 
  Methods       & 10 Clients              & 30 Clients              & 50 Clients              & 10 Clients              & 30 Clients              & 50 Clients              & 10 Clients              & 30 Clients              & 50 Clients                    & -              \\ \hline
  \rowcolor{gray!20}
  Local         & 39.47$\pm$0.03          & 34.43$\pm$0.14          & 31.28$\pm$0.18          & 41.43$\pm$0.04          & 41.81$\pm$0.14          & 42.57$\pm$0.12          & 67.98$\pm$0.07          & 64.39$\pm$0.10          & 62.73$\pm$0.23                & -              \\ \hline
  FedAvg~\cite{mcmahan2017communication}        & 40.89$\pm$0.25          & 38.66$\pm$0.08          & 36.71$\pm$0.20          & 39.86$\pm$0.06          & 41.40$\pm$0.02          & 41.02$\pm$0.16          & 69.06$\pm$0.07          & 67.95$\pm$0.04          & 66.89$\pm$0.08                & -              \\
  \rowcolor{gray!20}
  FedProx~\cite{MLSYS2020_1f5fe839}       & 36.63$\pm$0.14          & 35.31$\pm$0.17          & 33.61$\pm$0.59          & 37.53$\pm$0.09          & 37.43$\pm$0.08          & 37.40$\pm$0.07          & 68.27$\pm$0.05          & 66.75$\pm$0.19          & 66.03$\pm$0.16                & -              \\
  FedPer~\cite{Arivazhagan2019}        & 23.66$\pm$3.27          & 23.27$\pm$3.09          & 22.23$\pm$3.58          & 32.33$\pm$4.23          & 31.58$\pm$0.54          & 34.48$\pm$2.25          & 61.85$\pm$1.02          & 60.13$\pm$1.38          & 60.06$\pm$3.61                & -              \\
  \rowcolor{gray!20}
  GCFL~\cite{NEURIPS2021_9c6947bd}          & 39.97$\pm$0.89          & 38.63$\pm$0.49          & 36.87$\pm$0.31          & 39.54$\pm$0.41          & 42.12$\pm$0.11          & 41.27$\pm$0.22          & 69.16$\pm$0.13          & 68.02$\pm$0.10          & 66.93$\pm$1.34               & -              \\
  FedGNN~\cite{wu2021fedgnn}        & 37.46$\pm$0.12          & 36.47$\pm$0.24          & 34.92$\pm$0.26          & 36.58$\pm$0.16          & 36.77$\pm$0.12          & 36.95$\pm$0.15          & 68.59$\pm$0.21          & 67.30$\pm$0.17          & 66.41$\pm$0.23                & -              \\
  \rowcolor{gray!20}
  FedSage+\cite{NEURIPS2021_34adeb8e}      & 57.48$\pm$0.15          & 42.55$\pm$0.20          & 37.13$\pm$0.18          & 36.86$\pm$0.21          & 36.71$\pm$0.19          & 37.03$\pm$0.09          & \underline{76.64$\pm$0.08} & 70.56$\pm$0.25 & 70.34$\pm$0.27 & -              \\
  FED-PUB~\cite{baek2023personalized}       & 43.80$\pm$0.25          & 40.46$\pm$0.16          & 37.73$\pm$0.09          & 42.25$\pm$0.25    &  42.30$\pm$0.06    & 42.88$\pm$0.34 & 69.11$\pm$0.13          & 67.76$\pm$0.24          & 67.52$\pm$0.14                & -              \\
  \rowcolor{gray!20}
  FedGTA~\cite{li2023fedgta}        & 59.86$\pm$0.04    & 58.32$\pm$0.09    & 57.57$\pm$0.21    & 40.81$\pm$0.24          & 39.44$\pm$0.06          & 39.37$\pm$0.04          & 70.64$\pm$0.40          & 67.99$\pm$1.60          & 67.20$\pm$1.35                & -              \\
  AdaFGL~\cite{li2024adafgl}        & 64.44$\pm$0.03          & 61.77$\pm$0.02          & 59.55$\pm$0.01          & 39.39$\pm$0.05          & 41.19$\pm$0.15          & 40.71$\pm$0.25          & 69.07$\pm$0.72                   & 68.34$\pm$1.82                   & 66.80$\pm$1.31                         & -          \\
  \rowcolor{gray!20}
  FedTAD~\cite{zhu2024fedtad}        & 44.14$\pm$0.13          & 41.94$\pm$0.18          & 40.82$\pm$0.01          & 39.53$\pm$0.17          & 40.69$\pm$0.13          & 40.58$\pm$0.26          & 69.27$\pm$0.33                   & 68.43$\pm$0.05                   & 67.23$\pm$0.08                         & -          \\
  FedIIH~\cite{wentao2025fediih} & 65.48$\pm$0.12 & 63.32$\pm$0.06 & 62.42$\pm$0.10          & 42.63$\pm$0.02           & 42.40$\pm$0.05          & 42.65$\pm$0.21              & 69.35$\pm$0.25    &  68.09$\pm$0.26    &  67.37$\pm$0.14 & -         \\
  \rowcolor{gray!20}
  FedICI~\cite{yuintegrating} & \underline{65.83$\pm$0.15}          & \underline{63.99$\pm$0.17}          & \textbf{62.97$\pm$0.09}          & \underline{42.84$\pm$0.08}          & \underline{42.63$\pm$0.10}          & 42.96$\pm$0.11          & 75.74$\pm$0.15          & \underline{75.14$\pm$0.16}          & \underline{71.68$\pm$0.08}         & -              \\
  FedSSA~\cite{yu2026heterogeneity}                & 65.66$\pm$0.07 & 63.80$\pm$0.09 & 62.75$\pm$0.14 & 42.83$\pm$0.06 & 42.52$\pm$0.10 & \underline{42.97$\pm$0.15}    & 75.65$\pm$0.11    & 72.60$\pm$0.06    & 71.31$\pm$0.05 & -              \\ \hline
  \rowcolor{yellow!30}
 FedGMC (Ours) & \textbf{67.22$\pm$0.10} & \textbf{64.64$\pm$0.07} & \underline{62.91$\pm$0.10} & \textbf{42.93$\pm$0.09} & \textbf{42.69$\pm$0.12} & \textbf{43.37$\pm$0.16}    & \textbf{77.44$\pm$0.15}    & \textbf{75.24$\pm$0.15}    & \textbf{72.40$\pm$0.07} & -              \\ \hline
 \rowcolor{gray!50}
                & \multicolumn{3}{c}{Tolokers}                                                & \multicolumn{3}{c}{Questions}                                               & \multicolumn{4}{c}{Avg.}                                                                           \\ \cline{2-11} 
  Methods       & 10 Clients              & 30 Clients              & 50 Clients              & 10 Clients              & 30 Clients              & 50 Clients              & 10 Clients              & 30 Clients              & 50 Clients                    & All            \\ \hline
  \rowcolor{gray!20}
  Local         & 73.83$\pm$0.03          & 69.01$\pm$0.31          & 66.63$\pm$0.20          & 63.17$\pm$0.02          & 57.17$\pm$0.08          & 56.13$\pm$0.02          & 57.18                   & 53.36                   & 51.87                         & 54.14          \\ \hline
  FedAvg~\cite{mcmahan2017communication}        & 72.99$\pm$0.40          & 58.51$\pm$0.27          & 55.47$\pm$0.42          & 62.80$\pm$0.63          & 58.88$\pm$0.18          & 60.78$\pm$0.27          & 57.12                   & 53.08                   & 52.17                         & 54.12          \\
  \rowcolor{gray!20}
  FedProx~\cite{MLSYS2020_1f5fe839}       & 54.49$\pm$1.69          & 45.59$\pm$0.41          & 41.49$\pm$0.45          & 52.53$\pm$0.34          & 51.54$\pm$0.41          & 50.72$\pm$0.40          & 49.89                   & 47.32                   & 45.85                         & 47.69          \\
  FedPer~\cite{Arivazhagan2019}        & 39.60$\pm$0.11          & 59.44$\pm$0.79          & 41.92$\pm$0.06          & 61.31$\pm$0.29          & 53.41$\pm$1.53          & 50.29$\pm$0.10          & 43.75                   & 45.57                   & 41.80                         & 43.70          \\
  \rowcolor{gray!20}
  GCFL~\cite{NEURIPS2021_9c6947bd}          & 70.61$\pm$0.55          & 59.72$\pm$0.50          & 57.64$\pm$0.71          & 62.84$\pm$0.60          & 59.46$\pm$0.68          & 60.24$\pm$0.41          & 56.42                   & 53.59                   & 52.59                         & 54.20          \\
  FedGNN~\cite{wu2021fedgnn}        & 56.21$\pm$1.20          & 46.85$\pm$0.31          & 42.18$\pm$0.45          & 53.25$\pm$0.15          & 51.90$\pm$0.15          & 51.22$\pm$0.14          & 50.42                   & 47.86                   & 46.34                         & 48.20          \\
  \rowcolor{gray!20}
  FedSage+\cite{NEURIPS2021_34adeb8e}      & \underline{74.54$\pm$0.26} &  70.88$\pm$0.17    &  69.61$\pm$0.13    & 64.22$\pm$0.10          &  65.34$\pm$0.20    &  62.76$\pm$0.28    & 61.95             & 57.21             & 55.37              & 58.18    \\
  FED-PUB~\cite{baek2023personalized}       & 74.17$\pm$0.29    & 70.35$\pm$0.54          & 66.80$\pm$0.85          &  65.39$\pm$2.44    & 58.38$\pm$1.19          & 60.73$\pm$0.74          & 58.94                   & 55.85                   & 55.13                         & 56.64          \\
  \rowcolor{gray!20}
  FedGTA~\cite{li2023fedgta}        & 70.34$\pm$1.53          & 59.59$\pm$1.10          & 56.12$\pm$1.48          & 63.20$\pm$1.27          & 58.11$\pm$1.06          & 60.99$\pm$0.77          & 60.97                   & 56.69                   & 56.25                         & 57.97          \\
  AdaFGL~\cite{li2024adafgl}        & 70.01$\pm$1.91          & 58.94$\pm$1.11          & 56.25$\pm$1.35          & 61.90$\pm$0.80          & 58.93$\pm$2.18          & 60.68$\pm$0.63          & 60.96                   & 57.83                   & 56.80                         & 58.53          \\
  \rowcolor{gray!20}
  FedTAD~\cite{zhu2024fedtad}        & 69.34$\pm$1.26          & 62.11$\pm$0.27          & 56.39$\pm$0.52          & 61.96$\pm$0.54         & 59.24$\pm$0.36          & 60.24$\pm$0.92          & 56.85                   & 54.48                   & 53.05                         & 54.79          \\ 
  FedIIH~\cite{wentao2025fediih}     & 71.67$\pm$0.02          & 71.69$\pm$0.12          & 69.99$\pm$0.03      & 68.79$\pm$0.09  & \underline{66.98$\pm$0.04} & 64.73$\pm$0.35          & 63.58          & 62.50          & 61.43                & 62.50 \\
  \rowcolor{gray!20}
  FedICI~\cite{yuintegrating} & 74.42$\pm$0.06          & 72.23$\pm$0.14          & 70.55$\pm$0.12          & \underline{69.41$\pm$0.09}          & 66.68$\pm$0.06          & \underline{65.57$\pm$0.09}          & \underline{65.65}          & \underline{64.13}          & \underline{62.75}         & \underline{64.18}              \\
  FedSSA~\cite{yu2026heterogeneity}                & 74.33$\pm$0.09          & \underline{72.27$\pm$0.07} & \underline{71.03$\pm$0.09}         & 69.39$\pm$0.05    & 66.43$\pm$0.13  & 64.94$\pm$0.12 & 65.57          & 63.52          & 62.60                & 63.90             \\ \hline
  \rowcolor{yellow!30}
  FedGMC (Ours)        & \textbf{75.43$\pm$0.08}          & \textbf{73.97$\pm$0.09} & \textbf{72.60$\pm$0.14}         & \textbf{69.81$\pm$0.08}    & \textbf{67.38$\pm$0.10}  & \textbf{65.84$\pm$0.14} & \textbf{66.57}          & \textbf{64.78}          & \textbf{63.42}                & \textbf{64.92} \\ \hline
      \end{tabular}
       }
\end{table*}

\subsection{Additional Ablation Studies}\label{ap_ablation_study}
To further evaluate the contribution of each key component in our proposed FedGMC, we carry out ablation studies on eleven datasets under both non-overlapping and overlapping partitioning settings. Specifically, since there are three essential components in our FedGMC (\textit{i.e.}, semantic manifold calibration, structural manifold calibration, and global
manifold refinement), we employ `FedGMC (w/o Semantic)', `FedGMC (w/o Structural)', and `FedGMC (w/o Refinement)' to represent the reduced models, respectively. Therefore, we consider ablation studies, which are shown in Tab.~\ref{table_ablation}. We can observe that the performances consistently degrade across all datasets when any individual component is removed. This demonstrates that each component is essential and contributes significantly to the overall effectiveness of our proposed FedGMC.

\begin{table*}[t]
    \centering
    \scriptsize
    \caption{Ablation studies are conducted under both non-overlapping and overlapping partitioning settings on eleven datasets.}
    \label{table_ablation}
    \renewcommand{\arraystretch}{0.9} 
    \scalebox{0.75}{
    \begin{tabular}{lcccccc}
      \hline
  \rowcolor{gray!50}
  & \multicolumn{6}{c}{Cora}                                                                                                                                \\ \hline
  Methods & \begin{tabular}[c]{@{}c@{}}non-overlapping\\ 5 clients\end{tabular} & \begin{tabular}[c]{@{}c@{}}non-overlapping\\ 10 clients\end{tabular} & \begin{tabular}[c]{@{}c@{}}non-overlapping\\ 20 clients\end{tabular} & \begin{tabular}[c]{@{}c@{}}overlapping\\ 10 clients\end{tabular} & \begin{tabular}[c]{@{}c@{}}overlapping\\ 30 clients\end{tabular} & \begin{tabular}[c]{@{}c@{}}overlapping\\ 50 clients\end{tabular} \\ \hline
  FedGMC (w/o Semantic)    & 80.77$\pm$0.69 ($\downarrow$ 3.39)                                  & 79.18$\pm$0.72 ($\downarrow$ 4.47)                                   & 82.09$\pm$0.45 ($\downarrow$ 2.46)                              & 77.57$\pm$0.13 ($\downarrow$ 2.63)                               & 74.36$\pm$0.19 ($\downarrow$ 1.89)                    & 76.60$\pm$0.06 ($\downarrow$ 1.61)                               \\ \hline
  FedGMC (w/o Structural)     & 74.23$\pm$0.05 ($\downarrow$ 9.93)                                 & 75.26$\pm$0.17 ($\downarrow$ 8.39)                                   & 79.79$\pm$0.09 ($\downarrow$ 4.76)                                   & 78.10$\pm$0.05 ($\downarrow$ 2.10)                               & 74.26$\pm$0.08 ($\downarrow$ 1.99)                                   & 74.33$\pm$0.35 ($\downarrow$ 3.88)                               \\ \hline
  FedGMC (w/o Refinement)    & 81.31$\pm$0.40 ($\downarrow$ 2.85)                      & 80.56$\pm$0.95 ($\downarrow$ 3.09)                                   & 81.35$\pm$0.56 ($\downarrow$ 3.20)                                   & 78.53$\pm$0.18 ($\downarrow$ 1.67)                               & 75.20$\pm$0.06 ($\downarrow$ 1.05)                                   & 77.16$\pm$0.09 ($\downarrow$ 1.05)     \\ \hline
  FedGMC    & \textbf{84.16$\pm$0.11}                               & \textbf{83.65$\pm$0.16}                                 & \textbf{84.55$\pm$0.13}                                 & \textbf{80.20$\pm$0.09}                               & \textbf{76.25$\pm$0.09}                                   & \textbf{78.21$\pm$0.10}     \\ \hline
  \rowcolor{gray!50}
  & \multicolumn{6}{c}{CiteSeer}                                                                                                                                \\ \hline
  Methods & \begin{tabular}[c]{@{}c@{}}non-overlapping\\ 5 clients\end{tabular} & \begin{tabular}[c]{@{}c@{}}non-overlapping\\ 10 clients\end{tabular} & \begin{tabular}[c]{@{}c@{}}non-overlapping\\ 20 clients\end{tabular} & \begin{tabular}[c]{@{}c@{}}overlapping\\ 10 clients\end{tabular} & \begin{tabular}[c]{@{}c@{}}overlapping\\ 30 clients\end{tabular} & \begin{tabular}[c]{@{}c@{}}overlapping\\ 50 clients\end{tabular} \\ \hline
  FedGMC (w/o Semantic)    & 70.19$\pm$0.56 ($\downarrow$ 3.69)                                  & 74.10$\pm$0.36 ($\downarrow$ 3.49)                                   & 71.82$\pm$0.18 ($\downarrow$ 2.53)                              & 69.27$\pm$0.64 ($\downarrow$ 2.96)                               & 69.67$\pm$0.11 ($\downarrow$ 2.76)                    & 67.87$\pm$0.05 ($\downarrow$ 2.01)                               \\ \hline
  FedGMC (w/o Structural)     & 66.36$\pm$0.17 ($\downarrow$ 7.52)                                 & 74.11$\pm$0.51 ($\downarrow$ 3.48)                                   & 71.61$\pm$0.28 ($\downarrow$ 2.74)                                   & 67.58$\pm$0.66 ($\downarrow$ 4.65)                               & 69.78$\pm$0.20 ($\downarrow$ 2.65)                                   & 66.75$\pm$0.02 ($\downarrow$ 3.13)                               \\ \hline
  FedGMC (w/o Refinement)    & 70.09$\pm$0.10 ($\downarrow$ 3.79)                      & 74.80$\pm$0.66 ($\downarrow$ 2.79)                                   & 72.97$\pm$0.81 ($\downarrow$ 1.38)                                   & 69.81$\pm$0.53 ($\downarrow$ 2.42)                               & 70.52$\pm$0.50 ($\downarrow$ 1.91)                                   & 68.14$\pm$0.36 ($\downarrow$ 1.74)     \\ \hline
  FedGMC    & \textbf{73.88$\pm$0.11}                               & \textbf{77.59$\pm$0.06}                                 & \textbf{74.35$\pm$0.10}                                 & \textbf{72.23$\pm$0.13}                               & \textbf{72.43$\pm$0.12}                                   & \textbf{69.88$\pm$0.16}     \\ \hline
  \rowcolor{gray!50}
  & \multicolumn{6}{c}{PubMed}                                                                                                                                \\ \hline
  Methods & \begin{tabular}[c]{@{}c@{}}non-overlapping\\ 5 clients\end{tabular} & \begin{tabular}[c]{@{}c@{}}non-overlapping\\ 10 clients\end{tabular} & \begin{tabular}[c]{@{}c@{}}non-overlapping\\ 20 clients\end{tabular} & \begin{tabular}[c]{@{}c@{}}overlapping\\ 10 clients\end{tabular} & \begin{tabular}[c]{@{}c@{}}overlapping\\ 30 clients\end{tabular} & \begin{tabular}[c]{@{}c@{}}overlapping\\ 50 clients\end{tabular} \\ \hline
  FedGMC (w/o Semantic)    & 86.99$\pm$0.64 ($\downarrow$ 1.40)                                  & 86.20$\pm$0.60 ($\downarrow$ 2.47)                                   & 85.63$\pm$0.17 ($\downarrow$ 2.48)                              & 86.20$\pm$0.07 ($\downarrow$ 1.22)                               & 86.31$\pm$0.12 ($\downarrow$ 0.80)                    & 84.34$\pm$0.31 ($\downarrow$ 1.36)                               \\ \hline
  FedGMC (w/o Structural)     & 86.62$\pm$0.17 ($\downarrow$ 1.77)                                 & 85.36$\pm$0.07 ($\downarrow$ 3.31)                                   & 84.76$\pm$0.27 ($\downarrow$ 3.35)                                   & 85.29$\pm$0.15 ($\downarrow$ 2.13)                               & 85.52$\pm$0.28 ($\downarrow$ 1.59)                                   & 83.54$\pm$0.17 ($\downarrow$ 2.16)                               \\ \hline
  FedGMC (w/o Refinement)    & 86.78$\pm$0.29 ($\downarrow$ 1.61)                      & 86.50$\pm$0.08 ($\downarrow$ 2.17)                                   & 86.24$\pm$0.14 ($\downarrow$ 1.87)                                   & 86.17$\pm$0.06 ($\downarrow$ 1.25)                               & 86.26$\pm$0.16 ($\downarrow$ 0.85)                                   & 84.74$\pm$0.10 ($\downarrow$ 0.96)     \\ \hline
  FedGMC    & \textbf{88.39$\pm$0.09}                               & \textbf{88.67$\pm$0.12}                                 & \textbf{88.11$\pm$0.11}                                 & \textbf{87.42$\pm$0.07}                               & \textbf{87.11$\pm$0.07}                                   & \textbf{85.70$\pm$0.17}     \\ \hline
  \rowcolor{gray!50}
  & \multicolumn{6}{c}{Amazon-Computer}                                                                                                                                \\ \hline
  Methods & \begin{tabular}[c]{@{}c@{}}non-overlapping\\ 5 clients\end{tabular} & \begin{tabular}[c]{@{}c@{}}non-overlapping\\ 10 clients\end{tabular} & \begin{tabular}[c]{@{}c@{}}non-overlapping\\ 20 clients\end{tabular} & \begin{tabular}[c]{@{}c@{}}overlapping\\ 10 clients\end{tabular} & \begin{tabular}[c]{@{}c@{}}overlapping\\ 30 clients\end{tabular} & \begin{tabular}[c]{@{}c@{}}overlapping\\ 50 clients\end{tabular} \\ \hline
  FedGMC (w/o Semantic)    & 89.82$\pm$0.61 ($\downarrow$ 1.56)                                  & 90.06$\pm$0.66 ($\downarrow$ 1.79)                                   & 88.98$\pm$0.34 ($\downarrow$ 1.94)                              & 89.70$\pm$0.04 ($\downarrow$ 1.51)                               & 88.50$\pm$0.24 ($\downarrow$ 1.30)                    & 88.19$\pm$0.08 ($\downarrow$ 1.48)                               \\ \hline
  FedGMC (w/o Structural)     & 86.49$\pm$0.46 ($\downarrow$ 4.89)                                 & 86.53$\pm$0.12 ($\downarrow$ 5.32)                                   & 87.15$\pm$0.16 ($\downarrow$ 3.77)                                   & 85.79$\pm$0.15 ($\downarrow$ 5.42)                               & 84.65$\pm$0.05 ($\downarrow$ 5.15)                                   & 85.86$\pm$0.09 ($\downarrow$ 3.81)                               \\ \hline
  FedGMC (w/o Refinement)    & 89.44$\pm$0.27 ($\downarrow$ 1.94)                      & 90.18$\pm$0.31 ($\downarrow$ 1.67)                                   & 89.19$\pm$0.25 ($\downarrow$ 1.73)                                   & 90.31$\pm$0.09 ($\downarrow$ 0.90)                               & 88.84$\pm$0.13 ($\downarrow$ 0.96)                                   & 88.35$\pm$0.15 ($\downarrow$ 1.32)     \\ \hline
  FedGMC    & \textbf{91.38$\pm$0.10}                               & \textbf{91.85$\pm$0.14}                                 & \textbf{90.92$\pm$0.09}                                 & \textbf{91.21$\pm$0.10}                               & \textbf{89.80$\pm$0.09}                                   & \textbf{89.67$\pm$0.10}     \\ \hline
  \rowcolor{gray!50}
  & \multicolumn{6}{c}{Amazon-Photo}                                                                                                                                \\ \hline
  Methods & \begin{tabular}[c]{@{}c@{}}non-overlapping\\ 5 clients\end{tabular} & \begin{tabular}[c]{@{}c@{}}non-overlapping\\ 10 clients\end{tabular} & \begin{tabular}[c]{@{}c@{}}non-overlapping\\ 20 clients\end{tabular} & \begin{tabular}[c]{@{}c@{}}overlapping\\ 10 clients\end{tabular} & \begin{tabular}[c]{@{}c@{}}overlapping\\ 30 clients\end{tabular} & \begin{tabular}[c]{@{}c@{}}overlapping\\ 50 clients\end{tabular} \\ \hline
  FedGMC (w/o Semantic)    & 92.52$\pm$0.08 ($\downarrow$ 1.46)                                  & 92.53$\pm$0.06 ($\downarrow$ 2.22)                                   & 91.53$\pm$0.30 ($\downarrow$ 2.98)                              & 92.61$\pm$0.12 ($\downarrow$ 1.17)                               & 93.53$\pm$0.10 ($\downarrow$ 1.26)                    & 92.31$\pm$0.18 ($\downarrow$ 1.23)                               \\ \hline
  FedGMC (w/o Structural)     & 91.44$\pm$0.36 ($\downarrow$ 2.54)                                 & 92.95$\pm$0.25 ($\downarrow$ 1.80)                                   & 92.78$\pm$0.16 ($\downarrow$ 1.73)                                   & 90.41$\pm$0.19 ($\downarrow$ 3.37)                               & 92.77$\pm$0.08 ($\downarrow$ 2.02)                                   & 91.72$\pm$0.09 ($\downarrow$ 1.82)                               \\ \hline
  FedGMC (w/o Refinement)    & 91.86$\pm$0.04 ($\downarrow$ 2.12)                      & 92.97$\pm$0.14 ($\downarrow$ 1.78)                                   & 93.40$\pm$0.22 ($\downarrow$ 1.11)                                   & 92.68$\pm$0.04 ($\downarrow$ 1.10)                               & 93.77$\pm$0.06 ($\downarrow$ 1.02)                                   & 92.57$\pm$0.06 ($\downarrow$ 0.97)     \\ \hline
  FedGMC    & \textbf{93.98$\pm$0.06}                               & \textbf{94.75$\pm$0.07}                                 & \textbf{94.51$\pm$0.10}                                 & \textbf{93.78$\pm$0.15}                               & \textbf{94.79$\pm$0.13}                                   & \textbf{93.54$\pm$0.16}     \\ \hline
  \rowcolor{gray!50}
  & \multicolumn{6}{c}{ogbn-arxiv}                                                                                                                                \\ \hline
  Methods & \begin{tabular}[c]{@{}c@{}}non-overlapping\\ 5 clients\end{tabular} & \begin{tabular}[c]{@{}c@{}}non-overlapping\\ 10 clients\end{tabular} & \begin{tabular}[c]{@{}c@{}}non-overlapping\\ 20 clients\end{tabular} & \begin{tabular}[c]{@{}c@{}}overlapping\\ 10 clients\end{tabular} & \begin{tabular}[c]{@{}c@{}}overlapping\\ 30 clients\end{tabular} & \begin{tabular}[c]{@{}c@{}}overlapping\\ 50 clients\end{tabular} \\ \hline
  FedGMC (w/o Semantic)    & 68.53$\pm$0.07 ($\downarrow$ 2.42)                                  & 67.45$\pm$0.21 ($\downarrow$ 2.23)                                   & 66.46$\pm$0.18 ($\downarrow$ 2.42)                              & 65.46$\pm$0.13 ($\downarrow$ 2.27)                               & 64.69$\pm$0.09 ($\downarrow$ 1.75)                    & 63.28$\pm$0.18 ($\downarrow$ 2.59)                               \\ \hline
  FedGMC (w/o Structural)     & 68.56$\pm$0.16 ($\downarrow$ 2.39)                                 & 67.79$\pm$0.13 ($\downarrow$ 1.89)                                   & 67.02$\pm$0.09 ($\downarrow$ 1.86)                                   & 65.32$\pm$0.09 ($\downarrow$ 2.41)                               & 64.07$\pm$0.10 ($\downarrow$ 2.37)                                   & 63.64$\pm$0.18 ($\downarrow$ 2.23)                               \\ \hline
  FedGMC (w/o Refinement)    & 68.50$\pm$0.38 ($\downarrow$ 2.45)                      & 67.25$\pm$0.09 ($\downarrow$ 2.43)                                   & 67.22$\pm$0.10 ($\downarrow$ 1.66)                                   & 65.78$\pm$0.08 ($\downarrow$ 1.95)                               & 65.50$\pm$0.26 ($\downarrow$ 0.94)                                   & 64.42$\pm$0.10 ($\downarrow$ 1.45)     \\ \hline
  FedGMC    & \textbf{70.95$\pm$0.08}                               & \textbf{69.68$\pm$0.15}                                 & \textbf{68.88$\pm$0.13}                                 & \textbf{67.73$\pm$0.17}                               & \textbf{66.44$\pm$0.13}                                   & \textbf{65.87$\pm$0.09}     \\ \hline
  \rowcolor{gray!50}
  & \multicolumn{6}{c}{Roman-empire}                                                                                                                                \\ \hline
  Methods & \begin{tabular}[c]{@{}c@{}}non-overlapping\\ 5 clients\end{tabular} & \begin{tabular}[c]{@{}c@{}}non-overlapping\\ 10 clients\end{tabular} & \begin{tabular}[c]{@{}c@{}}non-overlapping\\ 20 clients\end{tabular} & \begin{tabular}[c]{@{}c@{}}overlapping\\ 10 clients\end{tabular} & \begin{tabular}[c]{@{}c@{}}overlapping\\ 30 clients\end{tabular} & \begin{tabular}[c]{@{}c@{}}overlapping\\ 50 clients\end{tabular} \\ \hline
  FedGMC (w/o Semantic)    & 70.56$\pm$0.04 ($\downarrow$ 1.39)                                  & 68.08$\pm$0.26 ($\downarrow$ 2.63)                                   & 65.50$\pm$0.26 ($\downarrow$ 1.49)                              & 65.46$\pm$0.02 ($\downarrow$ 1.76)                               & 63.64$\pm$0.11 ($\downarrow$ 1.00)                    & 60.69$\pm$0.16 ($\downarrow$ 2.22)                               \\ \hline
  FedGMC (w/o Structural)     & 64.63$\pm$0.20 ($\downarrow$ 7.32)                                 & 63.04$\pm$0.13 ($\downarrow$ 7.67)                                   & 61.38$\pm$0.14 ($\downarrow$ 5.61)                                   & 64.00$\pm$0.10 ($\downarrow$ 3.22)                               & 61.43$\pm$0.04 ($\downarrow$ 3.21)                                   & 59.35$\pm$0.16 ($\downarrow$ 3.56)                               \\ \hline
  FedGMC (w/o Refinement)    & 70.89$\pm$0.26 ($\downarrow$ 1.06)                      & 69.48$\pm$0.08 ($\downarrow$ 1.23)                                   & 65.49$\pm$0.05 ($\downarrow$ 1.50)                                   & 65.58$\pm$0.21 ($\downarrow$ 1.64)                               & 63.44$\pm$0.15 ($\downarrow$ 1.20)                                   & 60.70$\pm$0.17 ($\downarrow$ 2.21)     \\ \hline
  FedGMC    & \textbf{71.95$\pm$0.07}                               & \textbf{70.71$\pm$0.15}                                 & \textbf{66.99$\pm$0.17}                                 & \textbf{67.22$\pm$0.10}                               & \textbf{64.64$\pm$0.07}                                   & \textbf{62.91$\pm$0.10}     \\ \hline
  \rowcolor{gray!50}
  & \multicolumn{6}{c}{Amazon-ratings}                                                                                                                 \\ \hline
  Methods & \begin{tabular}[c]{@{}c@{}}non-overlapping\\ 5 clients\end{tabular} & \begin{tabular}[c]{@{}c@{}}non-overlapping\\ 10 clients\end{tabular} & \begin{tabular}[c]{@{}c@{}}non-overlapping\\ 20 clients\end{tabular} & \begin{tabular}[c]{@{}c@{}}overlapping\\ 10 clients\end{tabular} & \begin{tabular}[c]{@{}c@{}}overlapping\\ 30 clients\end{tabular} & \begin{tabular}[c]{@{}c@{}}overlapping\\ 50 clients\end{tabular} \\ \hline
  FedGMC (w/o Semantic)    & 43.90$\pm$0.29 ($\downarrow$ 1.39)                                  & 43.94$\pm$0.34 ($\downarrow$ 2.63)                                   & 45.39$\pm$0.36 ($\downarrow$ 1.49)                              & 41.10$\pm$0.01 ($\downarrow$ 1.76)                               & 41.44$\pm$0.03 ($\downarrow$ 1.00)                    & 42.32$\pm$0.11 ($\downarrow$ 2.22)                               \\ \hline
  FedGMC (w/o Structural)     & 41.77$\pm$0.40 ($\downarrow$ 7.32)                                 & 42.25$\pm$0.13 ($\downarrow$ 7.67)                                   & 44.50$\pm$0.04 ($\downarrow$ 5.61)                                   & 39.62$\pm$0.20 ($\downarrow$ 3.22)                               & 40.50$\pm$0.03 ($\downarrow$ 3.21)                                   & 40.69$\pm$0.23 ($\downarrow$ 3.56)                               \\ \hline
  FedGMC (w/o Refinement)    & 44.28$\pm$0.16 ($\downarrow$ 1.06)                      & 44.54$\pm$0.66 ($\downarrow$ 1.23)                                   & 45.97$\pm$0.30 ($\downarrow$ 1.50)                                   & 41.36$\pm$0.20 ($\downarrow$ 1.64)                               & 41.34$\pm$0.01 ($\downarrow$ 1.20)                                   & 42.34$\pm$0.03 ($\downarrow$ 2.21)     \\ \hline
  FedGMC    & \textbf{46.11$\pm$0.10}                               & \textbf{46.09$\pm$0.12}                                 & \textbf{46.56$\pm$0.08}                                 & \textbf{42.93$\pm$0.09}                               & \textbf{42.69$\pm$0.12}                                   & \textbf{43.37$\pm$0.16}     \\ \hline
  \rowcolor{gray!50}
  & \multicolumn{6}{c}{Minesweeper}                                                                                                                 \\ \hline
  Methods & \begin{tabular}[c]{@{}c@{}}non-overlapping\\ 5 clients\end{tabular} & \begin{tabular}[c]{@{}c@{}}non-overlapping\\ 10 clients\end{tabular} & \begin{tabular}[c]{@{}c@{}}non-overlapping\\ 20 clients\end{tabular} & \begin{tabular}[c]{@{}c@{}}overlapping\\ 10 clients\end{tabular} & \begin{tabular}[c]{@{}c@{}}overlapping\\ 30 clients\end{tabular} & \begin{tabular}[c]{@{}c@{}}overlapping\\ 50 clients\end{tabular} \\ \hline
  FedGMC (w/o Semantic)    & 85.69$\pm$0.27 ($\downarrow$ 1.10)                                  & 84.72$\pm$0.24 ($\downarrow$ 1.32)                                   & 82.94$\pm$0.35 ($\downarrow$ 1.90)                              & 76.35$\pm$0.06 ($\downarrow$ 1.09)                               & 73.52$\pm$0.02 ($\downarrow$ 1.72)                    & 70.80$\pm$0.08 ($\downarrow$ 1.60)                               \\ \hline
  FedGMC (w/o Structural)     & 80.43$\pm$0.19 ($\downarrow$ 6.36)                                 & 80.62$\pm$0.32 ($\downarrow$ 5.42)                                   & 79.30$\pm$0.11 ($\downarrow$ 5.54)                                   & 71.13$\pm$0.14 ($\downarrow$ 6.31)                               & 70.63$\pm$0.28 ($\downarrow$ 4.61)                                   & 68.35$\pm$0.19 ($\downarrow$ 4.05)                               \\ \hline
  FedGMC (w/o Refinement)    & 85.48$\pm$0.13 ($\downarrow$ 1.31)                      & 84.83$\pm$0.25 ($\downarrow$ 1.21)                                   & 82.51$\pm$0.12 ($\downarrow$ 2.33)                                   & 76.39$\pm$0.02 ($\downarrow$ 1.05)                               & 73.50$\pm$0.10 ($\downarrow$ 1.74)                                   & 70.82$\pm$0.02 ($\downarrow$ 1.58)     \\ \hline
  FedGMC    & \textbf{86.79$\pm$0.07}                               & \textbf{86.04$\pm$0.08}                                 & \textbf{84.84$\pm$0.10}                                 & \textbf{77.44$\pm$0.15}                               & \textbf{75.24$\pm$0.15}                                   & \textbf{72.40$\pm$0.07}     \\ \hline
  \rowcolor{gray!50}
  & \multicolumn{6}{c}{Tolokers}                                                                                                                 \\ \hline
  Methods & \begin{tabular}[c]{@{}c@{}}non-overlapping\\ 5 clients\end{tabular} & \begin{tabular}[c]{@{}c@{}}non-overlapping\\ 10 clients\end{tabular} & \begin{tabular}[c]{@{}c@{}}non-overlapping\\ 20 clients\end{tabular} & \begin{tabular}[c]{@{}c@{}}overlapping\\ 10 clients\end{tabular} & \begin{tabular}[c]{@{}c@{}}overlapping\\ 30 clients\end{tabular} & \begin{tabular}[c]{@{}c@{}}overlapping\\ 50 clients\end{tabular} \\ \hline
  FedGMC (w/o Semantic)    & 75.18$\pm$0.34 ($\downarrow$ 1.31)                                  & 72.77$\pm$0.37 ($\downarrow$ 1.26)                                   & 70.97$\pm$0.68 ($\downarrow$ 1.36)                              & 73.89$\pm$0.27 ($\downarrow$ 1.54)                               & 72.16$\pm$0.03 ($\downarrow$ 1.81)                    & 71.52$\pm$0.08 ($\downarrow$ 1.08)                               \\ \hline
  FedGMC (w/o Structural)     & 71.73$\pm$0.05 ($\downarrow$ 4.76)                                 & 72.61$\pm$0.19 ($\downarrow$ 1.42)                                   & 71.18$\pm$0.34 ($\downarrow$ 1.15)                                   & 72.19$\pm$0.09 ($\downarrow$ 3.24)                               & 72.33$\pm$0.09 ($\downarrow$ 1.64)                                   & 71.09$\pm$0.29 ($\downarrow$ 1.51)                               \\ \hline
  FedGMC (w/o Refinement)    & 75.10$\pm$0.11 ($\downarrow$ 1.39)                      & 72.90$\pm$0.05 ($\downarrow$ 1.13)                                   & 71.16$\pm$0.21 ($\downarrow$ 1.17)                                   & 73.62$\pm$0.05 ($\downarrow$ 1.81)                               & 72.53$\pm$0.07 ($\downarrow$ 1.44)                                   & 71.48$\pm$0.04 ($\downarrow$ 1.12)     \\ \hline
  FedGMC    & \textbf{76.49$\pm$0.15}                               & \textbf{74.03$\pm$0.09}                                 & \textbf{72.33$\pm$0.06}                                 & \textbf{75.43$\pm$0.08}                               & \textbf{73.97$\pm$0.09}                                   & \textbf{72.60$\pm$0.14}     \\ \hline
  \rowcolor{gray!50}
  & \multicolumn{6}{c}{Questions}                                                                                                                 \\ \hline
  Methods & \begin{tabular}[c]{@{}c@{}}non-overlapping\\ 5 clients\end{tabular} & \begin{tabular}[c]{@{}c@{}}non-overlapping\\ 10 clients\end{tabular} & \begin{tabular}[c]{@{}c@{}}non-overlapping\\ 20 clients\end{tabular} & \begin{tabular}[c]{@{}c@{}}overlapping\\ 10 clients\end{tabular} & \begin{tabular}[c]{@{}c@{}}overlapping\\ 30 clients\end{tabular} & \begin{tabular}[c]{@{}c@{}}overlapping\\ 50 clients\end{tabular} \\ \hline
  FedGMC (w/o Semantic)    & 68.51$\pm$0.13 ($\downarrow$ 1.22)                                  & 68.48$\pm$0.41 ($\downarrow$ 1.24)                                   & 62.75$\pm$0.10 ($\downarrow$ 3.06)                              & 67.79$\pm$0.12 ($\downarrow$ 2.02)                               & 65.57$\pm$0.27 ($\downarrow$ 1.81)                    & 64.28$\pm$0.20 ($\downarrow$ 1.56)                               \\ \hline
  FedGMC (w/o Structural)     & 67.24$\pm$0.27 ($\downarrow$ 2.49)                                 & 66.33$\pm$0.23 ($\downarrow$ 3.39)                                   & 63.36$\pm$0.17 ($\downarrow$ 2.45)                                   & 67.57$\pm$0.34 ($\downarrow$ 2.24)                               & 65.73$\pm$0.38 ($\downarrow$ 1.65)                                   & 63.72$\pm$0.33 ($\downarrow$ 2.12)                               \\ \hline
  FedGMC (w/o Refinement)    & 67.89$\pm$0.20 ($\downarrow$ 1.84)                      & 66.47$\pm$0.45 ($\downarrow$ 3.25)                                   & 63.70$\pm$0.73 ($\downarrow$ 2.11)                                   & 67.65$\pm$0.15 ($\downarrow$ 2.16)                               & 66.24$\pm$0.07 ($\downarrow$ 1.14)                                   & 64.69$\pm$0.31 ($\downarrow$ 1.15)     \\ \hline
  FedGMC    & \textbf{69.73$\pm$0.08}                               & \textbf{69.72$\pm$0.14}                                 & \textbf{65.81$\pm$0.12}                                 & \textbf{69.81$\pm$0.08}                               & \textbf{67.38$\pm$0.10}                                   & \textbf{65.84$\pm$0.14}     \\ \hline
    \end{tabular}
    }
  \end{table*}

\subsection{Additional Convergence Curves}\label{ap_convergence_curves}
To further evaluate the convergence of our proposed FedGMC and the compared methods, we present convergence curves on eight datasets under non-overlapping partitioning setting (see Fig.~\ref{fig_additional_convergence1}), and on six datasets under overlapping partitioning setting (see Fig.~\ref{fig_additional_convergence2}). We can observe that our proposed FedGMC converges stably. In contrast, the convergence curves of typical methods such as GCFL (\textit{e.g.}, Fig.~\ref{fig_additional_convergence2}(f)) exhibit pronounced instability. This instability arises because GCFL relies on dynamic clustering to guide federated aggregation. However, under strong heterogeneity, local models can undergo substantial changes between communication rounds, leading to rapid fluctuations in clustering. These fluctuations result in an inconsistent aggregation process and hinder stable knowledge transfer across clients. Consequently, the convergence process becomes more erratic, resulting in the observed oscillations in the convergence curves.

\begin{figure*}[!t]
  \centering
  \subfloat[\footnotesize{\textit{Cora}}]{\includegraphics[width=0.25\columnwidth]{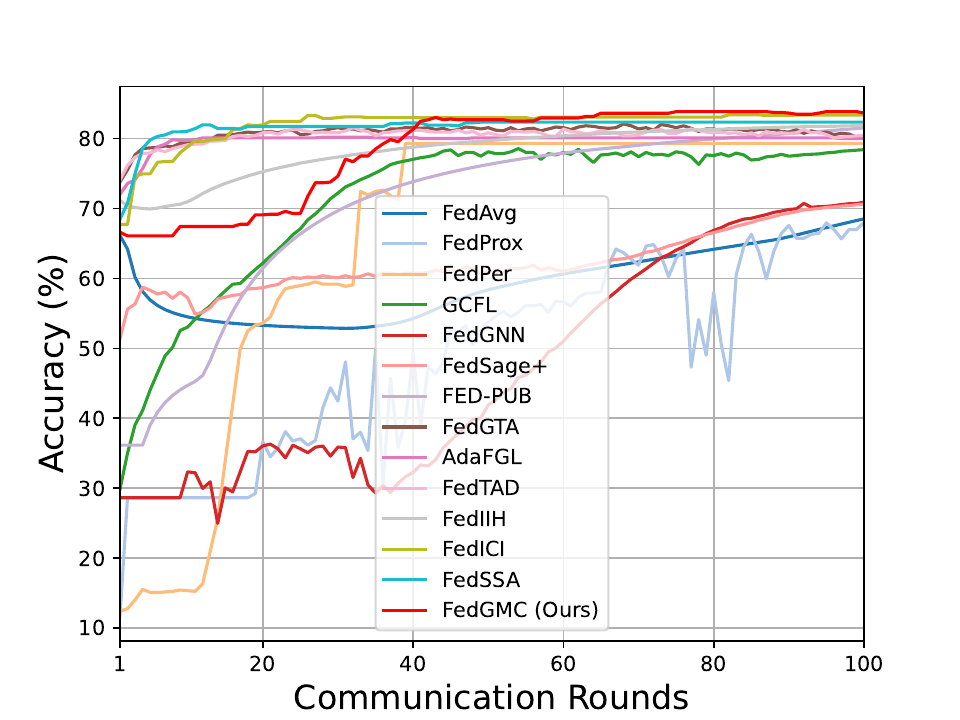}\label{fig_additional_convergence1_1}}
  \hfill
  \subfloat[\footnotesize{\textit{CiteSeer}}]{\includegraphics[width=0.25\columnwidth]{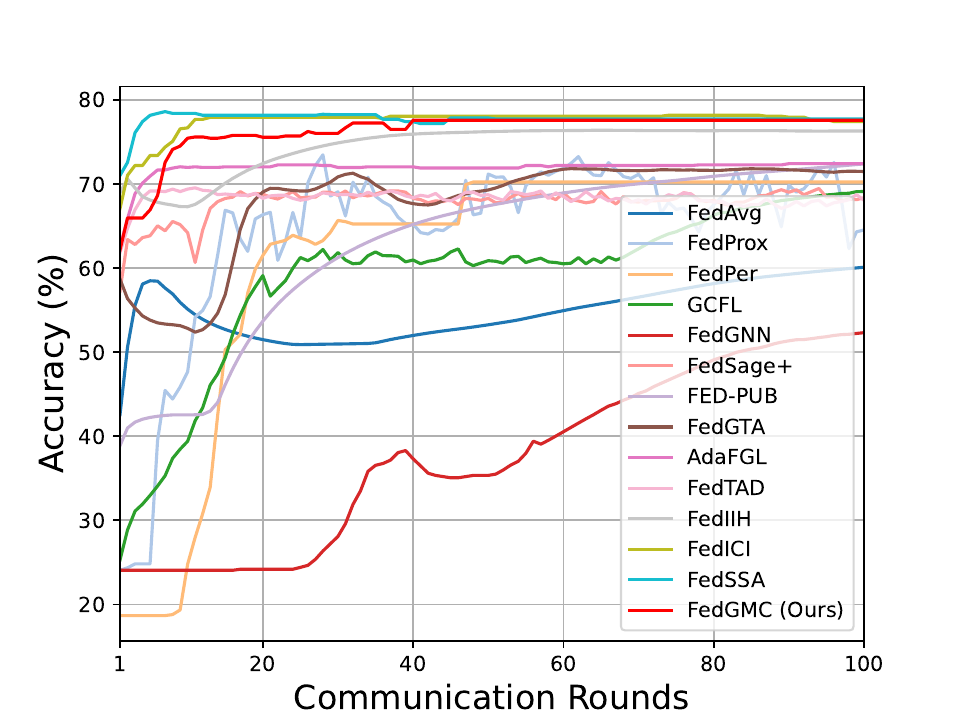}\label{fig_additional_convergence1_2}}
  \hfill
  \subfloat[\footnotesize{\textit{PubMed}}]{\includegraphics[width=0.25\columnwidth]{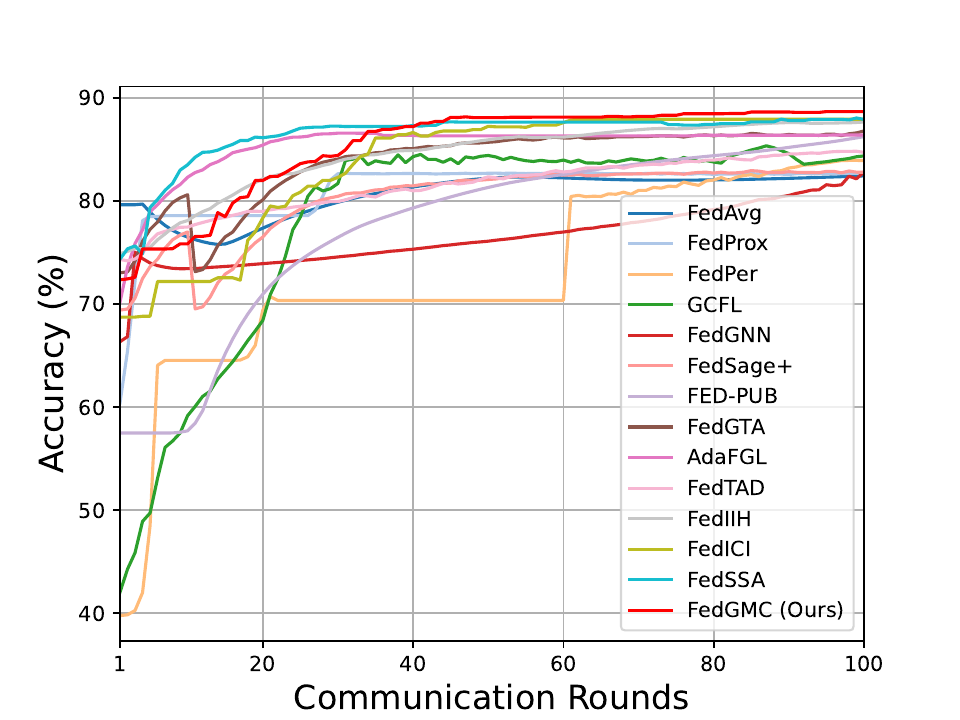}\label{fig_additional_convergence1_3}}
  \hfill
  \subfloat[\footnotesize{\textit{ogbn-arxiv}}]{\includegraphics[width=0.25\columnwidth]{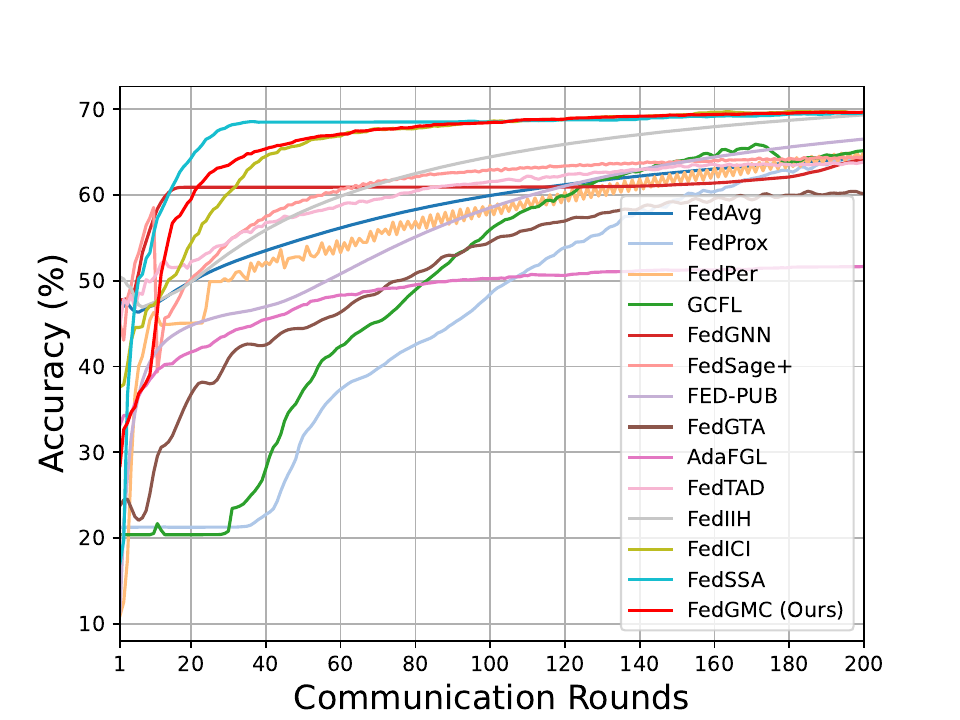}\label{fig_additional_convergence1_4}}
  \hfill
  \subfloat[\footnotesize{\textit{Roman-empire}}]{\includegraphics[width=0.25\columnwidth]{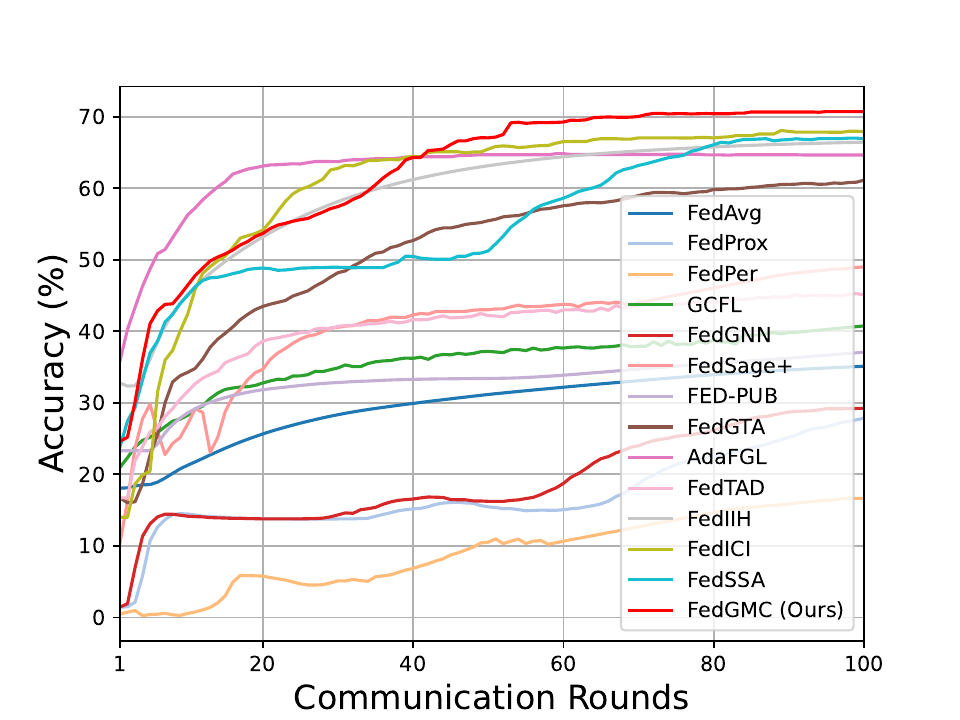}\label{fig_additional_convergence1_5}}
  \hfill
  \subfloat[\footnotesize{\textit{Minesweeper}}]{\includegraphics[width=0.25\columnwidth]{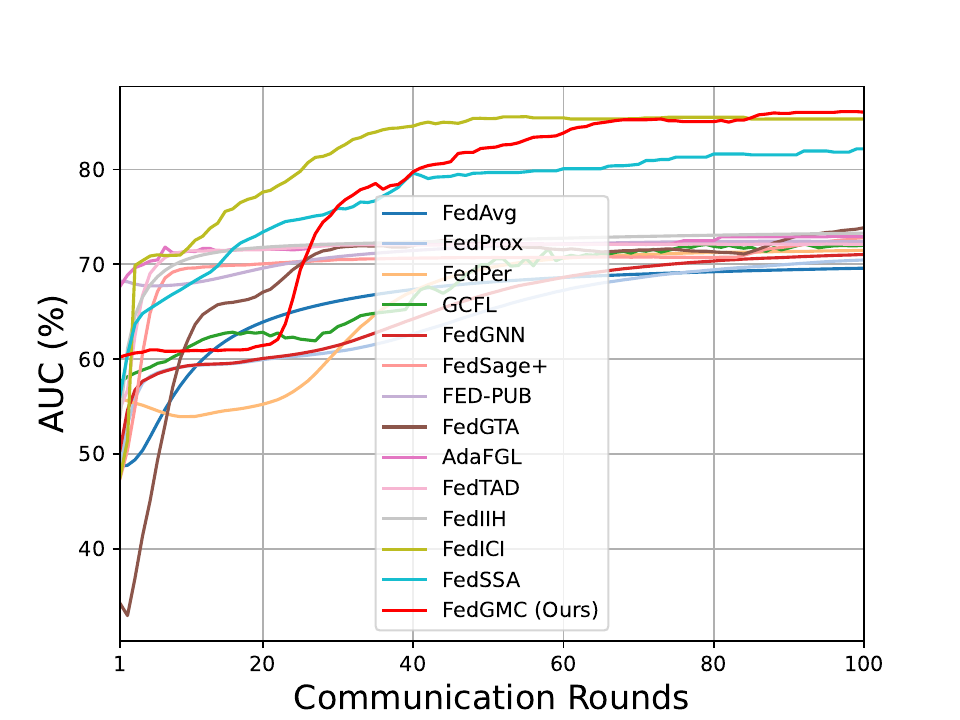}\label{fig_additional_convergence1_6}}
  \hfill
  \subfloat[\footnotesize{\textit{Tolokers}}]{\includegraphics[width=0.25\columnwidth]{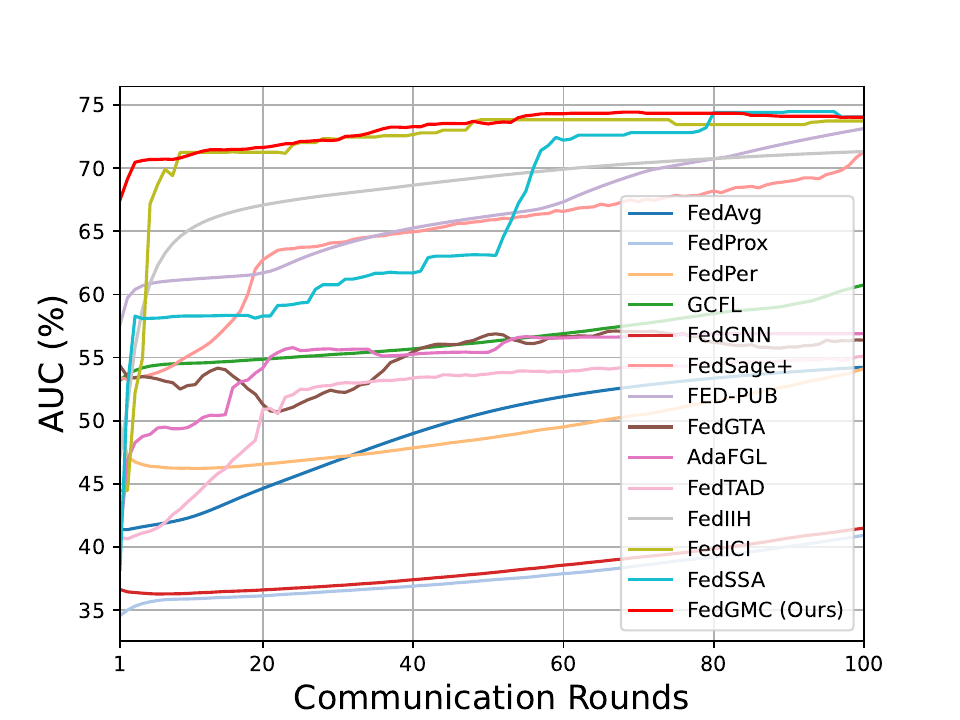}\label{fig_additional_convergence1_7}}
  \hfill
  \subfloat[\footnotesize{\textit{Questions}}]{\includegraphics[width=0.25\columnwidth]{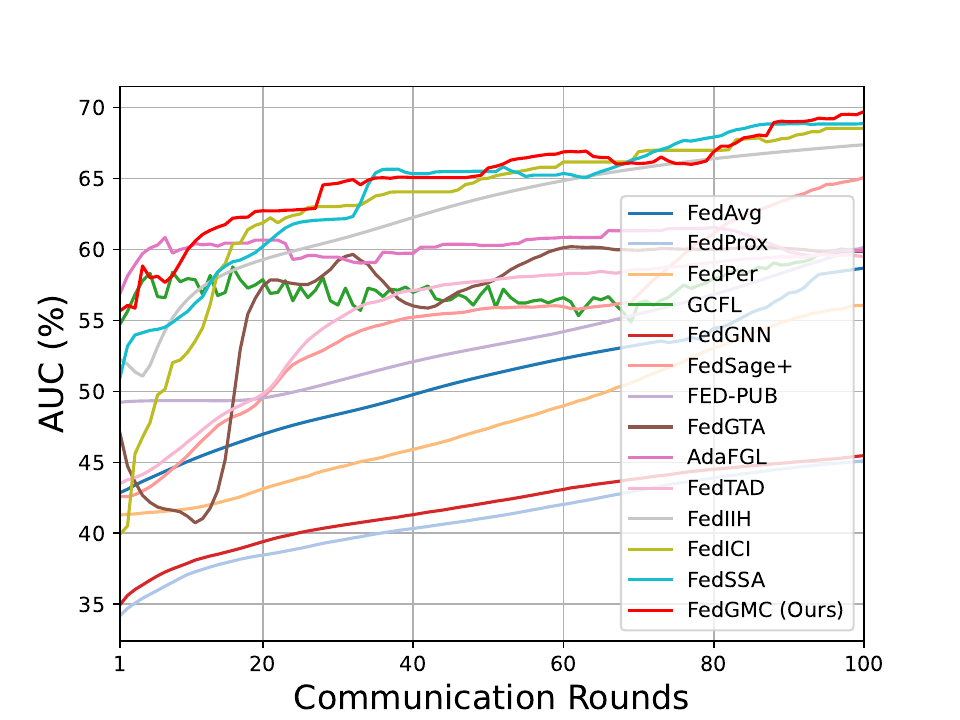}\label{fig_additional_convergence1_8}}
  \caption{Convergence curves on eight datasets under non-overlapping partitioning setting with 10 clients.}
  \label{fig_additional_convergence1}
\end{figure*}

\begin{figure*}[!t]
  \centering
  \subfloat[\footnotesize{\textit{Cora}}]{\includegraphics[width=0.33\columnwidth]{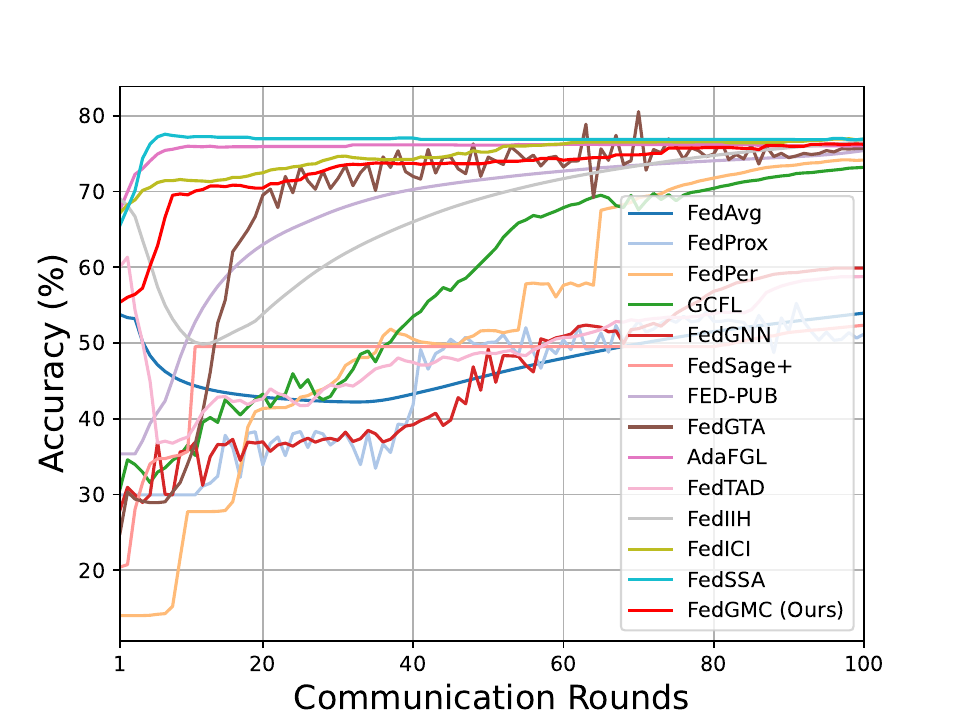}\label{fig_additional_convergence2_1}}
  \hfill
  \subfloat[\footnotesize{\textit{CiteSeer}}]{\includegraphics[width=0.33\columnwidth]{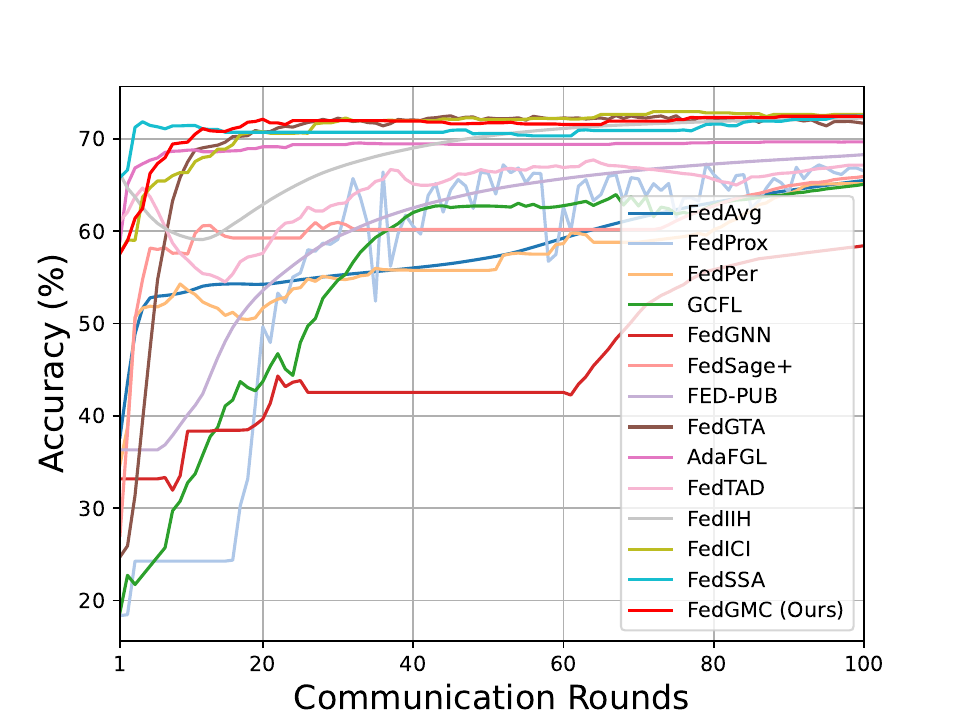}\label{fig_additional_convergence2_2}}
  \hfill
  \subfloat[\footnotesize{\textit{ogbn-arxiv}}]{\includegraphics[width=0.33\columnwidth]{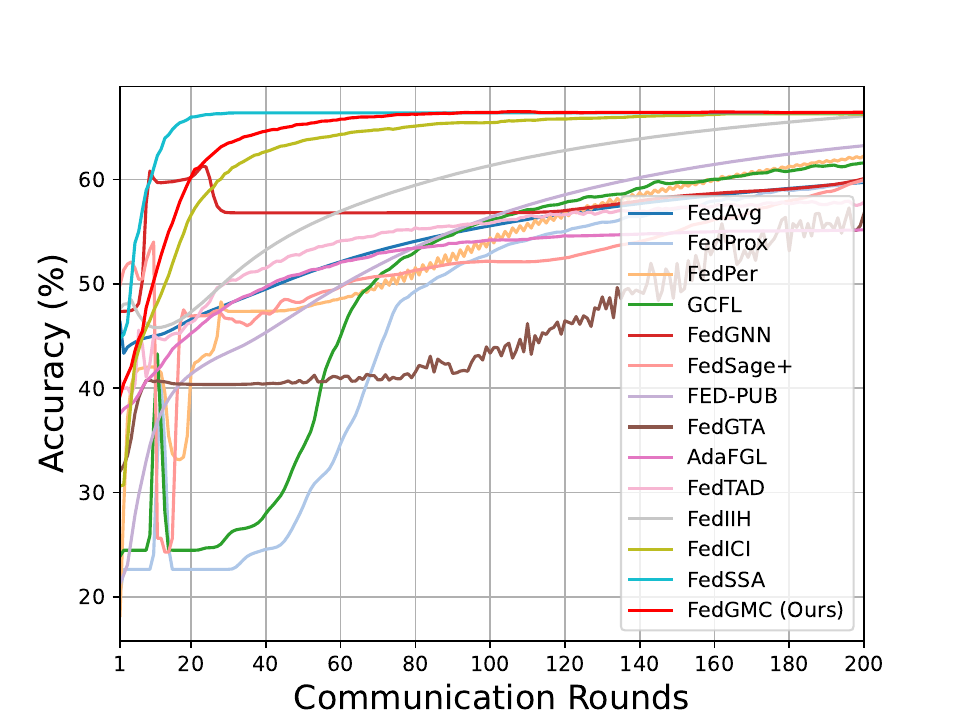}\label{fig_additional_convergence2_3}}
  \hfill
  \subfloat[\footnotesize{\textit{Roman-empire}}]{\includegraphics[width=0.33\columnwidth]{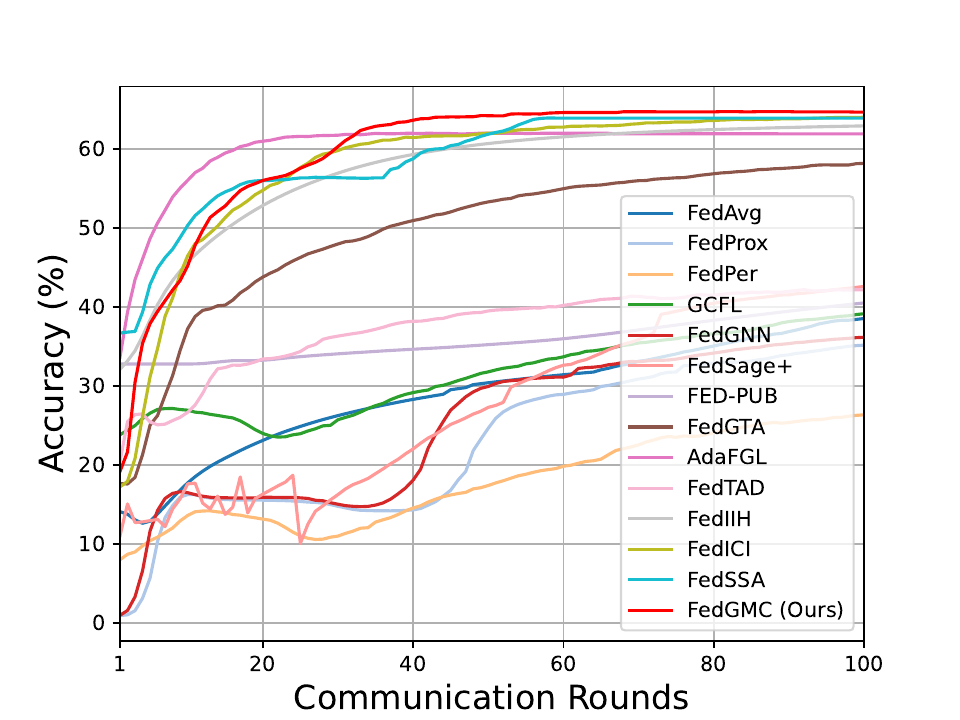}\label{fig_additional_convergence2_4}}
  \hfill
  \subfloat[\footnotesize{\textit{Minesweeper}}]{\includegraphics[width=0.33\columnwidth]{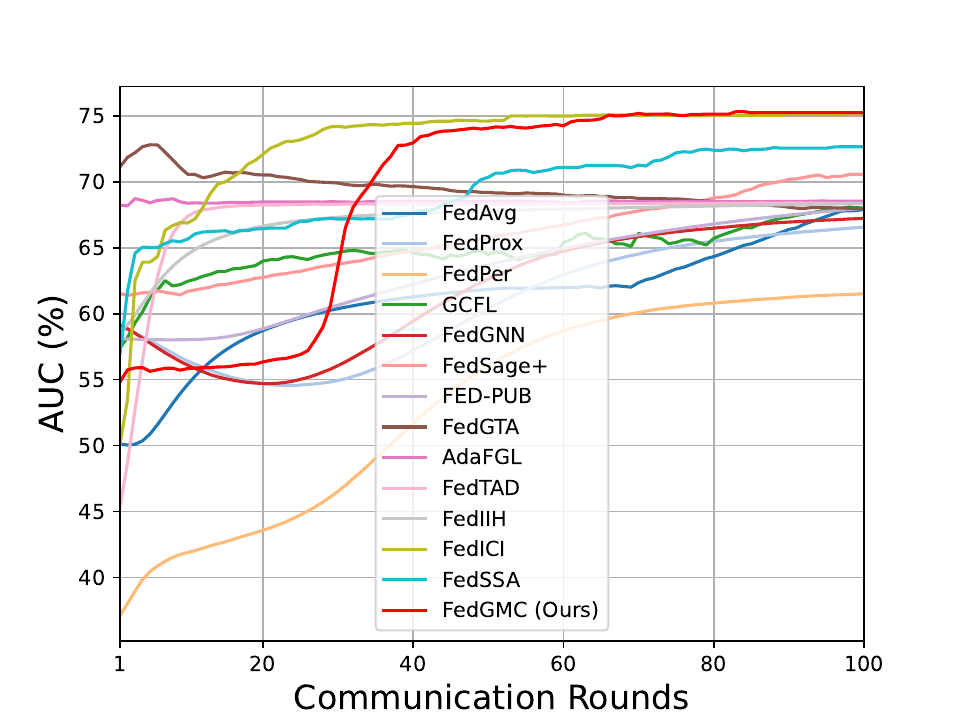}\label{fig_additional_convergence2_5}}
  \hfill
  \subfloat[\footnotesize{\textit{Questions}}]{\includegraphics[width=0.33\columnwidth]{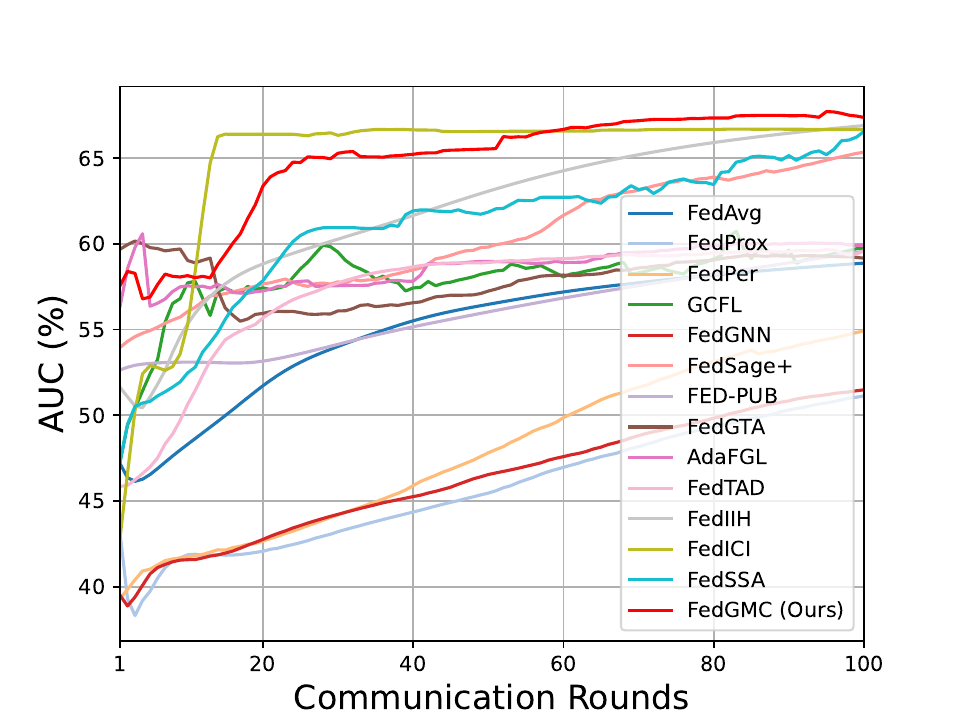}\label{fig_additional_convergence2_6}}
  \caption{Convergence curves on six datasets under overlapping partitioning settings with 30 clients.}
  \label{fig_additional_convergence2}
\end{figure*}

\subsection{Additional Sensitivity Analysis on Hyperparameters}\label{ap_sensitivity_analysis}
To further analyze the sensitivity of our proposed FedGMC to hyperparameters, we conduct additional sensitivity analyses on \textit{Cora} and \textit{Roman-empire} datasets. Specifically, we examine the impact of four key hyperparameters, namely the number of global structural manifold $Q$, the number of local structural manifold $B$, temperature hyperparameter $\tau$, and step size $\eta$. Fig.~\ref{fig_additional_sensitivity1} and Fig.~\ref{fig_additional_sensitivity2} present accuracy curves with variance bars under different values of hyperparameters. Experimental results show that our proposed FedGMC exhibits stable performance across a wide range of hyperparameters, which indicates that our FedGMC is not sensitive to the variation of hyperparameters.

\begin{figure}[t]
  \centering
  \subfloat[\footnotesize{\textit{Cora} under different $Q$}]{\includegraphics[width=0.2\columnwidth]{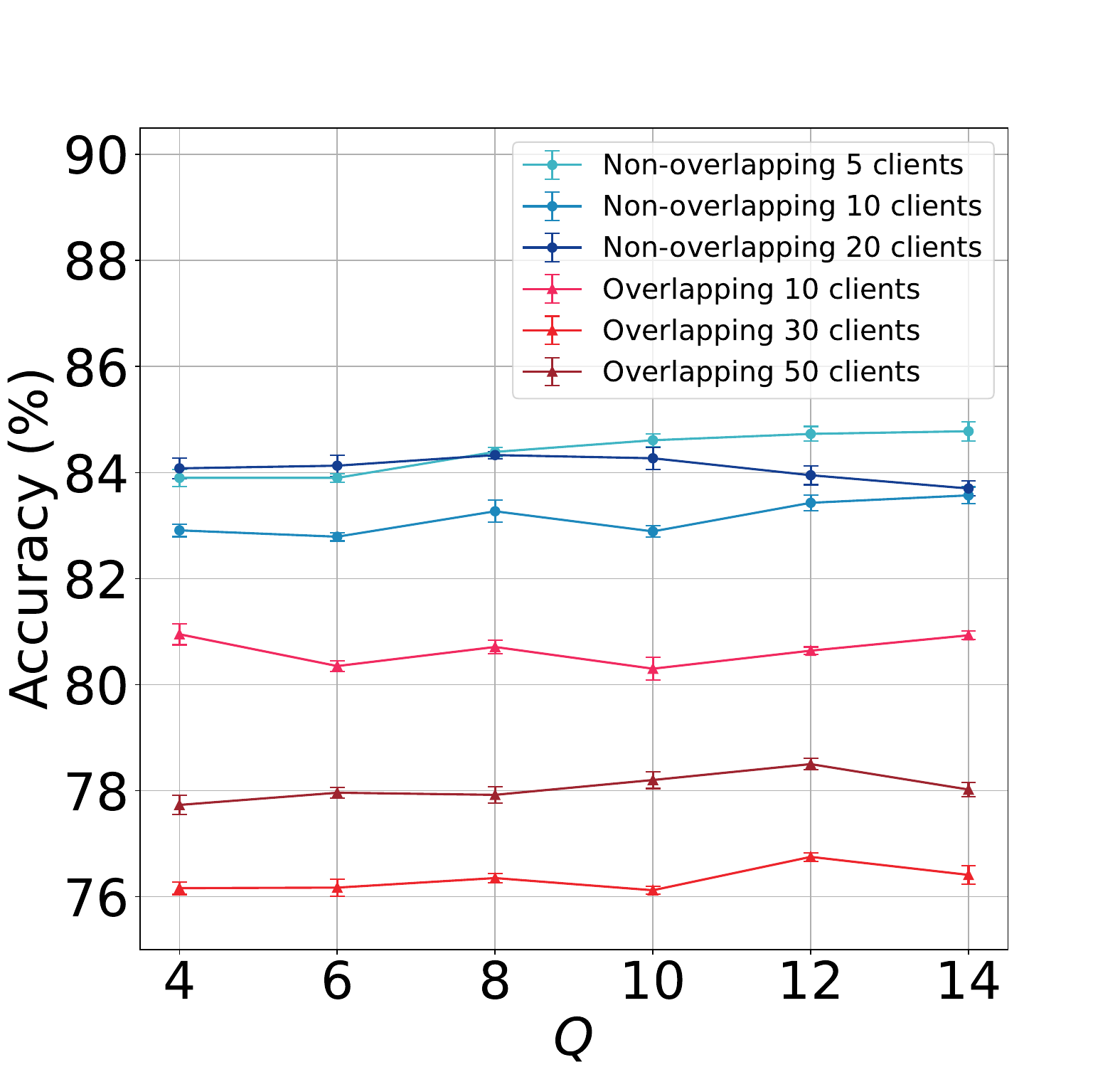}\label{fig_additional_sensitivity1_1}}
  \hfill
  \subfloat[\footnotesize{\textit{Cora} under different $B$}]{\includegraphics[width=0.2\columnwidth]{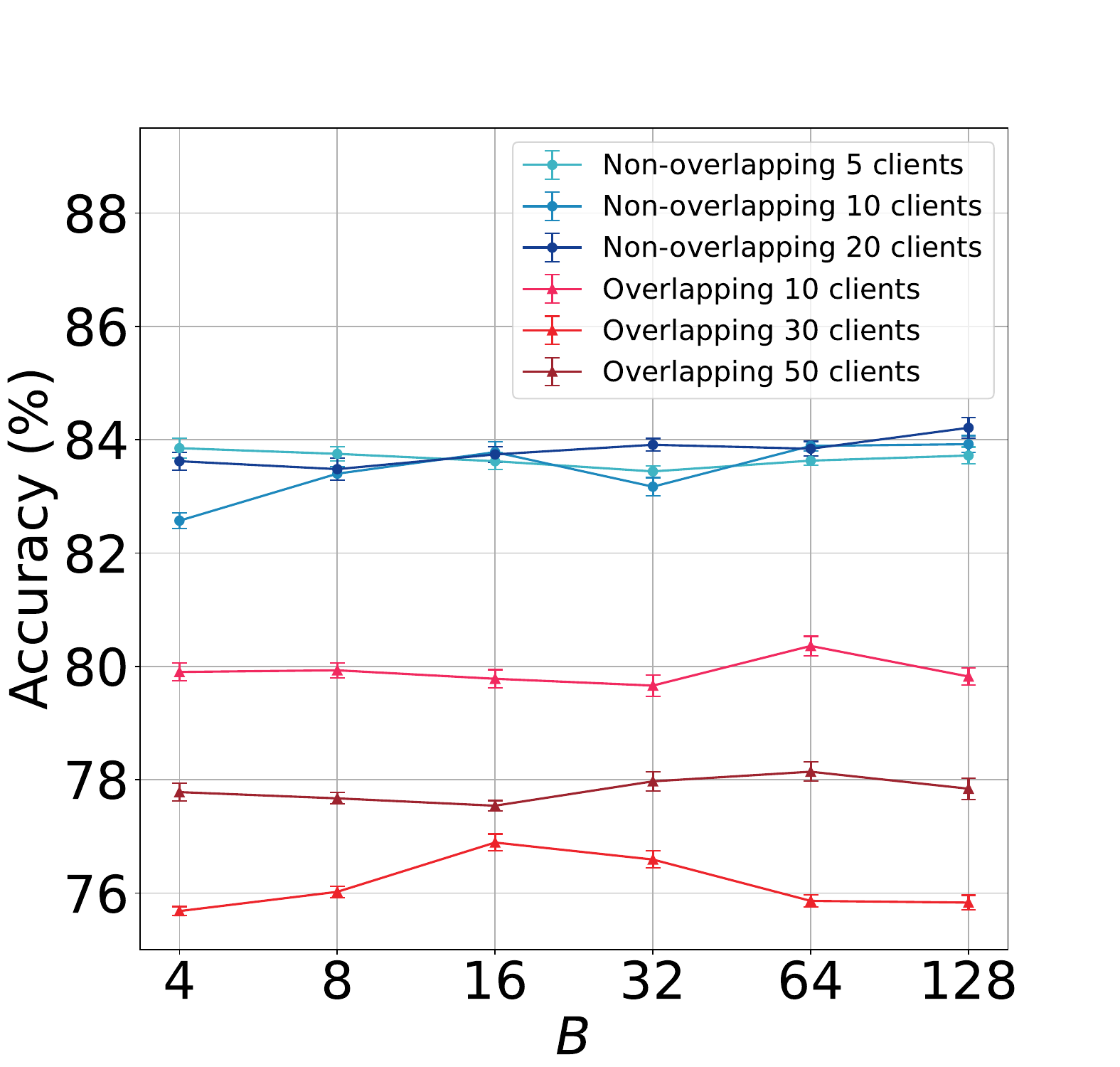}\label{fig_additional_sensitivity1_2}}
  \hfill
  \subfloat[\footnotesize{\textit{Cora} under different $\tau$}]{\includegraphics[width=0.2\columnwidth]{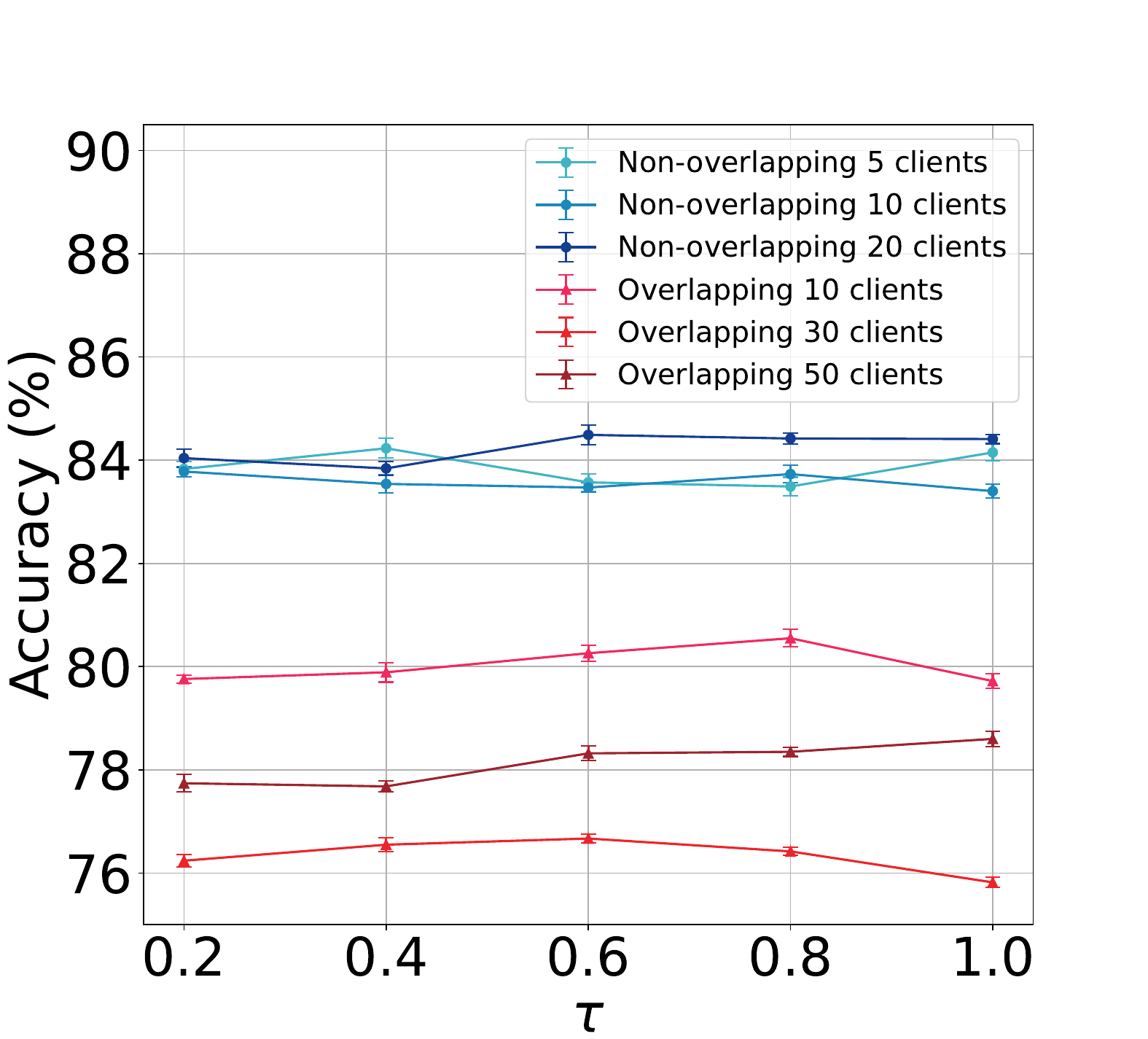}\label{fig_additional_sensitivity1_3}}
  \hfill
  \subfloat[\footnotesize{\textit{Cora} under different $\eta$}]{\includegraphics[width=0.2\columnwidth]{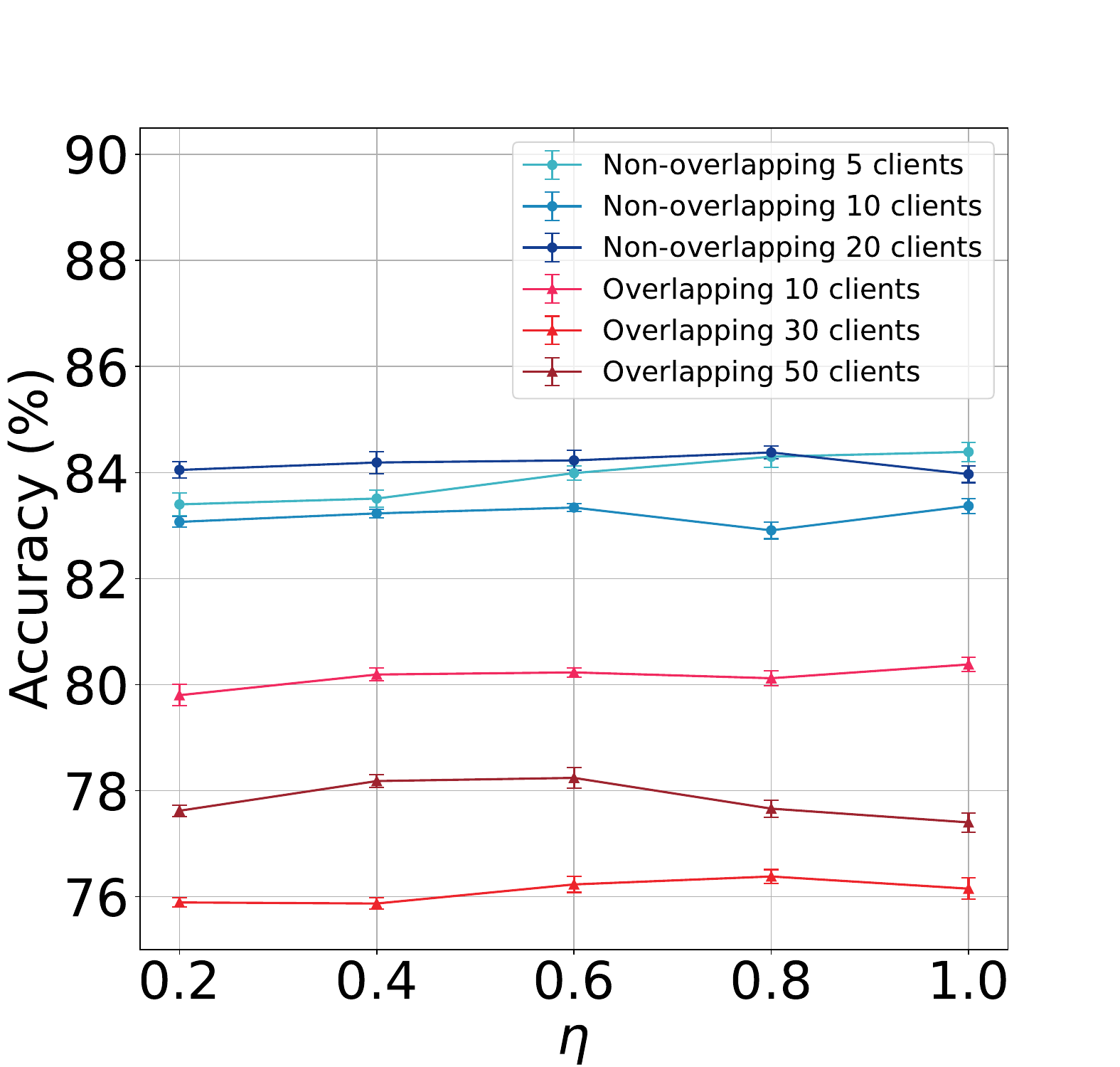}\label{fig_additional_sensitivity1_4}}
  \caption{Accuracy curves with variance bars on \textit{Cora} dataset under different values of $Q$, $B$, $\tau$, and $\eta$.}
  \label{fig_additional_sensitivity1}
\end{figure}

\begin{figure}[t]
  \centering
  \subfloat[\footnotesize{\textit{Roman-empire} under different $Q$}]{\includegraphics[width=0.2\columnwidth]{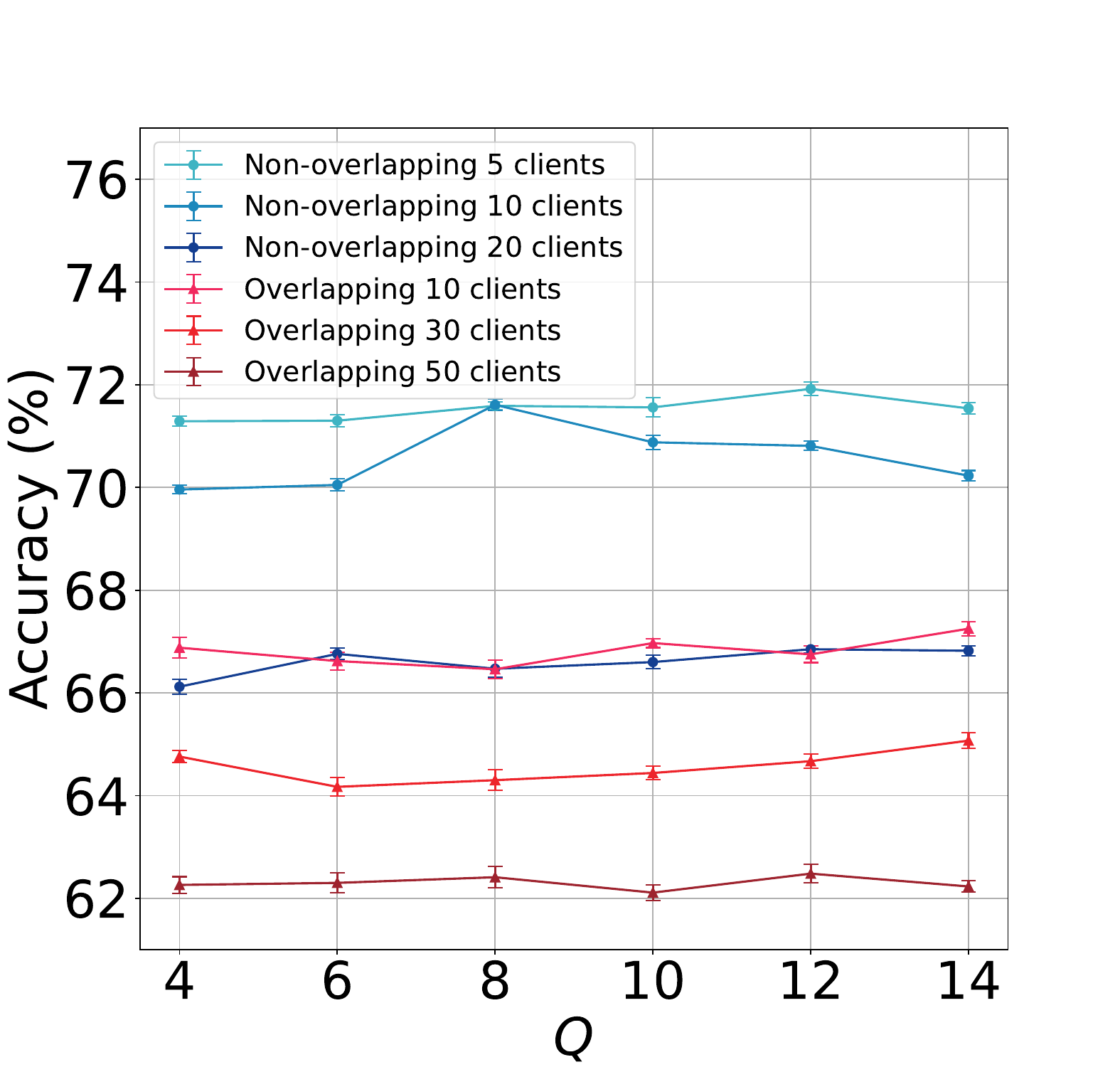}\label{fig_additional_sensitivity_1}}
  \hfill
  \subfloat[\footnotesize{\textit{Roman-empire} under different $B$}]{\includegraphics[width=0.2\columnwidth]{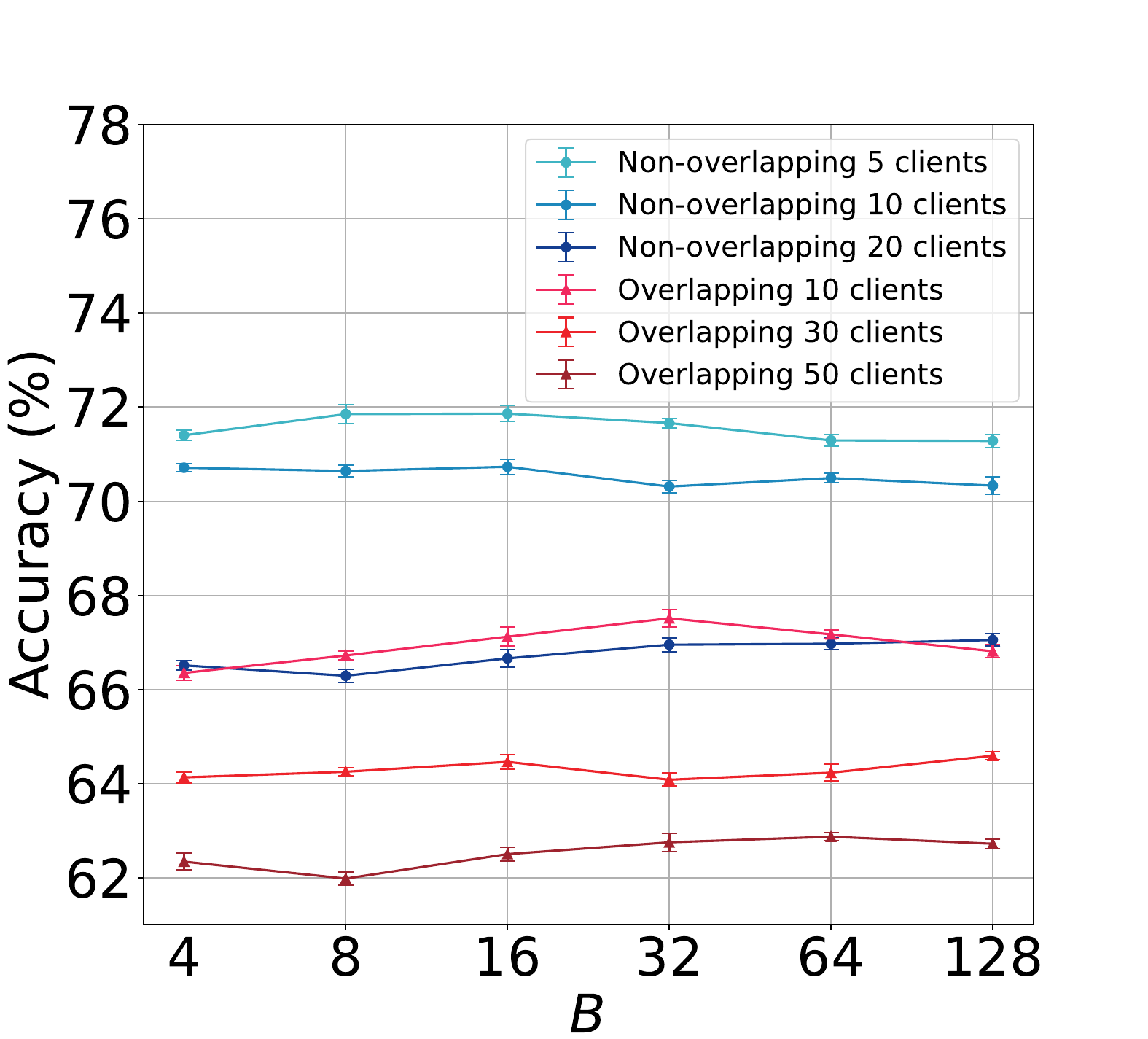}\label{fig_additional_sensitivity_2}}
  \hfill
  \subfloat[\footnotesize{\textit{Roman-empire} under different $\tau$}]{\includegraphics[width=0.2\columnwidth]{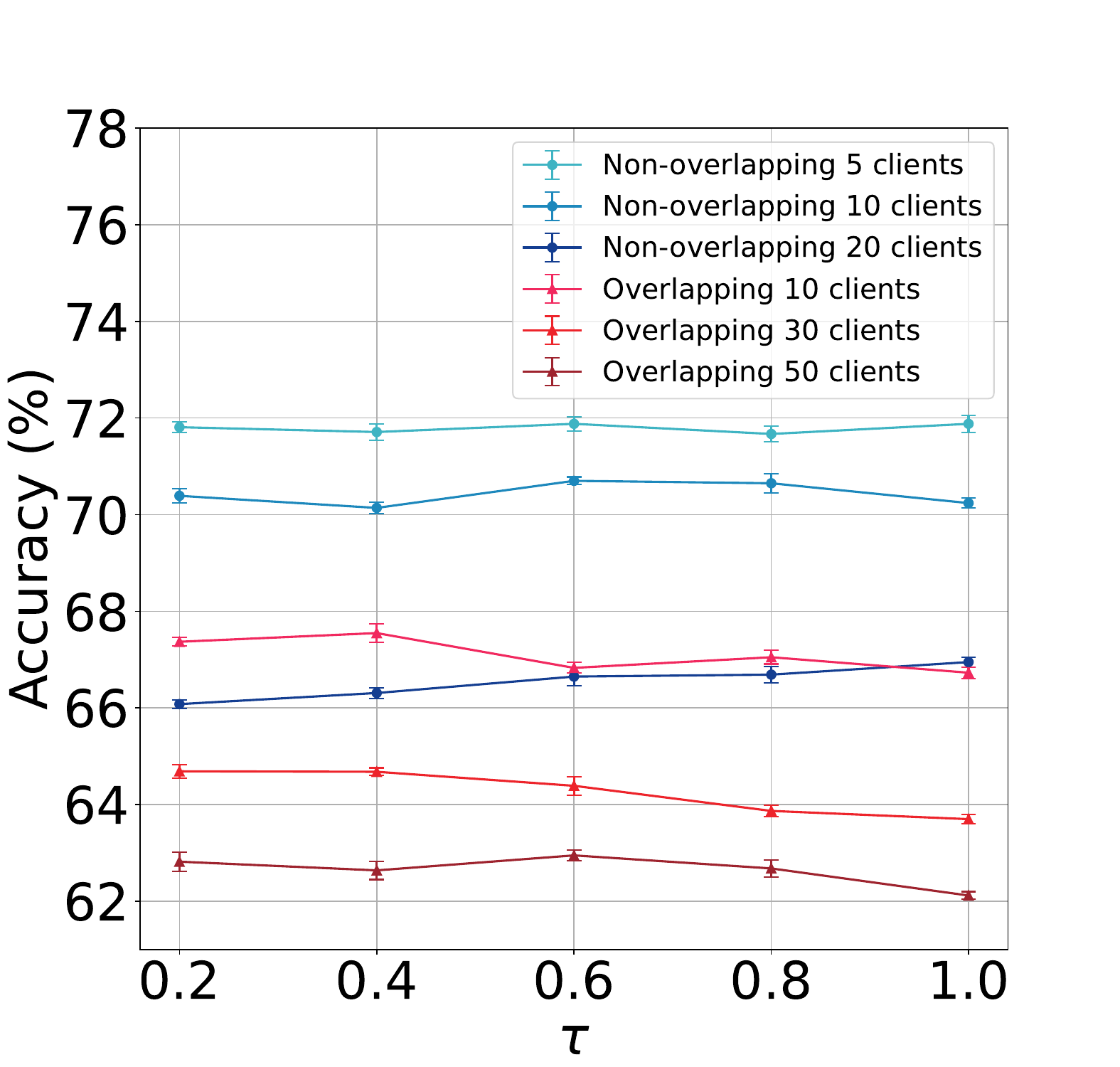}\label{fig_additional_sensitivity_3}}
  \hfill
  \subfloat[\footnotesize{\textit{Roman-empire} under different $\eta$}]{\includegraphics[width=0.2\columnwidth]{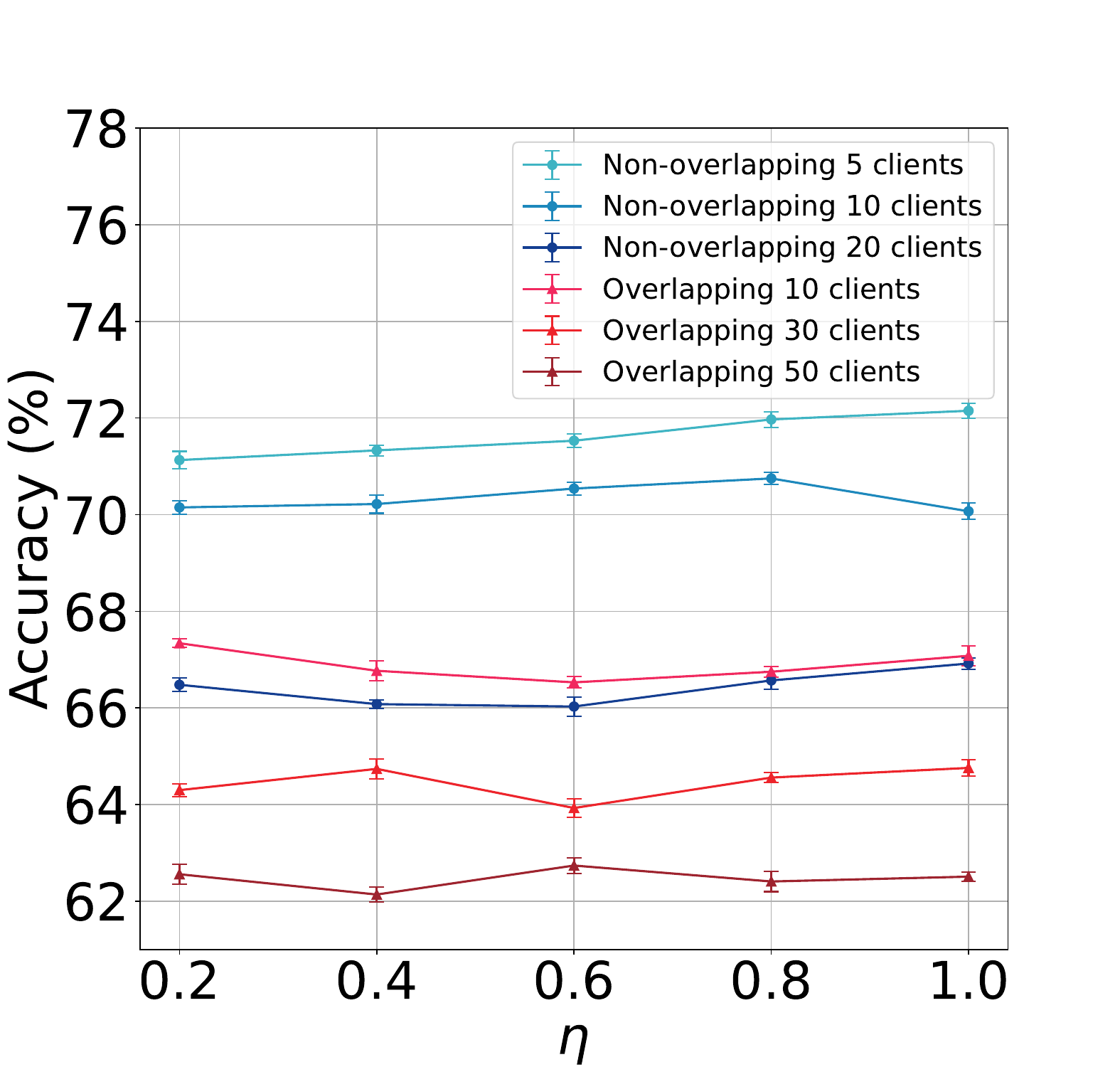}\label{fig_additional_sensitivity_4}}
  \caption{Accuracy curves with variance bars on \textit{Roman-empire} dataset under different values of $Q$, $B$, $\tau$, and $\eta$.}
  \label{fig_additional_sensitivity2}
\end{figure}

\subsection{Additional Case Studies}\label{ap_case_study}
To further illustrate the effectiveness of our proposed FedGMC in mitigating heterogeneity among clients, we conduct additional case studies on six datasets. To be specific, we first visualize global semantic manifolds of various clients by using t-SNE~\cite{van2008visualizing}. Subsequently, we visualize global structural manifolds of various clients. As shown in Fig.~\ref{fig_additional_case_study1} to Fig.~\ref{fig_additional_case_study6}, the 2D projections of semantic manifolds show the property of maximal equidistance, which demonstrates the effectiveness of our FedGMC in mitigating semantic heterogeneity. Similarly, global structural manifolds also show compact clusters, which validates the effectiveness of addressing structural heterogeneity. Moreover, after global refinement, both semantic and structural manifolds exhibit small variations while preserving the maximal equidistance property.

\begin{figure}[t]
	\centering
	\includegraphics[width=13.5cm]{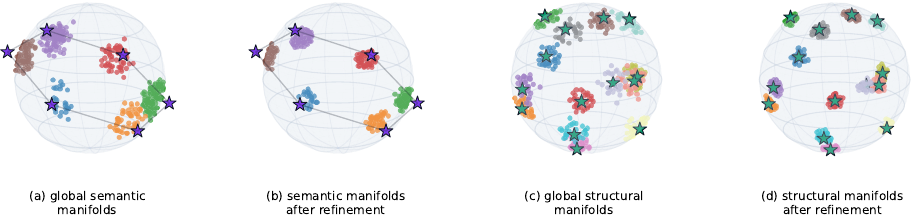}
	\caption{Case studies on \textit{CiteSeer} dataset under non-overlapping partitioning setting with 10 clients.}
	\label{fig_additional_case_study1}
\end{figure}

\begin{figure}[t]
	\centering
	\includegraphics[width=13.5cm]{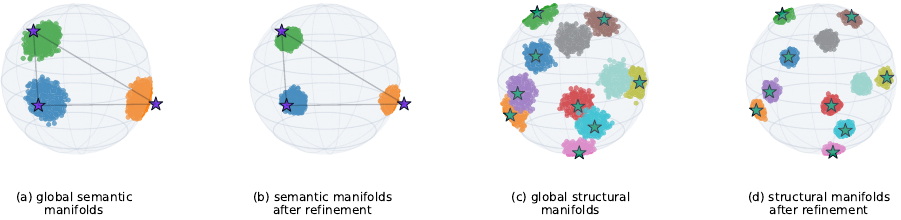}
	\caption{Case studies on \textit{PubMed} dataset under non-overlapping partitioning setting with 10 clients.}
	\label{fig_additional_case_study2}
\end{figure}

\begin{figure}[t]
	\centering
	\includegraphics[width=13.5cm]{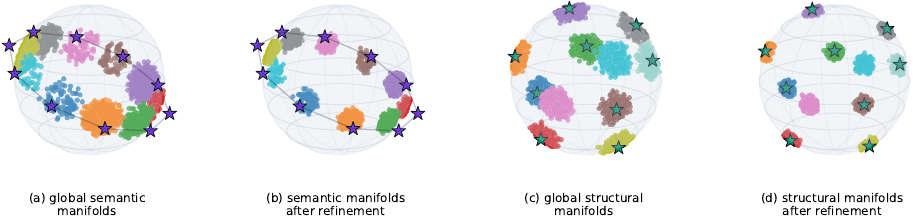}
	\caption{Case studies on \textit{Amazon-Computer} dataset under non-overlapping partitioning setting with 10 clients.}
	\label{fig_additional_case_study3}
\end{figure}

\begin{figure}[t]
	\centering
	\includegraphics[width=13.5cm]{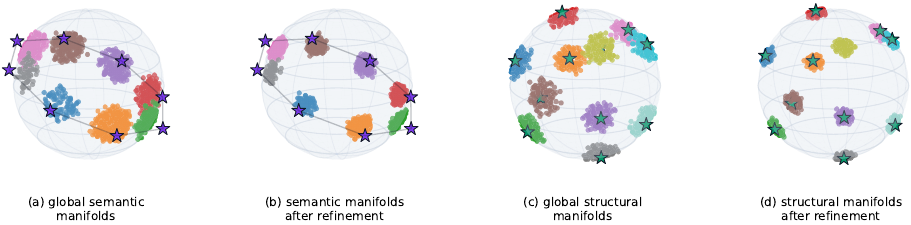}
	\caption{Case studies on \textit{Amazon-Photo} dataset under non-overlapping partitioning setting with 10 clients.}
	\label{fig_additional_case_study4}
\end{figure}

\begin{figure}[t]
	\centering
	\includegraphics[width=13.5cm]{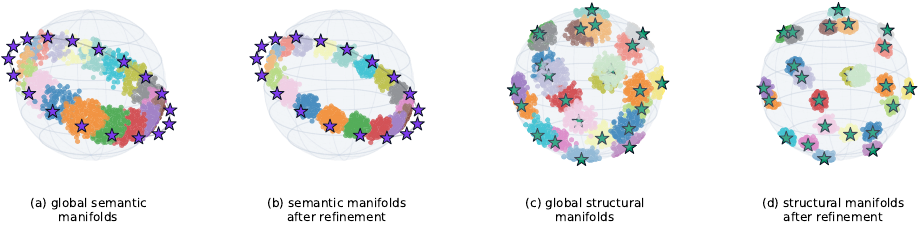}
	\caption{Case studies on \textit{Roman-empire} dataset under non-overlapping partitioning setting with 10 clients.}
	\label{fig_additional_case_study5}
\end{figure}

\begin{figure}[t]
	\centering
	\includegraphics[width=13.5cm]{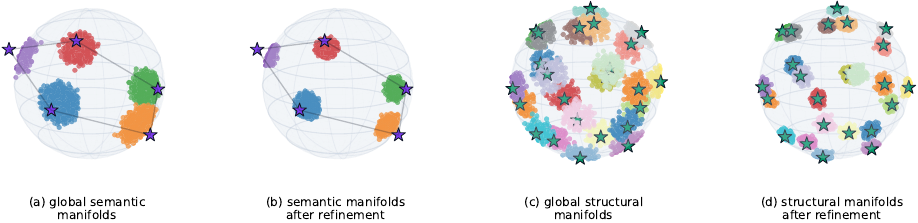}
	\caption{Case studies on \textit{Amazon-ratings} dataset under non-overlapping partitioning setting with 10 clients.}
	\label{fig_additional_case_study6}
\end{figure}

\subsection{Time of Each Communication Round}
\label{time_communication}
We report the time consumed per communication round for our proposed FedGMC and the compared baseline methods in Tab.~\ref{table6}. We can observe that FedGMC consistently achieves a lower time cost per communication round when compared with the average time of baseline methods. Notably, FedGMC is substantially more efficient than strong baseline methods such as FED-PUB and FedIIH. For instance, on \textit{Cora} dataset, FedGMC achieves more than a twofold speed improvement relative to the strong baseline method (\textit{i.e.}, FedIIH). This efficiency gain is attributed to the low complexity of FedGMC on both client side and server side.

\begin{table*}[t]
  \centering
  \scriptsize
  \caption{Time consumption (seconds) of each communication round for our proposed FedGMC and baseline methods on \textit{Cora} and \textit{Roman-empire} datasets.}
  \label{table6}
  \renewcommand{\arraystretch}{1} 
  \scalebox{1}{
  \begin{tabular}{lcccccc}
  \hline
  \rowcolor{gray!50}
                & \multicolumn{6}{c}{Cora}                                                                                                                                                                                                                                                                                                                                                                                                       \\ \hline
  Methods       & \begin{tabular}[c]{@{}c@{}}non-overlapping\\ 5 clients\end{tabular} & \begin{tabular}[c]{@{}c@{}}non-overlapping\\ 10 clients\end{tabular} & \begin{tabular}[c]{@{}c@{}}non-overlapping\\ 20 clients\end{tabular} & \begin{tabular}[c]{@{}c@{}}overlapping\\ 10 clients\end{tabular} & \begin{tabular}[c]{@{}c@{}}overlapping\\ 30 clients\end{tabular} & \begin{tabular}[c]{@{}c@{}}overlapping\\ 50 clients\end{tabular} \\ \hline
  \rowcolor{gray!20}
  FedAvg~\cite{mcmahan2017communication}        & 3.04                                                                & 2.19                                                                 & 5.63                                                                 & 5.40                                                             & 7.27                                                                 & 12.08                                                            \\ \hline
  FedProx~\cite{MLSYS2020_1f5fe839}       & 4.24                                                                & 4.65                                                                 & 8.86                                                                 & 5.49                                                             & 13.71                                                                & 22.43                                                            \\ \hline
  \rowcolor{gray!20}
  FedPer~\cite{Arivazhagan2019}        & 4.06                                                                & 4.13                                                                 & 8.17                                                                 & 4.16                                                             & 12.38                                                                & 20.24                                                            \\ \hline
  GCFL~\cite{NEURIPS2021_9c6947bd}          & 6.30                                                                & 9.38                                                                 & 18.87                                                                & 9.78                                                             & 27.93                                                                & 46.22                                                            \\ \hline
  \rowcolor{gray!20}
  FedGNN~\cite{wu2021fedgnn}        & 2.28                                                                & 4.42                                                                 & 8.76                                                                 & 5.40                                                             & 13.07                                                                & 23.04                                                            \\ \hline
  FedSage+\cite{NEURIPS2021_34adeb8e}      & 6.88                                                                & 8.55                                                                 & 17.88                                                                & 9.37                                                             & 16.97                                                                & 23.35                                                            \\ \hline
  \rowcolor{gray!20}
  FED-PUB~\cite{baek2023personalized}       & 22.04                                                               & 27.34                                                                & 60.31                                                                & 33.46                                                            & 80.54                                                                & 147.84                                                           \\ \hline
  FedGTA~\cite{li2023fedgta}        & 3.36                                                                & 2.24                                                                 & 5.89                                                                 & 4.30                                                             & 5.00                                                                 & 7.18                                                             \\ \hline
  \rowcolor{gray!20}
  AdaFGL~\cite{li2024adafgl}        & 1.99                                                                & 2.49                                                                 & 4.63                                                                 & 4.44                                                             & 6.16                                                                 & 7.83                                                             \\ \hline
  FedTAD~\cite{zhu2024fedtad}        & 4.91                                                                & 5.25                                                                 & 9.13                                                                 & 5.22                                                             & 12.84                                                                & 19.71                                                            \\ \hline
  \rowcolor{gray!20}
  FedIIH~\cite{wentao2025fediih}        & 19.57                                                               & 22.76                                                                & 56.09                                                                & 19.67                                                            & 65.49                                                                & 139.03                                                           \\ \hline
  FedICI~\cite{yuintegrating}        & 6.16                                                               & 2.40                                                                & 6.23                                                                & 5.09                                                            & 7.10                                                                & 11.18                                                           \\ \hline
  \rowcolor{gray!20}
  FedSSA~\cite{yu2026heterogeneity} & 2.87                                                                & 3.46                                                                 & 4.94                                                                 & 6.48                                                             & 11.17                                                                 & 16.47                                                            \\ \hline
  FedGMC (Ours) & 3.07                               & 6.32               & 9.25                                                              & 9.03                                                             & 12.83                                                                 & 17.03                                                            \\ \hline
  \rowcolor{yellow!30} \textbf{Average}       & 6.48                                                                & 7.54                                                                 & 16.05                                                                & 9.09                                                             & 20.89                                                            & 36.69                                                            \\ \hline
  \rowcolor{gray!50}
                & \multicolumn{6}{c}
                {Roman-empire}                                                                                                                                                                                                                                                                                                                                                                                               \\ \hline
  Methods       & \begin{tabular}[c]{@{}c@{}}non-overlapping\\ 5 clients\end{tabular} & \begin{tabular}[c]{@{}c@{}}non-overlapping\\ 10 clients\end{tabular} & \begin{tabular}[c]{@{}c@{}}non-overlapping\\ 20 clients\end{tabular} & \begin{tabular}[c]{@{}c@{}}overlapping\\ 10 clients\end{tabular} & \begin{tabular}[c]{@{}c@{}}overlapping\\ 30 clients\end{tabular} & \begin{tabular}[c]{@{}c@{}}overlapping\\ 50 clients\end{tabular} \\ \hline
  \rowcolor{gray!20}
  FedAvg~\cite{mcmahan2017communication}        & 7.02                                                                & 5.76                                                                 & 8.46                                                                 & 5.01                                                           & 18.47                                                                & 19.12                                                            \\ \hline
  FedProx~\cite{MLSYS2020_1f5fe839}       & 4.31                                                                & 6.73                                                                 & 9.04                                                                 & 6.90                                                             & 16.25                                                                & 23.66                                                            \\ \hline
  \rowcolor{gray!20}
  FedPer~\cite{Arivazhagan2019}        & 5.35                                                                & 6.19                                                                 & 9.19                                                                 & 6.47                                                             & 15.58                                                                & 20.22                                                            \\ \hline
  GCFL~\cite{NEURIPS2021_9c6947bd}          & 6.32                                                                & 9.58                                                                 & 18.37                                                                & 10.84                                                            & 29.91                                                                & 50.23                                                            \\ \hline
  \rowcolor{gray!20}
  FedGNN~\cite{wu2021fedgnn}        & 3.33                                                                & 6.60                                                                 & 9.43                                                                 & 6.45                                                             & 15.29                                                                & 22.82                                                            \\ \hline
  FedSage+\cite{NEURIPS2021_34adeb8e}      & 10.06                                                               & 14.82                                                                & 26.27                                                                & 23.09                                                            & 47.42                                                                & 62.72                                                            \\ \hline
  \rowcolor{gray!20}
  FED-PUB~\cite{baek2023personalized}       & 18.81                                                               & 28.03                                                                & 61.75                                                                & 28.45                                                            & 83.07                                                                & 133.05                                                           \\ \hline
  FedGTA~\cite{li2023fedgta}        & 2.17                                                                & 3.58                                                                 & 5.32                                                                 & 3.42                                                             & 6.12                                                                 & 9.92                                                             \\ \hline
  \rowcolor{gray!20}
  AdaFGL~\cite{li2024adafgl}        & 4.34                                                                & 4.14                                                                 & 5.55                                                                 & 6.49                                                             & 7.80                                                                 & 10.69                                                            \\ \hline
  FedTAD~\cite{zhu2024fedtad}        & 4.88                                                                & 8.55                                                                 & 15.03                                                                & 10.98                                                            & 22.42                                                                & 37.01                                                            \\ \hline
  \rowcolor{gray!20}
  FedIIH~\cite{wentao2025fediih}        & 17.45                                                               & 24.90                                                                & 47.01                                                                & 28.19                                                            & 61.17                                                                & 100.10                                                           \\ \hline
  FedICI~\cite{yuintegrating}        & 5.75                                                               & 5.71                                                                & 8.92                                                                & 9.35                                                            & 17.31                                                                & 20.14                                                           \\ \hline
  \rowcolor{gray!20}
  FedSSA~\cite{yu2026heterogeneity} & 5.78                                                                & 6.01                                                                 & 11.93                                                                 & 9.04                                                            & 16.59                                                                & 24.78                                                            \\ \hline
  FedGMC (Ours) & 7.04                                                                & 6.31                                                                 & 15.24                                                                 & 5.46                                                             & 18.13                                                                 & 33.39                                                            \\ \hline
  \rowcolor{yellow!30} \textbf{Average}       & 7.33                                                                & 9.78                                                                & 17.97                                                                & 11.44                                                            & 26.82                                                            & 40.56                                                            \\ \hline
  \end{tabular}
  }
\end{table*}
\section{Further Discussion}\label{appendix: further discussion}
\subsection{Broader Impact}
This paper could have the following positive impacts: (i) We provide a new method for GFL to deal with the heterogeneity. (ii) Our proposed method can greatly improve the performance of GFL on homophilic and heterophilic graph datasets in both non-overlapping and overlapping settings.

This paper presents work that aims to advance the field of GFL. We do not find any negative societal consequences of our work.
This paper does not raise any ethical concerns. This study does not involve human subjects, practices, data set releases, potentially harmful insights, methodologies, applications, potential conflicts of interest and sponsorship, discrimination/bias/fairness concerns, legal compliance, or research integrity issues. 

In summary, we believe that our proposed FedGMC has shown promising results in dealing with the heterogeneity in GFL. We will continue to improve our work for the benefit of the broader community.
\subsection{Limitations}
The proposed method is currently involving four hyperparameters. Future enhancements could include developing an adaptive variant that does not rely on hand-crafted tuning. Exploring strategies like using heuristic methods, may offer significant advances in the robustness of GFL.

\end{document}